\documentclass[twoside,11pt]{article}

\usepackage{jmlr2e}
\usepackage{amsmath, amssymb, multirow, paralist}
\usepackage{algorithm,algorithmic}
\usepackage{color}
\usepackage{booktabs}

\usepackage{subfigure}
\usepackage{graphicx} % more modern
\usepackage{hyperref}
\hypersetup{
colorlinks=true,
citecolor=blue,
linkbordercolor={1 1 1}, % set to white
citebordercolor={1 1 1} % set to white
}

\jmlrheading{}{}{0-0}{0/0}{0/0}{Yang et al}

\ShortHeadings{Primal-Dual Prox Method}{Yang et al.}
\firstpageno{1}

\def \X {\mathbf{X}}
\def \y {\mathbf{y}}

\def \w {\mathbf w}
\def \x {\mathbf x}

\def \u {\mathbf u}
\def \z {\mathbf z}

\def \Q {\mathcal Q}
\def \a {\mathbf{a}}
\def \b {\mathbf{b}}
\def \R {\mathbb{R}}

\def \balpha {\boldsymbol{\alpha}}
\def \bbeta {\boldsymbol{\beta}}

\begin{document}

\title{An Efficient Primal Dual Prox Method for\\ Non-Smooth Optimization}

\author{\name   Tianbao Yang \email yangtia1@msu.edu \\
       \name Mehrdad Mahdavi \email mahdavim@msu.edu\\
       \name Rong Jin \email rongjin@cse.msu.edu \\
       \addr Department of Computer Science and Engineering\\
       Michigan State University, East Lansing, MI, 48824, USA\\
      \name Shenghuo Zhu \email zsh@sv.nec-labs.com\\
      \addr NEC Labs America, Cupertino, CA, 95014, USA
}

\editor{?}

% The \icmltitle you define below is probably too long as a header.
% Therefore, a short form for the running title is supplied here
\maketitle

\begin{abstract}

We study the non-smooth optimization problems in machine learning, where both the loss function and the regularizer are non-smooth functions. Previous studies on efficient empirical loss minimization assume either a smooth loss function or a strongly convex regularizer, making them unsuitable for non-smooth optimization. We develop a simple yet efficient method for a family of non-smooth optimization problems where the dual form of the loss function is bilinear in primal and dual variables. We cast a non-smooth optimization problem into a minimax optimization problem, and develop a primal dual prox method that solves the minimax optimization problem at a rate of $O(1/T)$ {assuming that the proximal step can be efficiently solved}, significantly faster than a standard subgradient descent method that has an $O(1/\sqrt{T})$ convergence rate. Our empirical study verifies the efficiency of the proposed method for various non-smooth optimization problems that arise ubiquitously  in machine learning by comparing it to the state-of-the-art first order methods.

\end{abstract}
\section{Introduction}
\label{sec:introduction}

Formulating machine learning tasks as a regularized empirical loss minimization problem  makes an intimate connection between machine learning and mathematical  optimization. In regularized empirical loss minimization,  one tries to jointly minimize an empirical loss over training samples plus a regularization term of the model. This formulation includes support vector machine (SVM)~\citep{HastieEtAl2008}, support vector regression~\citep{Smola:2004:TSV:1011935.1011939}, Lasso~\citep{Zhu031-normsupport},  logistic regression,  and ridge regression~\citep{HastieEtAl2008} among many others. Therefore,  optimization methods play a central role in solving machine learning problems and  challenges exist in machine learning applications demand the development of new optimization algorithms. 

Depending on the application at hand,  various types of loss  and regularization functions have been introduced in the literature.  The efficiency of different optimization algorithms crucially depends on the specific structures of the loss and  the regularization functions. Recently, there have been significant  interests on  gradient  descent based methods  due to their simplicity and scalability to large datasets. A well-known example is the Pegasos algorithm~\citep{Shalev-Shwartz:2007:PPE:1273496.1273598} which minimizes the $\ell_2^2$ regularized hinge loss (i.e., SVM) and achieves a convergence rate of $O(1/T)$, where $T$ is the number of iterations, by exploiting the strong convexity of the regularizer. Several other first order algorithms~\citep{Ji2009,10.1109/ICDM.2009.128} are also proposed for smooth loss functions (e.g., squared loss and logistic loss) and non-smooth regularizers (i.e., $\ell_{1,\infty}$ and group lasso). They achieve a convergence rate of $O(1/T^2)$ by exploiting the smoothness of the loss functions. 

In this paper, we focus on a more challenging case where both the loss function and the regularizer are non-smooth, to which we refer as non-smooth optimization. Non-smooth optimization of regularized empirical loss  has found applications in many machine learning problems. Examples of non-smooth loss functions include hinge loss~\citep{citeulike:106699}, generalized hinge loss~\citep{Bartlett:2008:CRO:1390681.1442792}, absolute loss~\citep{HastieEtAl2008}, and $\epsilon$-insensitive loss~\citep{Rosasco:2004:LFS:996933.996940}; examples of non-smooth regularizers include lasso~\citep{Zhu031-normsupport}, group lasso~\citep{citeulike:448082}, sparse group lasso~\citep{DBLP:conf/icml/YangXKL10}, exclusive lasso~\citep{YZhouaistat}, $\ell_{1,\infty}$ regularizer~\citep{Quattoni:2009:EPL:1553374.1553484}, and trace norm regularizer~\citep{Rennie:2005:FMM:1102351.1102441}. 

Although there are already many existing studies on tackling smooth loss functions (e.g., square loss for regression, logistic loss for classification), or smooth regularizers (e.g., $\ell_2^2$ norm),  there  are serious  challenges in developing efficient algorithms for non-smooth optimization. In particular, common tricks, such as smoothing non-smooth objective functions~\citep{Nesterov2005,Nesterov:2005:EGT:1081200.1085585}, can not be applied to non-smooth optimization to improve convergence rate. This is because they require both the loss functions and regularizers be written in the maximization form of bilinear functions, which unfortunately are often violated, as we will  discuss  later. 
%More specifically, when both the loss function and the regularizer are non-smooth, it becomes challenging to design  gradient based algorithms  to achieve a convergence rate better than $O(1/\sqrt{T})$. Having discussed the popularity of  applications with non-smooth regularized empirical loss, devising efficient algorithms seems necessary.
In this work, we focus on optimization problems in machine learning where both the loss function and the regularizer are non-smooth. Our goal is to develop an efficient gradient based algorithm  that has a convergence rate of $O(1/{T})$ for a wide family of non-smooth loss functions and general non-smooth regularizers. 

%\textcolor{red}{A discussion of smoothing technique is needed here (it is only applicable to function with special structures). }

%Deriving efficient algorithms for non-smooth optimization is a very challenging work~\citep{10.1109/ICDM.2009.128}.

%The lack of an efficient algorithm for non-smooth optimization (i.e., both loss function and regularizer are non-smooth) motivates this work.
It is noticeable that according to the information based complexity theory~\citep{Traub:1988:IC:49153}, it is impossible to derive an efficient first order algorithm that generally works for all non-smooth objective functions. As a result, we focus on a  family of non-smooth optimization problems,  where the dual form of the non-smooth loss function is bilinear in both primal and dual variables. Additionally, we show that many non-smooth loss functions have this bilinear dual form. We derive an efficient gradient based method, with a convergence rate of $O(1/T)$, that explicitly updates both the primal and dual variables. The proposed method is referred to as   \textbf{Primal Dual Prox (Pdprox)} method. Besides its capability of dealing with non-smooth optimization, the proposed method is effective in handling the learning problems where additional constraints are introduced for dual variables.

The rest of this paper is organized as follows. Section~\ref{sec:related} reviews the related work on minimizing regularized empirical loss especially the first order methods for large-scale optimization. Section~\ref{sec:not-def} presents some notations and definitions.  Section~\ref{sec:algo} presents the proposed primal dual prox method, its convergence analysis, and several extensions of the proposed method. Section~\ref{sec:exp} presents the empirical studies, and Section~\ref{sec:conc} concludes this work.

\section{Related Work}\label{sec:related}
Our work is closely related to the previous studies on regularized empirical loss minimization. In the following discussion, we mostly focus on non-smooth loss functions and non-smooth regularizers. 

% Although there exist many studies on smooth loss functions (e.g. square loss for regression, logistic loss for classification), and smooth regularizers (e.g. $\ell_2^2$ norm),  there exist still great challenges in developing efficient algorithms for non-smooth optimization. In the following discussion, we focus on non-smooth loss functions and non-smooth regularizers. 

\paragraph{Non-smooth loss functions} Hinge loss is probably the most commonly used non-smooth loss function for classification. It is closely related to the max-margin criterion. A number of algorithms have been proposed to minimize the $\ell_2^2$ regularized hinge loss~\citep{platt98,citeulike:530839,Joachims:2006:TLS:1150402.1150429,Hsieh:2008:DCD:1390156.1390208,Shalev-Shwartz:2007:PPE:1273496.1273598}, and the $\ell_1$ regularized hinge loss~\citep{DBLP:conf/sdm/CaiSCLG10,Zhu031-normsupport,Fung02afeature}. Besides the hinge loss, recently a generalized hinge loss function~\citep{Bartlett:2008:CRO:1390681.1442792} has been proposed for cost sensitive learning. For regression, square loss is commonly used due to its smoothness.  However, non-smooth  loss functions such as absolute loss~\citep{HastieEtAl2008} and $\epsilon$-insensitive loss~\citep{Rosasco:2004:LFS:996933.996940}  are useful for robust regression. The Bayes optimal predictor of  square loss is the mean of the predictive distribution, while the Bayes optimal predictor  of absolute loss is the median of the predictive distribution. Therefore absolute loss is more robust for long-tailed error distributions and outliers~\citep{HastieEtAl2008}. \citep{Rosasco:2004:LFS:996933.996940} also proved that the estimation error bound for absolute loss and $\epsilon$-insensitive loss converges faster than that of square loss. Non-smooth piecewise linear loss function  has been used  in quantile regression~\citep{Koenker2005,tilmann}. Unlike the absolute loss, the piecewise linear loss function can model non-symmetric error in reality. 

\paragraph{Non-smooth regularizers}
Besides the simple non-smooth regularizers such as $\ell_1$, $\ell_2$,  and $\ell_\infty$ norms~\citep{Duchi:2009:EOB:1577069.1755882}, many other non-smooth regularizers have been employed in  machine learning tasks. \citep{citeulike:448082} introduced group lasso for selecting important explanatory factors in group manner.  The $\ell_{1,\infty}$ norm regularizer has been used for multi-task learning~\citep{journals/ml/ArgyriouEP08}. In addition, several recent works~\citep{DBLP:conf/ijcai/HouNYW11,citeulike:9315911,Liu:2009:MFL:1795114.1795154}  considered mixed $\ell_{2,1}$ regularizer for feature selection. \citep{YZhouaistat} introduced exclusive lasso for multi-task feature selection to model the scenario where variables within a single group compete with each other. Trace norm regularizer is another non-smooth regularizer, which has found applications in matrix completion~\citep{citeulike:9621053,DBLP:journals/corr/abs-0805-4471}, matrix factorization~\citep{Rennie:2005:FMM:1102351.1102441,citeulike:3224462}, and multi-task learning~\citep{journals/ml/ArgyriouEP08,Ji2009}.  The optimization algorithms presented in these works are usually limited: either  the convergence rate is not guaranteed~\citep{journals/ml/ArgyriouEP08,citeulike:9621053,DBLP:conf/ijcai/HouNYW11,citeulike:9315911,Rennie:2005:FMM:1102351.1102441,citeulike:3224462} or the loss functions are assumed to be smooth (e.g., the square loss or the logistic loss)~\citep{Liu:2009:MFL:1795114.1795154,Ji2009}. Despite  the significant efforts in developing algorithms  for minimizing regularized empirical losses,  it remains a challenge to design a  first order algorithm that is able to efficiently solve non-smooth optimization problems at a rate of $O(1/T)$ when both the loss function and the regularizer are non-smooth.

\paragraph{Gradient based optimization}Our work is closely related to (sub)gradient based optimization methods. The convergence rate of gradient based methods usually depends on the properties of the objective function to be optimized. When the objective function is strongly convex and smooth, it is well known that  gradient descent methods can achieve a geometric convergence rate~\citep{citeulike:163662}. When the objective function is  smooth but not strongly convex, the optimal convergence rate of a gradient descent method is  $O(1/T^2)$, and is achieved by the Nesterov's methods~\citep{RePEc:cor:louvco:2007076}. For  the objective function which is strongly convex but not smooth, the convergence rate becomes $O(1/T)$~\citep{Shalev-Shwartz:2007:PPE:1273496.1273598}. For general non-smooth objective functions, the optimal rate of any first order method is $O(1/\sqrt{T})$. Although it is not improvable in general, recent studies are able to improve this rate to $O(1/T)$ by exploring the special structure of the objective function~\citep{Nesterov2005,Nesterov:2005:EGT:1081200.1085585}. In addition, several methods are developed for composite optimization, where the objective function is written as a sum of a smooth and a non-smooth function~\citep{Lan08,RePEc:cor:louvco:2007076,2010arXiv1008.5204L}. Recently, these optimization techniques have been successfully applied to various machine learning problems, such as SVM~\citep{DBLP:journals/corr/abs-1008-4000}, general regularized empirical loss minimization~\citep{Duchi:2009:EOB:1577069.1755882,NIPS2009_0997},  trace norm minimization~\citep{Ji2009}, and multi-task sparse learning~\citep{10.1109/ICDM.2009.128}. Despite these efforts, one major limitation of the existing (sub)gradient based algorithms is that in order to achieve a convergence rate better than $O(1/\sqrt{T})$, they have to assume that the loss function is smooth or the regularizer is strongly convex, making them unsuitable for non-smooth optimization.

\paragraph{Convex-concave optimization}The present work is also related to convex-concave minimization. \citet{pau08} and \citet{Nemirovski2005} developed prox methods that have a convergence rate of $O(1/T)$, provided the gradients are Lipschitz continuous and have been applied to machine learning problems~\citep{Sun:2009:NIPS}. In contrast, our method achieves a rate of $O(1/T)$ without requiring the whole gradient but part of the gradient to be Lipschitz continuous.  Several other primal-dual algorithms have been developed for regularized empirical loss minimization that update both primal and dual variables. \citep{citeulike:7217768} proposed a primal-dual method based on gradient descent, which only achieves a rate of $O(1/\sqrt{T})$. It was generalized in~\citep{citeulike:10247176}, which shares the similar spirit of the proposed algorithm. However, the explicit convergence rate was not established even though the convergence is proved. \citep{mosci10nips} presented a primal-dual algorithm for group sparse regularization, which updates the primal variable by a prox method and the dual variable by a Newton's method. In contrast, the proposed algorithm is a first order method that does not require computing the Hessian matrix as the Newton's method does, and is therefore more scalable to large datasets. \citep{combettes:hal-00643381,arxivradu2012} proposed  primal-dual splitting algorithms for finding zeros of maximal monotone operators of special types. \citep{Lan:2011:PFM:1922499.1922500} considered the  primal-dual convex formulations for general cone programming and apply Nesterov's optimal first order method~\citep{RePEc:cor:louvco:2007076}, Nesterov's smoothing technique~\citep{Nesterov2005}, and Nemirovski's prox method~\citep{Nemirovski2005}. \citet{Nesterov:2005:EGT:1081200.1085585} proposed a primal dual gradient method for a special class of structured non-smooth optimization problems  by exploring an excessive gap technique. 

\paragraph{Optimizing non-smooth functions}We note that Nesterov's smoothing technique ~\citep{Nesterov2005} and excessive gap technique~\citep{Nesterov:2005:EGT:1081200.1085585} can be applied to non-smooth optimization and both achieve $O(1/T)$ convergence rate for a special class of non-smooth optimization problems. However, the limitation of these approaches is that they require  all the non-smooth terms (i.e., the loss and the regularizer) to be written as an explicit max structure that consists of a bilinear function in primal and dual variables, thus limits their applications to many machine learning problems. In addition, Nesterov's algorithms need to solve  additional maximizations problem at each iteration. In contrast, the proposed algorithm only requires mild condition on the non-smooth loss functions (section~\ref{sec:algo}), and allows for any commonly used non-smooth regularizers, without having to solve an additional optimization problem at each iteration.  Compared to Nesterov's algorithms, the proposed algorithm is applicable to a large class of non-smooth optimization problems, is easier to implement, its convergence analysis is much simpler, and its empirical performance is usually comparably favorable. 
Finally we noticed that, as we are preparing our manuscript, a related work~\citep{Chambolle:2011:FPA:1968993.1969036} has recently been published in the Journal of Mathematical Imaging and Vision that shares a similar idea as this work. {Both works maintain and update the primal and dual variables for solving a non-smooth optimization problem, and achieve the same convergence rate (i.e., $O(1/T)$). However, our work distinguishes from~\citep{Chambolle:2011:FPA:1968993.1969036} in following aspects:
(i) We propose and analyze two primal dual prox methods: one gives an extra gradient updating to dual variables and the other gives an extra gradient updating to primal variables. Depending on the nature of applications, one method may be more efficient than the others; %In contrast, ~\citep{Chambolle:2011:FPA:1968993.1969036} only considers the algorithm that gives an extra update to primal variables~\footnote{The differences between our algorithms and the algorithm in~\citep{Chambolle:2011:FPA:1968993.1969036} are also discussed in the appendix.}.
%(ii) The convergence result stated in~\citep{Chambolle:2011:FPA:1968993.1969036} (i.e., Theorem 1) is restricted to finite space and therefore cannot be applied directly to kernel learning. In contrast, we extend our algorithm to Reproduce Kernel Hilbert Space (RKHS) in Section~\ref{sec:ext}.
(ii) In Section~\ref{sec:imp}, we discuss how to efficiently solve the interim projection problems for updating both primal variable and  dual variable, a critical issue for making the proposed algorithm practically efficient. In contrast, \citep{Chambolle:2011:FPA:1968993.1969036} simply assumes that the interim projection problems can be solved efficiently;}
(iii) We focus our analysis and empirical studies on the optimization problems that are closely related to machine learning.  We demonstrate the effectiveness of the proposed algorithm on various classification, regression, and matrix completion tasks with non-smooth loss functions and non-smooth regularizers;  (iv) We also conduct analysis and experiments on the convergence of the proposed methods when dealing with the $\ell_1$ constraint on the dual variable, an approach that is commonly used in robust optimization, and observe that the proposed methods converge much faster when the  bound of the $\ell_1$ constraint is small  and  the obtained solution is more robust in terms of prediction in the presence of noise in labels. In contrast, the study~\citep{Chambolle:2011:FPA:1968993.1969036} only considers the application in image problems.

{ We also note that the proposed algorithm is closely related to proximal point algorithm~\citep{citeulike:9472207} as shown in~\citep{heyu12}, and many variants including the modified Arrow-Hurwicz method~\citep{popov1980}, the Doughlas-Rachford  (DR) splitting algorithm~\citep{1977776}, the alternating method of multipliers (ADMM)~\citep{Boyd:2011:DOS:2185815.2185816}, the forward-backward splitting algorithm~\citep{IOPORT.05500734}, the FISTA algorithm~\citep{Beck:2009:FIS:1658360.1658364}. For a detailed comparison with some of these algorithms, one can refer  to~\citep{Chambolle:2011:FPA:1968993.1969036}.}

\begin{figure}[t]
\centering
\subfigure[classification]{\includegraphics[scale=0.25]{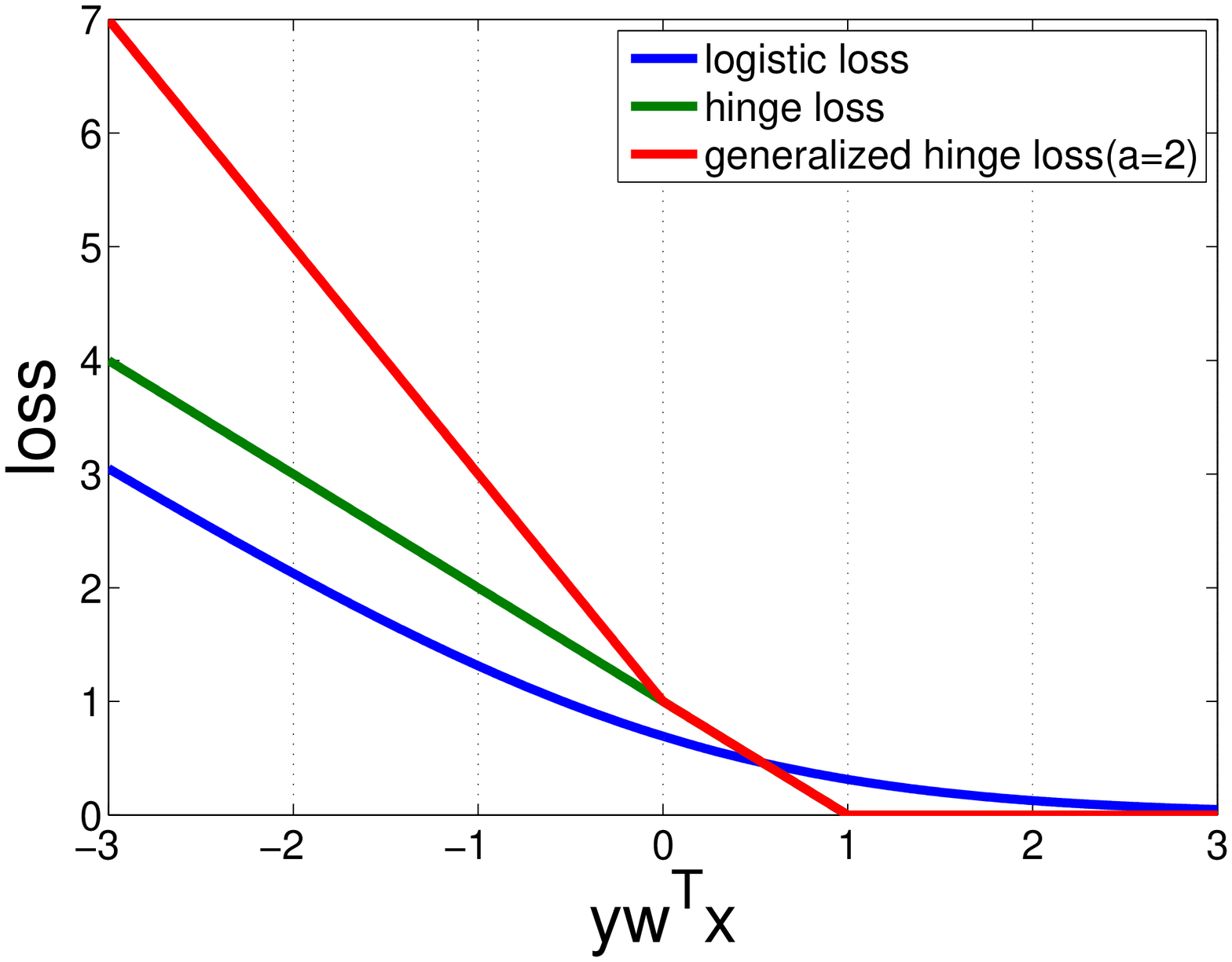}}\hspace*{-0.1in}
\subfigure[regression]{\includegraphics[scale=0.25]{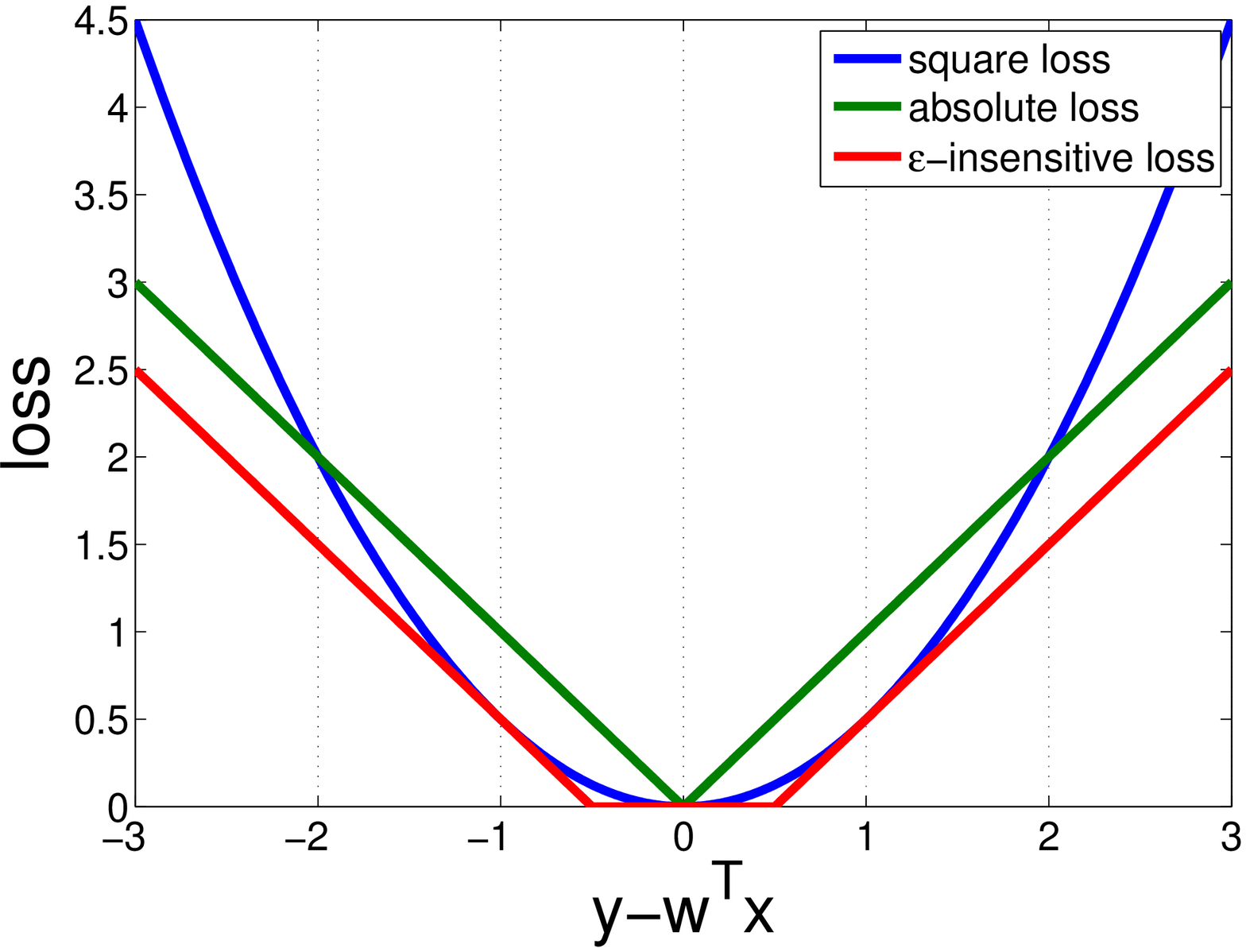}}\hspace*{-0.1in}\\
\caption{Loss functions}\label{fig:loss}
\end{figure}

 %Finally, we note that serveral first order stochastic methods~\citep{Juditsky:1124237, Shalev-Shwartz:2009:SML:1553374.1553493} have been proposed for regularized empirical loss minimization,  and Newton-type methods~\citep{Yu:2010:QAN:1756006.1756045} are proposed for non-smooth optimization. We did not include them for comparison since we focus on first order batch mode optimization. 

% we focus on deterministic approaches for optimization algorithms.
\section{Notations and Definitions}
\label{sec:not-def}
 In this section we provide the basic setup, some preliminary definitions and notations used throughout this paper.

We denote by $[n]$ the set of integers $\{1,\cdots, n\}$.  We denote by $(\x_i, y_i), i \in [n]$ the training examples, where $\x_i\in\mathcal X\subseteq\mathbb R^d$ and $y_i$ is the assigned class label, which is discrete for classification and continuous for regression. {We assume $\|\x_i\|_2\leq R, \; \forall i \in [n]$.} We denote by $\mathbf X=(\x_1,\cdots, \x_n)^{\top}$ and $\mathbf y=(y_1,\cdots, y_n)^{\top}$.   Let $\w\in\mathbb R^d$ denote the linear hypothesis, $\ell(\w; \x, y)$ denote a loss of prediction made by the hypothesis $\w$ on example $(\x, y)$, which is a convex function in terms of $\w$. Examples of convex loss function are hinge loss $\ell(\w; \x, y)= \max(1-y\w^{\top}\x, 0)$, and absolute loss $\ell(\w; \x, y) =|\w^{\top}\x - y|$. To characterize a function, we introduce the following definitions 
\begin{definition}
A function $\ell(\z): \mathcal Z\rightarrow\mathbb R$ is a $G$-Lipschitz continuous if 
\begin{align*}
|\ell(\z_1)-\ell(\z_2)|\leq G\|\z_1-\z_2\|_2, \forall \z_1,\z_2\in\mathcal Z.
\end{align*}
\end{definition}
\begin{definition}
A function $\ell(\z):\mathcal Z\rightarrow \mathbb R$ is a $\rho$-smooth function if its gradient is $\rho$-Lipschitz continuous 
\begin{align*}
\|\nabla \ell(\z_1) - \nabla \ell(\z_2)\|_2\leq \rho\|\z_1-\z_2\|_2, \forall \z_1,\z_2\in\mathcal Z.
\end{align*}
\end{definition}
A function is non-smooth if either its gradient is not well defined or  its gradient is not Lipschtiz continuous.  Examples of smooth loss functions are logistic loss $\ell(\w; \x, y) = \log(1+\exp(-y\w^{\top}\x))$, square loss $\ell(\w; \x, y) = \frac{1}{2}(\w^{\top}\x-y)^2$, and examples of non-smooth loss functions are hinge loss, and absolute loss. The difference between logistic loss and hinge loss,  square loss and absolute loss can be seen in Figure~\ref{fig:loss}. Examples of non-smooth regularizer include $R(\w)= \|\w\|_1$, i.e. $\ell_1$ norm, $R(\w)= \|\w\|_{\infty}$, i.e. $\ell_\infty$ norm. More examples can be found in section~\ref{sec:non-smooth}.

In this paper, we aim to solve the following optimization problem, which occurs in many machine learning problems, 
\begin{align}\label{eqn:mcp}
\min_{\w\in \mathbb R^d}\quad \mathcal L(\w)= \frac{1}{n}\sum_{i=1}^n\ell(\w; \x_i, y_i)  + \lambda R(\w),
\end{align}
where $\ell(\w; \x, y)$ is a non-smooth loss function,  $R(\w)$ is a non-smooth regularizer on $\w$, and $\lambda$ is a regularization parameter. 

%\textbf{Remark:} Note that for the ease  of presentation, we denote parameters in a vector form, the generalization to a matrix form is straightforward.

We denote by $\Pi_\Q [\widehat\z]=\arg\min\limits_{\z\in\Q}\frac{1}{2}\|\z-\widehat\z\|_2^2$ the projection of $\widehat\z$ into domain $\Q$, and by $\Pi_{\Q_1, \Q_2}\begin{pmatrix}\widehat\z_1\\ \widehat\z_2\end{pmatrix}$ the joint projection of $\widehat\z_1$ and $\widehat\z_2$ into domains $\Q_1$ and $\Q_2$, respectively. Finally, we use $[s]_{[0,a]}$ to  denote the projection of $s$ into $[0, a]$, where $a>0$.
%\begin{align*}
%[s]_{[0,s_+]} &= \left\{ \begin{array}{ll}
%       s & s\in[0,s_+]\\
%       0 & s<0\\
%        1&s>s_+\end{array} \right..
%\end{align*}
\section{Pdprox: A Primal Dual Prox Method for Non-Smooth Optimization}\label{sec:algo}

We first describe the non-smooth optimization problems that the proposed algorithm can be applied to,  and then present the primal dual prox method for non-smooth optimization. We then prove the convergence rate of the proposed algorithms and discuss several extensions. Proofs for technical lemmas  are deferred to the  appendix.  

\subsection{Non-Smooth Optimization}\label{sec:non-smooth}
We first focus our  analysis on linear classifiers and denote by $\w\in\mathbb R^d$ a linear model. The extension  to  nonlinear models is discussed in section~\ref{sec:ext}. Also, extension to a  collection of linear models $\mathbf W\in\mathbb R^{d\times K}$ can be done in a straightforward way.  We consider the following general \textit{non-smooth} optimization problem:
\begin{align}
\min_{\w\in \Q_\w} \Bigg{[}  \mathcal L(\w) =  \max_{\balpha\in \Q_{\balpha}}L(\w, \boldsymbol{\alpha}; \X, \mathbf y)+ \lambda R(\w) \Bigg{]} . \label{eqn:p-obj}
\end{align}
%where the first maximization term characterizes a non-smooth loss function, and $R(\w)$ is also a non-smooth regularizer.
The parameters $\w$ in domain $\Q_\w$ and  $\balpha$ in  domain $\Q_{\balpha}$ are referred to as  primal  and  dual variables, respectively.
% The parameters $\w$ and $\balpha$ are referred to as the primal variable and the dual variable, respectively, and $\Q_\w, \Q_{\balpha}$ are the domain of the primal variable and the dual variable, respectively. 
% 
 Since it is impossible to develop an efficient first order method for general non-smooth optimization, we focus on the family of non-smooth loss functions that can be characterized by bilinear function $L(\w, \balpha;\X, \mathbf y)$, i.e. 
\begin{align}
L(\w, \balpha; \X, \mathbf y) & = c_0(\X, \y) + \balpha^{\top}\mathbf a(\X, \y)  + \w^{\top}\b(\X, \y) + \w^{\top} \mathbf H(\X, \y) \balpha, \label{eqn:f}
\end{align}
where $c_0(\X, \y)$, $\mathbf a(\X, \y)$, $\mathbf b(\X, \y)$, and $\mathbf H(\X, \y)$ are the parameters depending on the training examples $(\X, \y)$ with consistent sizes. 
In the sequel, we denote by $L(\w, \balpha)=L(\w, \balpha; \mathbf X, \mathbf y)$  for simplicity, and by $G_\w(\w, \balpha)=\nabla_\w L(\w, \balpha)$ and $G_\alpha(\w,\balpha)=\nabla_{\balpha} L(\w, \balpha)$ the partial gradients of $L(\w, \balpha)$ in terms of $\w$ and $\balpha$, respectively. 

\paragraph{Remark 1}One direct consequence of assumption in (\ref{eqn:f}) is that the partial gradient $G_{\w}(\w,\balpha)$ is independent of $\w$, and $G_{\balpha}(\w,\balpha)$ is independent of $\balpha$, since $L(\w,\balpha)$ is bilinear in $\w$ and $\balpha$. We will explicitly exploit this property in developing the efficient optimization algorithms.  We also note that no explicit assumption is made for the regularizer $R(\w)$. This is in contrast to the smoothing techniques used in~\citep{Nesterov2005,Nesterov:2005:EGT:1081200.1085585}.

To efficiently  solve the optimization problem in~(\ref{eqn:mcp}), we need first  turn it  into the form~(\ref{eqn:p-obj}). To this end,  we assume that the loss function can be written into a dual form, which is  bilinear in the primal and the dual variables, i.e. 
\begin{align}\label{eqn:dual}
\ell(\w; \x_i, y_i)= \max_{\alpha_i\in\Delta_\alpha}f(\w, \alpha_i; \x_i, y_i),
\end{align}
where $f(\w, \alpha; \x, y)$ is a bilinear function in $\w$ and $\alpha$, and $\Delta_\alpha$ is the domain of variable $\alpha$.  Using (\ref{eqn:dual}), we cast problem~(\ref{eqn:mcp}) into~(\ref{eqn:p-obj}) with $L(\w, \balpha; \X, \y)$ given by
\begin{align}\label{eqn:dual2}
L(\w, \balpha; \X, \y) = \frac{1}{n}\sum_{i=1}^n f(\w, \alpha_i; \x_i, y_i),
\end{align}
with $\balpha=(\alpha_1,\cdots, \alpha_n)^{\top}$ defined in the domain $\Q_{\balpha} = \{\balpha=(\alpha_1,\cdots, \alpha_n)^{\top}, \alpha_i\in\Delta_\alpha\}$. 

Before delving into the description of the proposed algorithms and their analysis, we give a few examples that show many non-smooth loss functions can be written in the form of (\ref{eqn:dual}):
\begin{itemize}
\item Hinge loss~\citep{citeulike:106699}: 
\begin{align*}
\ell(\w; \x, y)&=\max(0, 1-y\w^{\top}\x)=\max_{\alpha\in[0, 1]} \alpha(1-y\w^{\top}\x).
\end{align*}

\item Generalized hinge loss~\citep{Bartlett:2008:CRO:1390681.1442792}:
\begin{align*}
\hspace*{-0.3in}\ell(\w;\x, y) &= \left\{ \begin{array}{ll}
         1-ay\w^{\top}\x & \mbox{if $y\w^{\top}\x \leq 0$}\\
        1-y\w^{\top}\x & \mbox{if $0<y\w^{\top}\x < 1$}\\
        0&\mbox{if $y\w^{\top}\x\geq 1$}\end{array} \right.\\
        &\hspace*{-0.3in}=  \max_{\alpha_1 \geq0, \alpha_2 \geq 0\atop\alpha_1+\alpha_2\leq 1}\alpha_1(1-ay\w^{\top}\x) + \alpha_2(1-y\w^{\top}\x),
         \end{align*}
 where $a>1$.
 \item Absolute loss~\citep{HastieEtAl2008}: 
 \[
 \hspace*{-1in}\ell(\w; \x, y)=|\w^{\top}\x-y|=\max_{\alpha\in[-1,1]}\alpha (\w^{\top}\x- y).
 \]
 \item $\epsilon$-insensitive loss~\citep{Rosasco:2004:LFS:996933.996940} : 
 \begin{align*}
 &\ell(\w; \x, y)=\max(|\w^{\top}\x-y|-\epsilon, 0)=\max_{\alpha_1\geq 0,\alpha_2\geq 0\atop \alpha_1+\alpha_2\leq 1}\left[ (\w^{\top}\x - y)(\alpha_1-\alpha_2) - \epsilon(\alpha_1+\alpha_2)\right].
 \end{align*}
 \item Piecewise linear loss~\citep{Koenker2005}: 
 \begin{align*}
\hspace*{-0.3in}\ell(\w;\x, y) &= \left\{ \begin{array}{ll}
        a|\w^{\top}\x- y|& \mbox{if $\w^{\top}\x \leq y$}\\
       (1-a)|\w^{\top}\x-y| & \mbox{if $\w^{\top}\x \geq y$}\end{array} \right.\\
        &\hspace*{-0.3in}=  \max_{\alpha_1 \geq0, \alpha_2 \geq 0\atop\alpha_1+\alpha_2\leq 1}\alpha_1a(y-\w^{\top}\x) + \alpha_2(1-a)(\w^{\top}\x-y).
 \end{align*}
 \item $\ell_2$ loss~\citep{citeulike:9315911}: 
 \begin{align*}
 \ell(\mathbf W; \x, \mathbf y) =  \|\mathbf W^{\top}\x-\mathbf y\|_2 = \max_{\|\alpha\|_2\leq 1} \alpha^{\top}(\mathbf W^{\top}\x-\mathbf y),
 \end{align*}
   where $\mathbf y\in\mathbb R^K$ is multiple class label vector and $\mathbf W=(\w_1,\cdots, \w_K)$. 
\end{itemize}
%More generally, any piecewise linear loss functions can be written in~(\ref{eqn:dual}) and~(\ref{eqn:f}). In the sequel, we denote by $L(\w, \alpha)=L(\w, \alpha; \mathbf X, \mathbf y)$  for simplicity, and by $G_\w(\w, \alpha)=\nabla_\w L(\w, \alpha)$ and $G_\alpha(\w,\alpha)=\nabla_\alpha L(\w, \alpha)$ the partial gradients of $L(\w, \alpha)$ in terms of $\w$ and $\alpha$, respectively.
%The following lemma shows an important property of the bilinear function $f(\w, \alpha; \x, y)$, 
%\begin{lemma}\label{lem:1}
% Let $f(\w,\alpha; \x, y)$ be bilinear in $\w$ and $\alpha$ as in~(\ref{eqn:f}). Assuming there exists $c>0$ such that $\|H(\x,y)\|^2_2\leq c,\forall \x, y$, then for any $\alpha_1, \alpha_2 \in\Q_\alpha$, and $\w_1,\w_2$ we have
%\begin{align}
%\|\nabla_{\alpha} f(\w_1,\alpha_1;\x, y) - \nabla_\alpha f(\w_2, \alpha_2; \x, y)\|^2_2&\leq c\|\w_1-\w_2\|^2_2\label{eqn:const}\\
%\|\nabla_{\w} f(\w_1,\alpha_1;\x, y) - \nabla_\w f(\w_2, \alpha_2; \x, y)\|^2_2&\leq c\|\alpha_1-\alpha_2\|^2_2\label{eqn:const2}
%\end{align}
%\end{lemma}

Besides the non-smooth loss function $\ell(\w; \x, y)$, we also assume that the regularizer $R(\w)$ is a non-smooth function. Many non-smooth regularizers are used in machine learning problems. We list a few of them in the following, where $\mathbf W=(\w_1,\cdots, \w_K)$, $\w_k\in\mathbb R^d$ and $\w^j$ is the $j$th row of $\mathbf W$.
\begin{itemize}
\item  lasso:  $R(\w)=\|\w\|_1$, $\ell_2$ norm: $R(\w)=\|\w\|_2$, and $\ell_{\infty}$ norm: $R(\w)=\|\w\|_\infty$.
\item group lasso: $R(\w)=\sum_{g=1}^K \sqrt{d_g}\|\w_g\|_2$, where $\w_g\in\mathbb R^{d_g}$.
\item exclusive lasso: $R(\mathbf W)= \sum_{j=1}^d \|\w^j\|_1^{2}$.
\item $\ell_{2,1}$ norm: $R(\mathbf W)= \sum_{j=1}^d \|\w^j\|_{2}$.
\item $\ell_{1,\infty}$ norm: $R(\mathbf W)=\sum_{j=1}^d \|\w^{j}\|_\infty$.
\item trace norm: $R(\mathbf W)=\|\mathbf W\|_1$, the summation of singular values of $\mathbf W$.
\item other regularizers: $R(\mathbf W)=\left(\sum_{k=1}^K\|\w_k\|_2\right)^2$.
\end{itemize}
%, i.e.
%\[
%\Pi_{\Q_1, \Q_2}\begin{pmatrix}\widehat\z_1\\ \widehat\z_2\end{pmatrix}= \arg\min_{\z_1\in\Q_1\atop \z_2\in\Q_2} \frac{1}{2}\|\z_1-\widehat \z_1\|_2^2 + \frac{1}{2}\|\z_2-\widehat \z_2\|_2^2
%\]
Note that unlike~\citep{Nesterov2005,Nesterov:2005:EGT:1081200.1085585}, we do not further require the non-smooth regularizer  to be written into a bilinear dual form, which could be violated by many non-smooth regularizers, e.g. $R(\mathbf W)=\left(\sum_{k=1}^K\|\w_k\|_2\right)^2$ or more generally $R(\w) =V(\|\w\|)$, where $V(z)$ is a monotonically increasing function. 

We close this section by presenting a lemma showing an important property of the bilinear function $L(\w, \balpha)$. 
\begin{lemma}\label{lem:1}
 Let $L(\w,\balpha)$ be bilinear in $\w$ and $\balpha$ as in~(\ref{eqn:f}). Given fixed $\X, \y$ there exists $c>0$ such that $\|H(\X,\y)\|^2_2\leq c$, then for any $\balpha_1, \balpha_2 \in\Q_{\balpha}$, and $\w_1,\w_2\in\Q_\w$ we have
\begin{align}
\|G_\alpha(\w_1,\balpha_1) - G_\alpha (\w_2, \balpha_2)\|^2_2&\leq c\|\w_1-\w_2\|^2_2\label{eqn:const},\\
\|G_{\w} (\w_1,\balpha_1) - G_\w (\w_2, \balpha_2)\|^2_2&\leq c\|\balpha_1-\balpha_2\|^2_2\label{eqn:const2}.
\end{align}
\end{lemma}
%With regard to Lemma~\ref{lem:1}, we need to make the following two remarks. 
%\begin{itemize}
%\item 
\paragraph{Remark 2}The value of constant $c$ in Lemma~\ref{lem:1} is an input to our algorithms  used to set the step size. In the Appendix \ref{app:constant}, we show how to estimate constant $c$ for certain loss functions.  In addition the constant $c$ in bounds~(\ref{eqn:const}) and~(\ref{eqn:const2}) do not have to be the same as shown by the the  example of generalized hinge loss in Appendix \ref{app:constant}. It should be noticed that the inequalities in Lemma~\ref{lem:1} indicate $L(\w,\balpha)$ has Liptschitz continuos gradients, however, the gradient of the whole objective with respect to $\w$, i.e.,  $G_\w(\w,\balpha)+\lambda \partial R(\w)$ is not Lipschitz continuous due to the general non-smooth term $R(\w)$, which prevents previous convex-concave minimization scheme~\citep{pau08,Nemirovski2005} not applicable. 

\begin{algorithm}[t]
\caption{The Pdprox-dual Algorithm for Non-Smooth Optimization}\label{alg:1}
\begin{algorithmic}[1]
\STATE \textbf{Input}: step size $\gamma = \sqrt{1/(2c)}$, where $c$ is specified in~(\ref{eqn:const}).
 \STATE \textbf{Initialization}: $\w_0=\mathbf 0, \bbeta_0=\mathbf 0$
\FOR{$t=1, 2, \ldots$}

  \STATE $\displaystyle \balpha_t=\Pi_{\mathcal Q_{\balpha}}\left[\bbeta_{t-1}+\gamma G_{\balpha}(\w_{t-1},\bbeta_{t-1})\right]$

   \STATE
$
   \w_t=\mathop{\arg\min}_{\w\in\Q_\w} \frac{1}{2}\left\|\w-\left(\w_{t-1}-\gamma G_\w(\w_{t-1},\balpha_t)\right)\right\|_2^2 +\gamma \lambda R(\w)
$
      \STATE $\displaystyle \bbeta_t=\Pi_{\mathcal Q_{\balpha}}\left[\bbeta_{t-1} + \gamma G_{\balpha}(\w_t, \balpha_t)\right]$
\ENDFOR
\STATE {\bf Output} $\widehat{\w}_T = \sum_{t=1}^T \w_t/T$ and $\widehat{\balpha}_T = \sum_{t=1}^T \balpha_t/ T$.
\end{algorithmic}
\end{algorithm}

\subsection{The Proposed Primal-Dual Prox Methods}
In this subsection, we present two variants  of Primal Dual Prox (Pdprox) method for solving the non-smooth optimization problem in~(\ref{eqn:p-obj}). The common feature shared by the two algorithms is that they update both the primal and the dual variables at each iteration. In contrast, most first order methods only update the primal variables. The key advantages of the proposed algorithms is that  they are able to capture the {\bf sparsity structures of both primal and dual variables}, which is usually the case when both the regularizer and the loss functions are both non-smooth. 
The two algorithms differ from each other in the number of  copies for the dual or the primal variables, and the specific order for updating those.  Although our analysis shows that the two algorithms share the same convergence rate; however,  our empirical studies show that the one algorithm is  more preferable than the other depending on the nature of the applications. 
%second algorithm is more advantageous when the number of training examples is significantly larger than the number of features (i.e. $n\gg d$).  

\paragraph{Pdprox-dual algorithm}Algorithm~\ref{alg:1} shows the first primal dual prox algorithm for optimizing the problem in~(\ref{eqn:p-obj}). Compared to the other gradient based algorithms, Algorithm~\ref{alg:1} has several interesting features: 
\begin{enumerate}
\item[(i)] it updates both the dual variable $\balpha$ and the primal variable $\w$. This is  useful when additional constraints are introduced for the dual variables, as we will discuss  later. 
\item[(ii)] it introduces an extra dual variable $\bbeta$ in addition to $\balpha$, and updates both $\balpha$ and $\bbeta$ at each iteration by a gradient mapping. The gradient mapping on the dual variables into a sparse domain allows the proposed algorithm to capture the sparsity of the dual
variables (more discussion on how the sparse constraint on the dual variable affects the
convergence is presented in Section~\ref{sec:ext}).  Compared to the second algorithm presented below, we refer to Algorithm~\ref{alg:1} as \textbf{Pdprox-dual} algorithm since it introduces an extra dual variable in updating.
\item[(iii)] the primal variable $\w$ is updated by a composite gradient mapping~\citep{RePEc:cor:louvco:2007076} in step 5.  Solving a composite gradient mapping in this step allows the proposed algorithm
to capture the sparsity of the primal variable. Similar to many other approaches for composite optimization~\citep{Duchi:2009:EOB:1577069.1755882,NIPS2009_0997}, we assume that the mapping in step 5 can be solved efficiently. (This is the only assumption we made on the non-smooth regularizer. The discussion in Section~\ref{sec:imp} shows that the proposed
algorithm can be applied to a large family of non-smooth regularizers).

\item[(iv)] the step size $\gamma$ is fixed to $\sqrt{1/(2c)}$, where $c$ is the constant specified in Lemma~\ref{lem:1}.
This is in contrast to most gradient based methods where the step size depends on $T$ and/or $\lambda$. This feature is particularly useful in implementation as we often observe that the performance of a gradient method is sensitive to the choice of the step size.
\end{enumerate}

%\subsection{The Second Algorithm: Pdprox-primal}\label{subsec:sec}
\paragraph{Pdprox-primal algorithm}In Algorithm~\ref{alg:1},  we maintain two copies of the dual variables $\balpha$ and $\bbeta$, and update them by two gradient mappings~\footnote{The extra gradient mapping on $\bbeta$ can also be replaced with a simple calculation, as discussed in subsection~\ref{sec:imp}. }. We can actually save one gradient mapping on the dual variable by first updating the primal variable $\w_t$, and then updating $\balpha_t$ using partial gradient computed with $\w_t$. As a tradeoff, we add an extra primal variable $\u$, and update it by a simple calculation. The detailed steps are shown in Algorithm~\ref{alg:2}.  Similar to Algorithm~\ref{alg:1}, Algorithm~\ref{alg:2} also needs to compute two partial gradients (except for the initial partial gradient on the primal variable), i.e.,  $G_\w(\cdot, \balpha_t)$ and $G_{\balpha}(\w_t, \cdot)$. Different from Algorithm~\ref{alg:1}, Algorithm~\ref{alg:2} (i) maintains $(\w_t, \balpha_t, \u_t)$ at each iteration with $O(2d+n)$  memory, while Algorithm~\ref{alg:1} maintains $(\balpha_t, \w_t, \bbeta_t)$ at each iteration with $O(2n+d)$ memory; (ii) and replaces one gradient mapping on an extra dual variable $\bbeta_t$ with a simple update on an extra primal variable $\u_t$.   Depending on the nature of applications, one method may be more efficient than the other. For example,  if the dimension $d$ is much larger than the number of examples $n$, then Algorithm~\ref{alg:1} would be more preferable than Algorithm~\ref{alg:2}. When the number of examples $n$ is much larger than the dimension $d$, Algorithm~\ref{alg:2} could save the memory and the computational cost.  However, as shown by our analysis in Section~\ref{sec:conv},  the convergence rate of two algorithms are the same. Because it introduces an extra primal variable, we refer to Algorithm~\ref{alg:2} as the \textbf{Pdprox-primal} algorithm.

\begin{remark}
It should be noted that  although Algorithm~\ref{alg:1} uses a similar strategy for updating the dual variables $\balpha$ and $\bbeta$, but it is significantly different from the mirror prox method~\citep{Nemirovski2005}. First, unlike the mirror prox method that introduces an auxiliary variable for $\w$, Algorithm~\ref{alg:1} introduces a composite gradient mapping for updating $\w$.  Second, Algorithm~\ref{alg:1} updates $\w_t$ using the partial gradient computed from the updated dual variable $\balpha_t$ rather than $\bbeta_{t-1}$. Third, Algorithm~\ref{alg:1} does not assume that the overall objective function has Lipschitz continuous gradients, a key assumption that limits the application of the mirror prox method.
\end{remark}

\begin{remark}
A similar algorithm with an extra primal variable is also proposed in a recent work~\citep{Chambolle:2011:FPA:1968993.1969036}. It is slightly different from Algorithm~\ref{alg:2} in the order of updating on the primal variable and the dual variable, and the gradients used in the updating. We discuss the differences between the Pdprox method and the algorithm in~\citep{Chambolle:2011:FPA:1968993.1969036} with our notations in Appendix~\ref{sec:apc}. 
%With a simple match between our notations and the notations in~\citep{Chambolle:2011:FPA:1968993.1969036},  it is starighfoward to show that the updates of Algorithm 1 in~\citep{Chambolle:2011:FPA:1968993.1969036} can be written as:
\end{remark}

%As compared to Algorithm~\ref{alg:1}, we refer to Algorithm~\ref{alg:2} that introduces an auxiliary primal variable in updating as \textbf{Pdprox-primal} algorithm.

\begin{algorithm}[t]
\caption{The Pdprox-primal Algorithm for Non-Smooth Optimization}\label{alg:2}
\begin{algorithmic}[1]
\STATE \textbf{Input}: step size $\gamma = \sqrt{1/(2c)}$, where $c$ is specified in~(\ref{eqn:const2}).
 \STATE \textbf{Initialization}: $\u_0=\mathbf 0, \balpha_0=\mathbf 0$
\FOR{$t=1, 2, \ldots$}
 
   \STATE
$
   \w_t=\mathop{\arg\min}_{\w\in\Q_\w} \frac{1}{2}\left\|\w-\left(\u_{t-1}-\gamma G_\w(\u_{t-1},\balpha_{t-1})\right)\right\|_2^2 +\gamma \lambda R(\w)
$

  \STATE $\displaystyle \balpha_t=\Pi_{\mathcal Q_{\balpha}}\left[\balpha_{t-1}+\gamma G_{\balpha}(\w_{t},\balpha_{t-1})\right]$

\STATE  $\u_t = \w_t +\gamma (G_\w(\u_{t-1}, \balpha_{t-1}) - G_\w(\w_t, \balpha_t))$
\ENDFOR
\STATE {\bf Output} $\widehat{\w}_T = \sum_{t=1}^T \w_t/T$ and $\widehat{\balpha}_T = \sum_{t=1}^T \balpha_t/ T$.
\end{algorithmic}
\end{algorithm}

\subsection{Convergence Analysis }\label{sec:conv}
This section establishes bounds on the convergence rate of the proposed algorithms. We begin by presenting  a theorem about  the convergence rate of Algorithms~\ref{alg:1}  and~\ref{alg:2}. For ease of analysis, we first write  (\ref{eqn:p-obj}) into the following equivalent minimax formulation
\begin{align}\label{eqn:obj}
\min_{\w\in\Q_\w} \max_{\balpha\in\Q_{\balpha}}\; F(\w, \balpha) = L(\w, \balpha)+ \lambda R(\w).
\end{align}
\break
Our main result is stated in the following theorem.
{\begin{theorem}\label{them:main}
%Let $(\w^*, \alpha^*)$ be the optimal solution to (\ref{eqn:obj}).
By running Algorithm~\ref{alg:1} or Algorithm~\ref{alg:2} with $T$ steps, we have
\begin{align*}
F(\widehat\w_T, \balpha)  - F(\w, \widehat\balpha_T) \leq  \frac{\|\w\|_2^2+\|\balpha\|^2_2}{\sqrt{(2/c)}T},
\end{align*}
for any $\w\in \Q_\w$ and $\balpha\in \Q_{\balpha}$. In particular, 
\begin{align*}
\mathcal L(\widehat\w_T) -  \mathcal D(\widehat\balpha_T) \leq \frac{\|\widetilde\w_T\|_2^2+\|\widetilde\balpha_T\|_2^2}{\sqrt{(2/c)}T}
\end{align*}
where $\mathcal D(\balpha)=\min_{\w\in\Q_\w}F(\w,\balpha)$ is the dual objective, 
%\begin{align*}
%\max_{\balpha\in\Q_{\balpha}}F(\widehat\w_T, \balpha)  - \min_{\w\in\Q_\w}F(\w, \widehat\balpha_T) \leq  \frac{\|\w^*\|_2^2+\|\balpha^*\|^2_2}{\sqrt{(2/c)}T},
%\end{align*}
%where $(\w^*, \alpha^*)$ is the optimal solution to $\max_{\balpha\in\Q_{\balpha}} F(\widehat\w_T, \balpha)-\min_{\w\in\Q_\w}F(\w,\widehat\balpha_T)$.%, and $D$ is the bound on $\|\w^*\|_2$ by $\|\w^*\|_2^2\leq D$.
and $\widetilde\w_T,\widetilde\balpha_T$ are given by $\widetilde\w_T=\arg\min_{\w\in\Q_\w}F(\w,\widehat\balpha_T)$, $\widetilde\balpha_T=\arg\max_{\balpha\in\Q_\alpha}F(\widehat\w_T, \balpha)$.
\end{theorem}}
 { \begin{remark}It is worth mentioning that in contrast to most previous studies whose convergence  rates are derived for the optimality of either the primal objective or the dual objective,  the convergence result  in Theorem~\ref{them:main}  is on the  duality gap, which can serve a certificate of the convergence  for the proposed algorithm.  It is not difficult to show that when $\Q_\w = \R^d$ the dual objective can be computed by 
 \[
 \mathcal D(\balpha) = c_0(\X, \y) + \balpha^{\top}\a(\X, \y) - \lambda R^*\left(\frac{-\b(\X, \y) - H(\X, \y)\balpha}{\lambda} \right)
 \]
 where $R^*(\u)$ is the convex conjugate of $R(\w)$. For example, if $R(\w) = 1/2\|\w\|_2^2$, $R^*(\u) = \frac{1}{2}\|\u\|_2^2$; if $R(\w) = \|\w\|_p$, $R^*(\u) = I(\|\u\|_q\leq 1)$,  where $I(\cdot)$ is an indicator function, $p=1,2,\infty$ and $1/p + 1/q=1$. 
\end{remark}
 }

Before proceeding to  the proof of Theorem~\ref{them:main}, we present the following Corollary that  follows immediately from Theorem~\ref{them:main}  and states the convergence bound for the objective $\mathcal L(\w)$ in~(\ref{eqn:p-obj}).  
%Finally, we state the convergence bound for the objective $\mathcal L(\w)$ in~(\ref{eqn:p-obj}) in the following corrolary. 
\begin{corollary}
Let $\w^*$ be the optimal solution to (\ref{eqn:p-obj}), bounded by $\|\w^*\|_2^2\leq D_1$, and $\|\balpha\|_2^2\leq D_2,\forall \balpha\in\Q_{\balpha}$. By running Algorithm~\ref{alg:1} or~\ref{alg:2} with $T$ iterations, we have
\begin{align*}
\mathcal L(\widehat{\w}_T) - \mathcal L(\w^*) \leq \frac{D_1 + D_2}{\sqrt{(2/c)}T}.
\end{align*}
\end{corollary}
\begin{proof}
Let $\w=\w^*=\arg\min_{\w\in\Q_\w}\mathcal L(\w)$ and $\widetilde\balpha_T=\arg\max_{\balpha\in\Q_{\balpha}}F(\widehat\w_T,\balpha)$  in Theorem~\ref{them:main}, 
then we have
\begin{align*}
\max_{\balpha\in\Q_{\balpha}}F(\widehat\w_T, \balpha) - F(\w^*, \widehat\balpha_T)\leq \frac{\|\w^*\|_2^2 +\|\widetilde\balpha_T\|_2^2}{\sqrt{(2/c)}T},
\end{align*}
Since $\mathcal L(\w) = \max\limits_{\balpha \in \Q_{\balpha}} F(\w, \balpha) \geq F(\w, \widehat{\balpha}_T)$, then we have
\begin{align*}
\mathcal L(\widehat \w_T) - \mathcal L(\w^*)\leq \frac{D_1 + D_2}{\sqrt{(2/c)}T}.
\end{align*}

\end{proof}

In order to aid understanding, we present the proof of Theorem~\ref{them:main} for each algorithm separately in the following subsections. 
\subsubsection{Convergence Analysis of Algorithm~\ref{alg:1}}

For the simplicity of analysis, we assume $\Q_\w = \mathbb R^d$ is the whole  Euclidean space. We discuss how to generalize the analysis to a convex domain $Q_\w$ in Section~\ref{sec:ext}.   In order to prove Theorem~\ref{them:main} for Algorithm~\ref{alg:1}, we present a series of lemmas to pave the path for the proof. We first restate the key updates in Algorithm~\ref{alg:1} as follows: 
\begin{align}
\balpha_t& = \Pi_{\Q_{\balpha}} \left[\bbeta_{t-1} + \gamma G_{\balpha} (\w_{t-1}, \bbeta_{t-1})\right] ,\label{eqn:upstep1}\\
\w_t & =\arg\min_{\w\in\mathbb R^d} \frac{1}{2}\|\w- (\w_{t-1}- \gamma G_\w(\w_{t-1},\balpha_t))\|_2^2+\gamma\lambda R(\w)\label{eqn:upstep2},\\
\bbeta_t&=\Pi_{\Q_{\balpha}}\left[\bbeta_{t-1} + \gamma G_{\balpha}(\w_t, \balpha_t)\right].\label{eqn:upstep3}
\end{align}

\begin{lemma}\label{lem:6}
The updates in Algorithm~\ref{alg:1} are equivalent to the following gradient mappings,
\begin{align*}
\begin{pmatrix}\balpha_t\\ \w_t\end{pmatrix}=\Pi_{\Q_{\balpha}, \mathbb R^d}\begin{pmatrix}\\ \bbeta_{t-1}+\gamma G_{\balpha}(\u_{t-1}, \bbeta_{t-1})\\ \u_{t-1} - \gamma( G_\w(\u_{t-1}, \balpha_t) +\lambda\mathbf v_t)\end{pmatrix},
\end{align*}
and
\begin{align*}
\begin{pmatrix}\bbeta_t\\ \u_t\end{pmatrix}=\Pi_{\Q_\alpha, \mathbb R^d}\begin{pmatrix}\\ \bbeta_{t-1}+\gamma G_{\balpha}(\w_{t}, \balpha_{t})\\ \u_{t-1} - \gamma( G_\w(\w_{t}, \balpha_t) + \lambda\mathbf v_t)\end{pmatrix},
\end{align*}
with initialization $\u_0=\w_0$, where $\mathbf v_t \in \partial R(\w_t)$ is a partial gradient of the regularizer on $\w_t$.
\end{lemma}
\begin{proof}
First, we argue that there exists a fixed (sub)gradient $\mathbf v_t \in \partial R(\w_t)$  such that the composite gradient mapping~(\ref{eqn:upstep2}) is equivalent to the following gradient mapping,
\begin{align}\label{eqn:proxmap}
\w_t &=\Pi_{\mathbb R^d}\left[\w_{t-1}- \gamma \left(G_\w(\w_{t-1}, \balpha_t)+ \lambda\mathbf v_t \right)\right].
\end{align}

To see this,  since $\w_t$ is the optimal solution to~(\ref{eqn:upstep2}), by first order optimality condition,  there exists a subgradient $\mathbf v_t=\partial R(\w_t)$ such that $\w_t - \w_{t-1} + \gamma G_\w(\w_{t-1}, \balpha_t)+ \gamma\lambda\mathbf  v_t =\boldsymbol0$, i.e. 
\[
\w_t=\w_{t-1}-\gamma G_\w(\w_{t-1}, \balpha_t) - \gamma\lambda \mathbf v_t,
\]
which is equivalent to~(\ref{eqn:proxmap}) since the projection $\Pi_{\mathbb R^d}$ is an identical mapping. 

%The reason is that by the first order  condition of optimality of convex programming,  since $\w_t$ is the optimal solution to~(\ref{eqn:upstep2}),  there exists a subgradient $\mathbf v_t=\partial R(\w_t)$ such that $\w_t - \w_{t-1} + \gamma G_\w(\w_{t-1}, \balpha_t)+ \gamma\lambda\mathbf  v_t=0$, i.e. 
%\[
%\w_t=\w_{t-1}-\gamma G_\w(\w_{t-1}, \balpha_t) - \gamma\lambda \mathbf v_t
%\]
%which is equivalent to~(\ref{eqn:proxmap}) since the projection $\Pi_{\mathbb R^d}$ is an identical mapping. 

Second,  the updates in Algorithm~\ref{alg:1} for $(\balpha, \bbeta, \w)$ are equivalent to the following updates for $(\balpha, \bbeta, \w, \u)$
\begin{align}
\balpha_t& = \Pi_{\Q_{\balpha}} \left[\bbeta_{t-1} + \gamma G_{\balpha} (\u_{t-1}, \bbeta_{t-1})\right] \nonumber,\\
\w_t &=\Pi_{\mathbb R^d}\left[\u_{t-1}- \gamma \left(G_\w(\u_{t-1}, \balpha_t)+ \lambda\mathbf v_t \right)\right],\label{eqn:w}\\
\bbeta_t&=\Pi_{\Q_{\balpha}}\left[\bbeta_{t-1} + \gamma G_{\balpha}(\w_t, \balpha_t)\right],\nonumber\\
\u_t & =\w_{t} - \gamma (G_\w(\w_{t}, \balpha_t) - G_\w(\u_{t-1}, \balpha_t)),\label{eqn:u}
\end{align}
with initialization $\u_0=\w_0$. The reason is because  $\u_t=\w_t, t=1,\cdots$ due to $G_\w(\w_{t}, \balpha_t) = G_\w(\u_{t-1}, \balpha_t)$, where we use the fact that $L(\w,\balpha)$ is linear in $\w$.

%The updates  for $\balpha_t, \bbeta_t$ specified in Lemma~\ref{lem:6} are the same to Lemma~\ref{lem:5}.  The update for $\w_t$ follows Lemma~\ref{lem:4}.  In terms of updating  $\u_t$, we note that 
%\begin{align}\label{eqn:w}
%\w_t  = \u_{t-1} - \gamma( G_\w(\u_{t-1}, \balpha_t) +\lambda\mathbf v_t)
%\end{align}
Finally, by plugging~(\ref{eqn:w}) for $\w_t$ into the update for $\u_t$ in~(\ref{eqn:u}), we complete the proof of Lemma~\ref{lem:6}.
\end{proof}

The reason that we translate the updates for $(\balpha_t, \w_t, \bbeta_t)$ in Algorithm~\ref{alg:1} into the updates for $(\balpha_t, \w_t, \bbeta_t, \u_t)$ in  Lemma~{\ref{lem:6}} is because it allows us to fit the updates for $(\balpha_t, \w_t, \bbeta_t, \u_t)$ into Lemma~\ref{lem:8} as presented in Appendix \ref{app:lem:7}, which leads us to a key inequality as stated in Lemma~\ref{lem:7} to prove Theorem~\ref{them:main}.
\begin{lemma}\label{lem:7} For all $t=1, 2, \cdots$, and any $\w\in\mathbb R^d, \balpha\in\Q_{\balpha}$, we have
\begin{align*}
&\gamma\begin{pmatrix} G_\w(\w_t, \balpha_t) +\lambda \mathbf v_t\\-G_{\balpha} (\w_t, \balpha_t)\end{pmatrix}^{\top}\begin{pmatrix}\w_t-\w\\ \balpha_t-\balpha\end{pmatrix}\leq \frac{1}{2}\left\|\begin{pmatrix}\w-\u_{t-1}\\ \balpha-\bbeta_{t-1}\end{pmatrix}\right\|_2^2 - \frac{1}{2}\left\|\begin{pmatrix}\w-\u_t\\ \balpha -\bbeta_t\end{pmatrix}\right\|_2^2\\
&\hspace*{1.2in}+{\gamma^2}\left\|G_{\balpha}(\w_t, \balpha_t)- G_{\balpha}(\u_{t-1}, \bbeta_{t-1})\right\|_2^2- \frac{1}{2}\|\w_t-\u_{t-1}\|_2^2.
%&+ \frac{\gamma^2}{2}\left \|\frac{1}{n}\overline{\X}\w_{t}-\frac{1}{n}\overline{\X}\z_{t-1}\right\|^2
%&
%&\leq \frac{1}{2}\left\|\begin{pmatrix}\w-\z_{t-1}\\ u-v_{t-1}\end{pmatrix}\right\|^2 - \frac{1}{2}\left\|\begin{pmatrix}\w-\z_t\\ u -v_t\end{pmatrix}\right\|^2 \\
%&+ \frac{\gamma^2}{2}\sum_i\left\|\frac{1}{n}y_i\x_i^{\top}\left(\w_t-\z_{t-1}\right)\right\|^2 - \frac{1}{2}\|\w_t-\z_{t-1}\|^2\\
%&\leq \frac{1}{2}\left\|\begin{pmatrix}\w-\z_{t-1}\\ u-v_{t-1}\end{pmatrix}\right\|^2 - \frac{1}{2}\left\|\begin{pmatrix}\w-\z_t\\ u -v_t\end{pmatrix}\right\|^2 \\
%&+ \frac{\gamma^2 R^2}{2n}\left\|\left(\w_t-\z_{t-1}\right)\right\|^2 - \frac{1}{2}\|\w_t-\z_{t-1}\|^2\\
\end{align*}
\end{lemma}
The proof of Lemma~\ref{lem:7} is deferred to Appendix \ref{app:lem:7}. We are now ready to prove the main  theorem for Algorithm~\ref{alg:1}. 
\begin{proof}[of Theorem~\ref{them:main} for Algorithm~\ref{alg:1}]
Since $F(\w,\balpha)$ is convex in $\w$ and concave in $\balpha$, we have
\begin{align*}
&F(\w_t,\balpha_t) - F(\w,\balpha_t) \leq (G_\w(\w_t, \balpha_t)+\lambda\mathbf v_t)^{\top}(\w_t-\w),\\
&F(\w_t,\balpha) - F(\w_t, \balpha_t)\leq -G_{\balpha}(\w_t, \balpha_t)^{\top}(\balpha_t-\balpha),
\end{align*}
where $\mathbf v_t \in \partial R(\w_t)$ is the partial gradient of $R(\w)$ on $\w_t$ stated in Lemma~\ref{lem:6}.
Combining the above inequalities with Lemma~\ref{lem:7}, we have
\begin{align*}
&\gamma \left( F(\w_t, \balpha_t) - F(\w, \balpha_t) + F(\w_t, \balpha) - F(\w_t, \balpha_t) \right)\\
%&\leq \frac{1}{2}\left\|\begin{pmatrix}\w-\u_{t-1}\\ \alpha-\beta_{t-1}\end{pmatrix}\right\|_2^2 - \frac{1}{2}\left\|\begin{pmatrix}\w-\u_t\\ \alpha -\beta_t\end{pmatrix}\right\|_2^2 \\
%&+{\gamma^2}\left\|G_\alpha(\w_t, \alpha_t)- G_\alpha(\u_{t-1}, \beta_{t-1})\right\|_2^2\\
%&- \frac{1}{2}\|\w_t-\u_{t-1}\|_2^2\\
&\leq \frac{1}{2}\left\|\begin{pmatrix}\w-\u_{t-1}\\ \balpha-\bbeta_{t-1}\end{pmatrix}\right\|_2^2 - \frac{1}{2}\left\|\begin{pmatrix}\w-\u_t\\ \balpha -\bbeta_t\end{pmatrix}\right\|_2^2 + {\gamma^2}\|G_{\balpha}(\w_t, \balpha_t) - G_{\balpha}(\u_{t-1}, \bbeta_{t-1})\|_2^2 \\
&\hspace*{0.1in}- \frac{1}{2}\|\w_t-\u_{t-1}\|_2^2\\
&\leq\frac{1}{2}\left\|\begin{pmatrix}\w-\u_{t-1}\\ \balpha-\bbeta_{t-1}\end{pmatrix}\right\|_2^2 - \frac{1}{2}\left\|\begin{pmatrix}\w-\u_t\\ \balpha -\bbeta_t\end{pmatrix}\right\|_2^2+ {\gamma^2}c\|\w_t-\u_{t-1}\|_2^2 - \frac{1}{2}\|\w_t-\u_{t-1}\|_2^2\\
&\leq \frac{1}{2}\left\|\begin{pmatrix}\w-\u_{t-1}\\ \balpha-\bbeta_{t-1}\end{pmatrix}\right\|_2^2 - \frac{1}{2}\left\|\begin{pmatrix}\w-\u_t\\ \balpha -\bbeta_t\end{pmatrix}\right\|_2^2,
\end{align*}
where the second inequality follows the inequality~(\ref{eqn:const}) in Lemma~\ref{lem:1} and the fact $\gamma=\sqrt{1/(2c)}$. By adding the inequalities of all iterations and dividing both sides by $T$, we have
\begin{equation}\label{eqn:bound}
\frac{1}{T}\sum_{t=1}^T \left(F(\w_t, \balpha) - F(\w, \balpha_t)\right) \leq \frac{\|\w\|_2^2+\|\balpha\|_2^2}{\sqrt{(2/c)}\;T}.
\end{equation}
We complete the proof by using the definitions of $\widehat\w_T, \widehat\balpha_T$, and the convexity-concavity of $F(\w,\balpha)$ with respect to $\w$ and $\balpha$, respectively. 
\end{proof}
%Followed by Lemma~\ref{lem:7}, and  that  $F(\w, \alpha)$ is convex in $\w$ and concave in $\alpha$, we have
%\begin{align*}
%&\gamma \left( F(\w_t, \alpha_t) - F(\w, \alpha_t) + F(\w_t, \alpha) - F(\w_t, \alpha_t) \right)\\
%%&\leq \frac{1}{2}\left\|\begin{pmatrix}\w-\u_{t-1}\\ \alpha-\beta_{t-1}\end{pmatrix}\right\|_2^2 - \frac{1}{2}\left\|\begin{pmatrix}\w-\u_t\\ \alpha -\beta_t\end{pmatrix}\right\|_2^2 \\
%%&+{\gamma^2}\left\|G_\alpha(\w_t, \alpha_t)- G_\alpha(\u_{t-1}, \beta_{t-1})\right\|_2^2\\
%%&- \frac{1}{2}\|\w_t-\u_{t-1}\|_2^2\\
%&\leq \frac{1}{2}\left\|\begin{pmatrix}\w-\u_{t-1}\\ \alpha-\beta_{t-1}\end{pmatrix}\right\|_2^2 - \frac{1}{2}\left\|\begin{pmatrix}\w-\u_t\\ \alpha -\beta_t\end{pmatrix}\right\|_2^2 \\
%&+ {\gamma^2}\sum_{i=1}^n\left |\frac{1}{n}\left(\nabla f_\alpha(\w_t,\alpha^t_i;\x_i, y_i)-\nabla f_\alpha(\u_{t-1},\beta^{t-1}_i;\x_i, y_i)\right)\right|^2 - \frac{1}{2}\|\w_t-\u_{t-1}\|_2^2\\
%&\leq \frac{1}{2}\left\|\begin{pmatrix}\w-\u_{t-1}\\ \alpha-\beta_{t-1}\end{pmatrix}\right\|_2^2 - \frac{1}{2}\left\|\begin{pmatrix}\w-\u_t\\ \alpha -\beta_t\end{pmatrix}\right\|_2^2
%\end{align*}
%where the last step follows the inequality in~(\ref{eqn:const}) and the fact $\gamma=\sqrt{n/(2c)}$. By adding the inequalities of all iterations, we have
%\begin{equation}\label{eqn:bound}
%\sum_{t=1}^T \left(F(\w_t, \alpha) - F(\w, \alpha_t)\right) \leq \frac{\|\w\|_2^2+\|\alpha\|_2^2}{\sqrt{(2n/c)} T}
%\end{equation}
%We complete the proof by using the definitions of $\widehat\w_T, \widehat\alpha_T$, and the convexity, concavity of $F(\w,\alpha)$.
%\end{proof}
\subsubsection{Convergence Analysis of Algorithm~\ref{alg:2}}
We can prove the convergence bound for Algorithm~\ref{alg:2} by following the same path. In the following we present the key lemmas similar to Lemmas~\ref{lem:6} and \ref{lem:7}, with proofs omitted.
\begin{lemma}\label{lem:11}
There exists a fixed partial gradient $\mathbf v_t \in \partial R(\w_t)$ such that the updates in Algorithm~\ref{alg:2} are equivalent to the following gradient mappings,
\begin{align*}
\begin{pmatrix}\w_t\\\balpha_t\end{pmatrix}=\Pi_{\mathbb R^d,\Q_{\balpha}}\begin{pmatrix}  \u_{t-1} - \gamma( G_\w(\u_{t-1}, \bbeta_{t-1}) +\lambda\mathbf v_t)\\\bbeta_{t-1}+\gamma G_{\balpha}(\w_{t}, \bbeta_{t-1})\end{pmatrix}
\end{align*}
and
\begin{align*}
\hspace*{-0.2in}\begin{pmatrix} \u_t\\\bbeta_t\end{pmatrix}=\Pi_{\mathbb R^d,\Q_{\balpha}}\begin{pmatrix} \u_{t-1} - \gamma( G_\w(\w_{t}, \balpha_t) + \lambda\mathbf v_t)\\\bbeta_{t-1}+\gamma G_{\balpha}(\w_{t}, \balpha_{t})\end{pmatrix},
\end{align*}
with initialization $\bbeta_0=\balpha_0$.
\end{lemma}
\begin{lemma}\label{lem:12} For all $t=1, 2, \cdots$,  and any $\w\in\mathbb R^d, \balpha\in\Q_\alpha$, we have
\begin{align*}
&\gamma\begin{pmatrix} G_\w(\w_t, \balpha_t) +\lambda \mathbf v_t\\-G_{\balpha} (\w_t, \balpha_t)\end{pmatrix}^{\top}\begin{pmatrix}\w_t-\w\\ \balpha_t-\balpha\end{pmatrix}\leq \frac{1}{2}\left\|\begin{pmatrix}\w-\u_{t-1}\\ \balpha-\bbeta_{t-1}\end{pmatrix}\right\|_2^2 - \frac{1}{2}\left\|\begin{pmatrix}\w-\u_t\\ \balpha -\bbeta_t\end{pmatrix}\right\|_2^2\\
&\hspace*{1.2in}+{\gamma^2}\left\|G_\w(\w_t, \balpha_t)- G_\w(\u_{t-1}, \bbeta_{t-1})\right\|_2^2- \frac{1}{2}\|\balpha_t-\bbeta_{t-1}\|_2^2.
%&+ \frac{\gamma^2}{2}\left \|\frac{1}{n}\overline{\X}\w_{t}-\frac{1}{n}\overline{\X}\z_{t-1}\right\|^2
%&
%&\leq \frac{1}{2}\left\|\begin{pmatrix}\w-\z_{t-1}\\ u-v_{t-1}\end{pmatrix}\right\|^2 - \frac{1}{2}\left\|\begin{pmatrix}\w-\z_t\\ u -v_t\end{pmatrix}\right\|^2 \\
%&+ \frac{\gamma^2}{2}\sum_i\left\|\frac{1}{n}y_i\x_i^{\top}\left(\w_t-\z_{t-1}\right)\right\|^2 - \frac{1}{2}\|\w_t-\z_{t-1}\|^2\\
%&\leq \frac{1}{2}\left\|\begin{pmatrix}\w-\z_{t-1}\\ u-v_{t-1}\end{pmatrix}\right\|^2 - \frac{1}{2}\left\|\begin{pmatrix}\w-\z_t\\ u -v_t\end{pmatrix}\right\|^2 \\
%&+ \frac{\gamma^2 R^2}{2n}\left\|\left(\w_t-\z_{t-1}\right)\right\|^2 - \frac{1}{2}\|\w_t-\z_{t-1}\|^2\\
\end{align*}
\end{lemma}
%\begin{theorem}\label{them:main2}
%%Let $(\w^*, \alpha^*)$ be the optimal solution to (\ref{eqn:obj}).
%By running Algorithm~\ref{alg:2} with $T$ steps, we have
%\begin{align*}
%F(\widehat\w_T, \balpha)  -  F(\w, \widehat\balpha_T) \leq  \frac{\|\w\|_2^2+\|\balpha\|^2_2}{\sqrt{(2/c)}T}
%\end{align*}
%for any $\w\in\mathbb R^d$ and $\balpha\in\Q_{\balpha}$.
%%where $(\w, \alpha^*)$ is the optimal solution to $\max_{\alpha\in\Q_\alpha} F(\widehat\w_T, \alpha)-\min_{\w}F(\w,\widehat\alpha_T)$.%, and $D$ is the bound on $\|\w^*\|_2$ by $\|\w^*\|_2^2\leq D$.
%\end{theorem}
\begin{proof}[of Theorem~\ref{them:main} for Algorithm~\ref{alg:2}]
Similar to proof of Theorem~\ref{them:main} for Algorithm~\ref{alg:1},  we have
\begin{align*}
&\gamma \left( F(\w_t, \balpha_t) - F(\w, \balpha_t) + F(\w_t, \balpha) - F(\w_t, \balpha_t) \right)\\
%&\leq \frac{1}{2}\left\|\begin{pmatrix}\w-\u_{t-1}\\ \alpha-\beta_{t-1}\end{pmatrix}\right\|_2^2 - \frac{1}{2}\left\|\begin{pmatrix}\w-\u_t\\ \alpha -\beta_t\end{pmatrix}\right\|_2^2 \\
%&+{\gamma^2}\left\|G_\alpha(\w_t, \alpha_t)- G_\alpha(\u_{t-1}, \beta_{t-1})\right\|_2^2\\
%&- \frac{1}{2}\|\w_t-\u_{t-1}\|_2^2\\
&\leq \frac{1}{2}\left\|\begin{pmatrix}\w-\u_{t-1}\\ \balpha-\bbeta_{t-1}\end{pmatrix}\right\|_2^2 - \frac{1}{2}\left\|\begin{pmatrix}\w-\u_t\\ \balpha -\bbeta_t\end{pmatrix}\right\|_2^2 + {\gamma^2}\left\|G_\w(\w_t, \balpha_t)- G_\w(\u_{t-1}, \bbeta_{t-1})\right\|_2^2\\
&- \frac{1}{2}\|\balpha_t-\bbeta_{t-1}\|_2^2\\
&\leq \frac{1}{2}\left\|\begin{pmatrix}\w-\u_{t-1}\\ \balpha-\bbeta_{t-1}\end{pmatrix}\right\|_2^2 - \frac{1}{2}\left\|\begin{pmatrix}\w-\u_t\\ \balpha -\bbeta_t\end{pmatrix}\right\|_2^2 + {\gamma^2}c\|\balpha_t-\bbeta_{t-1}\|_2^2 - \frac{1}{2}\|\balpha_t-\bbeta_{t-1}\|_2^2\\
&\leq \frac{1}{2}\left\|\begin{pmatrix}\w-\u_{t-1}\\ \balpha-\bbeta_{t-1}\end{pmatrix}\right\|_2^2 - \frac{1}{2}\left\|\begin{pmatrix}\w-\u_t\\ \balpha -\bbeta_t\end{pmatrix}\right\|_2^2,
\end{align*}
where the last step follows the inequality~(\ref{eqn:const2}) in Lemma~\ref{lem:1} and the fact $\gamma=\sqrt{1/(2c)}$. By adding the inequalities of all iterations and dividing both sides by $T$, we have
\begin{equation}\label{eqn:bound}
\frac{1}{T}\sum_{t=1}^T \left(F(\w_t, \balpha) - F(\w, \balpha_t)\right) \leq \frac{\|\w\|_2^2+\|\balpha\|_2^2}{\sqrt{(2/c)} \;T}.
\end{equation}
We complete the proof by using the definitions of $\widehat\w_T, \widehat\balpha_T$, and the convexity-concavity of $F(\w,\balpha)$ with respect to $\w$ and $\balpha$, respectively.
\end{proof}

\paragraph{Comparison with Pegasos on $\ell^2_2$ regularizer} We compare the proposed algorithm to the Pegasos algorithm~\citep{Shalev-Shwartz:2007:PPE:1273496.1273598}~\footnote{We compare to the deterministic Pegasos that computes the gradient using all examples at each iteration. It would be criticized  that it is not fair to compare with Pegasos since it is a stochastic algorithm, however, such a comparison (both theoretically and empirically) would provide a formal evidence that solving the min-max problem by a primal dual method  with an extra-gradient may yield better convergence  than solving the  primal problem. %In addition, our empirical studies favor the primal dual prox method over Pegasos (either in stochastic version or in batch version) on medium-sized data.
} for minimizing the $\ell_2^2$ regularized hinge loss. Although in this case both algorithms achieve a convergence rate of $O(1/T)$, their dependence on the regularization parameter $\lambda$ is very different. In particular, the convergence rate of the proposed algorithm is $O\left(\frac{(1 + n\lambda)R}{\sqrt{2n}\lambda T}\right)$ by noting that  $\|\w^*\|_2^2 =O(1/\lambda)$, $\|\balpha^*\|^2_2\leq \|\balpha^*\|_1\leq n$, and $c=R^2/n$, while the Pegasos algorithm has a convergence rate of $\widetilde O\left(\frac{(\sqrt{\lambda}+R)^2}{\lambda T}\right)$, where $\widetilde O(\cdot)$ suppresses a logarithmic term $
\ln(T)$. According to the common assumption of learning theory~\citep{Wu:2005:SSM:1118516.1118524,citeulike:416064}, the optimal $\lambda$ is $O(n^{-1/(\tau + 1)})$ if the probability measure can be approximated by the closure of RKHS $\mathcal H_{\kappa}$ with exponent $0 < \tau \leq 1$. As a result, the convergence rate of the proposed algorithm is $O(\sqrt{n}R/T)$ while the convergence rate of Pegasos is $O(n^{1/(1+\tau)}R^2/T)$. Since $\tau \in (0, 1]$, the proposed algorithm could be more efficient than the Pegasos algorithm, particularly when $\lambda$ is sufficiently small. This is verified by our empirical studies in section~\ref{sec:exp5} (see Figure~\ref{fig:3}). {It is also interesting to note that the convergence rate of Pdprox has a better dependence on $R$,  the $\ell_2$ norm bound of examples $\|\x\|_2\leq R$, compared to $R^2$ in the convergence rate of Pegasos. Finally, we mention that the proposed algorithm is a deterministic algorithm that requires a full pass of all training examples at each iteration, while Pegasos can be purely stochastic by sampling one example for computing the sub-gradient, which maintains the same convergence rate. It  remains an interesting and open problem to extend the Pdprox algorithm to its stochastic or randomized version with a similar convergence rate.}

\subsection{Implementation Issues}\label{sec:imp}
{In this subsection, we discuss some implementation issues: (1) how to efficiently solve the optimization problems for updating the primal and dual variables in Algorithms~\ref{alg:1} and~\ref{alg:2}; (2) how to set a good step size; and (3) how to implement the algorithms efficiently. }

Both $\balpha$ and $\bbeta$ are updated by a gradient mapping that requires computing the projection into the domain $\Q_{\balpha}$. When $\mathcal Q_{\balpha}$ is only consisted of box constraints (e.g., hinge loss, absolute loss, and $\epsilon$-insensitive loss), the projection $\prod_{\Q_\alpha}[\widehat \alpha]$ can be computed by thresholding. When $\Q_{\balpha}$ is comprised of both box constraints and a linear constraint (e.g., generalized hinge loss), the following lemma gives an efficient algorithm for computing $\prod_{\Q_{\balpha}}[\widehat\balpha]$.

\begin{lemma}\label{lem:2}
For $\Q_{\balpha}=\{\balpha: \balpha\in[0, s]^n,\balpha^{\top}\mathbf v\leq \rho\}$, the optimal solution $\balpha^*$ to projection $\prod_{\Q_{\balpha}}[\widehat \balpha]$ is computed by 
\[
    \alpha^*_i= [\widehat\alpha_i-\eta v_i]_{[0, s]}, \forall i \in [n], 
\]
where $\eta=0$ if $\sum_i[\widehat\alpha_i]_{[0, s]}v_i\leq \rho$ and otherwise is the solution to the following equation
\begin{align}\label{eqn:bin}
\sum_i [\widehat\alpha_i-\eta v_i]_{[0, s]}v_i-\rho=0.
\end{align}
Since $\sum_i [\widehat\alpha_i-\eta v_i]_{[0, s]}v_i-1$ is monotonically decreasing in $\eta$, we can solve $\eta$ in~(\ref{eqn:bin}) by a bi-section search.
\end{lemma}
{\begin{remark}
It is notable that when the domain is a simplex type domain, i.e. $\sum_i\alpha_i\leq \rho$, Duchi et al.~\citep{Duchi:2008:EPL:1390156.1390191} has proposed  more efficient algorithms for solving the projection problem. 
\end{remark}}

%\paragraph{Replace the extra gradient mapping in Algorithm~\ref{alg:1} with a simple calculation}\label{sec:betaproj}
{Moreover, we can further improve the efficiency of Algorithm~\ref{alg:1} by removing the gradient mapping on $\bbeta$. The key idea is similar to the analysis provided in subsection~\ref{sec:ext} for arguing that the convergence rate presented in Theorem~\ref{them:main} for Algorithm~\ref{alg:2} holds for any convex domain $\Q_\w$. Actually, the update on $\balpha$ is equivalent to 
%Using the above analysis, we can replace the extra gradient mapping on $\bbeta$ in Algorithm~\ref{alg:1} with a simple calculation. Since the update on  $\balpha$ is  equivalent to 
\begin{align*}
\balpha_t &=  \arg\min_{\balpha} \frac{1}{2}\|\balpha - (\bbeta_{t-1} +  \gamma G_{\balpha}(\w_{t-1}, \bbeta_{t-1}))\|_2^2 + \gamma Q(\balpha),
\end{align*}
which together with the first order optimality condition implies
\begin{align*}
\balpha_t &=  \bbeta_{t-1} +  \gamma G_{\balpha}(\w_{t-1}, \bbeta_{t-1}) -  \gamma \partial Q(\balpha_t),
\end{align*}
where 
\begin{align*}
Q(\balpha) =\left\{ \begin{array}{cc}0&\balpha\in\Q_{\balpha}\\ +\infty&\text{ otherwise }\end{array}\right.,
\end{align*}
is the indicator function of the domain $\Q_{\balpha}$. 
Then we can update the $\beta_t$ by 
\begin{align*}
\bbeta_t &=  \arg\min_{\balpha} \frac{1}{2}\|\balpha - (\bbeta_{t-1} +  \gamma G_{\balpha}(\w_{t}, \balpha_{t}) - \partial Q(\balpha_t))\|_2^2,\\
& =   \bbeta_{t-1} +  \gamma G_{\balpha}(\w_{t}, \balpha_{t}) - \partial Q(\balpha_t)
\end{align*}
which can be computed simply by 
\begin{align*}
\bbeta_t &=  \balpha_t +  \gamma (G_{\balpha}(\w_{t}, \balpha_{t}) - G_{\balpha}(\w_{t-1}, \bbeta_{t-1})). 
\end{align*}
The new Pdprox-dual algorithm is presented in Algorithm~\ref{alg:3}.  To prove the convergence rate of Algorithm~\ref{alg:3}, we can follow the same analysis to first prove the duality gap  for  $L(\w, \balpha) + \lambda R(\w) - Q(\balpha)$ and then absorb $\Q(\balpha)$ into the domain constraint of $\balpha$. The convergence result presented in Theorem~\ref{them:main} holds the same for Algorithm~\ref{alg:3}. 
\begin{algorithm}[t]
\caption{The Pdprox-dual Algorithm for Non-Smooth Optimization}\label{alg:3}
\begin{algorithmic}[1]
\STATE \textbf{Input}: step size $\gamma = \sqrt{1/(2c)}$, where $c$ is specified in~(\ref{eqn:const}).
 \STATE \textbf{Initialization}: $\w_0=\mathbf 0, \bbeta_0=\mathbf 0$
\FOR{$t=1, 2, \ldots$}

  \STATE $\displaystyle \balpha_t=\Pi_{\mathcal Q_{\balpha}}\left[\bbeta_{t-1}+\gamma G_{\balpha}(\w_{t-1},\bbeta_{t-1})\right]$

   \STATE
$
   \w_t=\mathop{\arg\min}_{\w\in\Q_\w} \frac{1}{2}\left\|\w-\left(\w_{t-1}-\gamma G_\w(\w_{t-1},\balpha_t)\right)\right\|_2^2 +\gamma \lambda R(\w)
$
      \STATE $\displaystyle \bbeta_t=\balpha_t +  \gamma (G_{\balpha}(\w_{t}, \balpha_{t}) - G_{\balpha}(\w_{t-1}, \bbeta_{t-1})) $
\ENDFOR
\STATE {\bf Output} $\widehat{\w}_T = \sum_{t=1}^T \w_t/T$ and $\widehat{\balpha}_T = \sum_{t=1}^T \balpha_t/ T$.
\end{algorithmic}
\end{algorithm}}
\begin{remark}
In Appendix~\ref{sec:apc}, we show that the updates on $(\w_t, \balpha_t)$ of Algorithm~\ref{alg:3} are essentially  the same  to the Algorithm 1 in~\citep{Chambolle:2011:FPA:1968993.1969036}, if we remove the extra dual variable  in Algorithm~\ref{alg:3} and the extra primal variable in  Algorithm 1 in~\citep{Chambolle:2011:FPA:1968993.1969036}. However, the difference is that  in Algorithm~\ref{alg:3}, we maintain two dual variables and one primal variable at each iteration, while the Algorithm 1 in~\citep{Chambolle:2011:FPA:1968993.1969036} maintains two primal variables and one dual variable at each iteration. 

%A similar algorithm with an extra primal variable is also proposed in a recent work~\citep{Chambolle:2011:FPA:1968993.1969036}. It is slightly different from Algorithm~\ref{alg:2} in the order of updating on the primal variable and the dual variable, and the gradients used in the updating. We present the  updates in~\citep{Chambolle:2011:FPA:1968993.1969036} with our notations in Appendix~\ref{sec:apc}. 
%With a simple match between our notations and the notations in~\citep{Chambolle:2011:FPA:1968993.1969036},  it is starighfoward to show that the updates of Algorithm 1 in~\citep{Chambolle:2011:FPA:1968993.1969036} can be written as:
\end{remark}

For the composite gradient mapping for $\w\in \Q_\w=\mathbb R^d$, there is a closed form solution for simple regularizers (e.g., $\ell_1,\ell_2$)  and decomposable regularizers (e.g., $\ell_{1,2}$). Efficient algorithms are available for composite gradient mapping when the regularizer is the $\ell_{\infty}$ and $\ell_{1,\infty}$, or trace norm. More details can be found in~\citep{Duchi:2009:EOB:1577069.1755882,Ji2009}. Here we present an efficient solution to a general regularizer $V(\|\w\|)$, where $\|\w\|$ is either a simple regularizer (e.g., $\ell_1$, $\ell_2$, and $\ell_{\infty}$) or a decomposable regularizer (e.g., $\ell_{1,2}$ and $\ell_{1, \infty}$), and $V(z)$ is convex and monotonically increasing for $z \geq 0$. An example is $V(\|\w\|)=(\sum_{k}\|\w_k\|_2)^2$, where $\w_1, \ldots, \w_K$ forms a partition of $\w$.
\begin{lemma}\label{lem:gr}
Let $V_*(\eta)$ be the convex conjugate of $V(z)$, i.e. $V(z) = \max_{\eta}\eta z- V_*(\eta)$. Then the solution to the composite mapping
\begin{align*}
\w^*=\arg\min_{\w\in\Q_\w}\frac{1}{2}\|\w-\widehat \w\|_2^2  + \lambda V(\|\w\|),
\end{align*}
can be computed by 
\begin{align*}
\w^*=\arg\min_{\w\in\Q_\w}\frac{1}{2}\|\w-\widehat \w\|_2^2  + \lambda\eta \|\w\|, 
\end{align*}
where $\eta$ satisfies $\|\w^*\|-V'_*(\eta)=0.$ Since both $\|\w^*\|$ and $-V'_*(\eta)$ are non-increasing functions in $\eta$, we can efficiently compute $\eta$ by a bi-section search.
\end{lemma}

{The value of the step size $\gamma$ in Algorithms~\ref{alg:2} and~\ref{alg:3} depends on the value of $c$, a constant that upper bounds the spectral norm square  of the matrix $H(\mathbf X, \mathbf y)$.   In many machine learning applications, by assuming a bound on the data (e.g., $\|\x\|_2\leq R$), one can easily compute an estimate of  $c$. We present derivations of the constant $c$ for hinge loss and generalized hinge loss in Appendix~\ref{app:constant}. However, the computed  value of $c$ might be  overestimated, thus the step size $\gamma$ is underestimated. Therefore,  to improve the empirical performances, one can scale up the estimated value of $\gamma$ by a factor larger than one and choose the best factor by tuning among a set of values. In addition, the authors in~\citep{Chambolle:2011:FPA:1968993.1969036} suggested a two step sizes scheme with $\tau$ for updating the primal variable and $\sigma$ for updating the dual variable. Depending on the nature of applications, one may observe better performances by carefully choosing the ratio between the two step sizes provided that $\sigma$ and $\tau$ satisfy $\sigma\tau\leq 1/c$. In the last subsection, we observe the improved performance for solving SVM by using the two step sizes scheme and by carefully tuning the ratio between the two step sizes. Furthermore, \citep{Pock:2011:DPF:2355573.2356473} presents a technique for computing diagonal preconditioners in the cases when estimating the value of $c$ is difficult for complex problems, and applies it to general linear programing problems and some computer vision problems. 

Finally, we discuss the two implementation schemes for Algorithms~\ref{alg:2} and~\ref{alg:3}.  Note that in Algorithm~\ref{alg:2}, we maintain and update two primal variables $\w_t, \u_t\in\R^d$, while in Algorithm~\ref{alg:3} we maintain and update two dual variables $\balpha_t, \bbeta_t\in\R^n$. We refer to the implementation with two primal variables as double-primal implementation and the one with two dual variables as double-dual implementation. In fact, we can also implement Algorithm~\ref{alg:2} by double-dual implementation and implement Algorithm~\ref{alg:3} by double-primal implementation. For Algorithm~\ref{alg:2}, in which the updates are 
 \begin{align*}
&\w_t = \min_{\w\in\Q_\w} \frac{\|\w - (\u_{t-1} - \gamma
G_\w(\balpha_{t-1}))\|_2^2}{2} + \gamma \lambda R(\w)\\
&\balpha_t = \min_{\balpha\in\Q_{\balpha}}\frac{\|\balpha
-(\balpha_{t-1} + \gamma G_{\balpha}(\w_t) ) \|_2^2}{2}\\
&\u_t = \w_t + \gamma (G_\w(\balpha_{t-1}) - G_\w(\balpha_t)),
\end{align*}
we can plug the expression of $\u_t$ into $\w_t$ and obtain
 \begin{align*}
&\w_t = \min_{\w\in\Q_\w} \frac{\|\w - (\w_{t-1}+ 2\gamma G_\w(\balpha_{t-2})- 2\gamma
G_\w(\balpha_{t-1}))\|_2^2}{2} + \gamma \lambda R(\w)\\
&\balpha_t = \min_{\balpha\in\Q_{\balpha}}\frac{\|\balpha
-(\balpha_{t-1} + \gamma G_{\balpha}(\w_t) ) \|_2^2}{2}
\end{align*}
To implement above updates, we can only maintain one primal variable and two dual variables.  Depending on the nature of implementation, one may be better than the other. For example, if the number of examples $n$ is much larger than the number of dimensions $d$, the double-primal implementation may be more efficient than the double-dual implementation, and vice versa. In subsection~\ref{sec:exp5}, we provide more examples and an experiment to demonstrate this. }

\subsection{Extensions and Discussion}\label{sec:ext}
\paragraph{Nonlinear model }For a nonlinear model, the min-max formulation becomes
\begin{align*}
\min_{g\in\mathcal H_\kappa} \max_{\balpha\in\Q_{\balpha}} L(g,\balpha; \X, \y)+ \lambda R(g),
\end{align*}
where $\mathcal H_\kappa$ is a Reproducing Kernel Hilbert Space (RKHS) endowed with a kernel function $\kappa(\cdot, \cdot)$.
Algorithm~\ref{alg:1} can be applied to obtain the nonlinear model  by changing the primal variable to $g$. For example, step 5 in Algorithm~\ref{alg:1} is modified to the following composite gradient mapping
\begin{align}\label{eqn:mklprox}
g_t= \mathop{\arg\min_{g \in \mathcal H_\kappa}}  \frac{1}{2}\left\|g-\hat g_{t-1}\right\|_{\mathcal{H}_{\kappa}}^2 + \gamma \lambda R(g),
 \end{align}
 where 
 \begin{align*}
\hat g_{t-1}=\left(g_{t-1}-\gamma\nabla_g L(g_{t-1},\balpha_t;\X, \y)\right).
 \end{align*}
Similar changes can be made to  Algorithm~\ref{alg:2} for the extension to the nonlinear model. {To end this discussion, we make several remarks.  (1) The gradient with respect to the primal variable (i.e., the kernel predictor $g\in\mathcal H_{\kappa}$) is computed on each $g(\x_i)=\langle g, \kappa(\x_i,\cdot)\rangle$ by $\kappa(\x_i,\cdot)$. (2) We can perform the computation by manipulating  on a finite number of parameters  due to the representer theorem provided that the regularizer $R(g)$ is a monotonic norm~\citep{Bach:2011:OSP:2208224}. Therefore,  we only need to maintain and update the coefficients $\zeta=(\zeta_1, \ldots, \zeta_n)$ in $g=\sum_{i=1}^n\zeta_i\kappa(\x_i,\cdot)$. (3) The primal dual prox method for optimization with nonlinear model has been adopted in our prior work~\citep{Yangicml12mkl} for multiple kernel learning where the regularizer is $R(g_1,\ldots, g_m) = (\sum_{k=1}^m\|g_k\|_{\mathcal H_k})^2$. It can also be generalized to solve MKL with more general sparsity-induced norms. (\citep{Bach:2011:OSP:2208224} considers how to compute the proximal mapping in~(\ref{eqn:mklprox}) for more general sparsity induced norms.) }

%which can be solved efficiently when $R(g) $ is  a monotonically increasing function of $\|g\|_{\mathcal H_{\kappa}}$.

{t\paragraph{Incorporating the bias term}
%Algorithms~\ref{alg:1} or~\ref{alg:2} can be modified to incorporate
It is easy to learn a bias term $w_0$ in the classifier $\w^{\top}\x + w_0$ by Pdprox without too many changes. We can use the augmented feature vector $\widehat \x_i = \left(1\atop \x_i\right)$ and the augmented weight vector $\widehat\w=\left(w_0\atop \w\right)$, and  run Algorithms~\ref{alg:1} or~\ref{alg:2} with no changes except that the regularizer $R(\widehat\w) = R(\w)$ does not involve  $w_0$ and the step size $\gamma=\sqrt{1/(2c)}$ will be a different value due to the change in the bound of the new feature vectors by $\|\widehat \x\|_2\leq \sqrt{1+R^2}$, which would yield a different  value of $c$ in Lemma~\ref{lem:1} (c.f. Appendix~\ref{app:constant}). }

%viewing $b$ as an
%additional primal variable. For example, in Algorithm~\ref{alg:1}, at each iteration, we can update $b$ by  one prox mapping
%$b_{t}= \prod_{\Q_b} [b_{t-1} + \gamma \nabla_b L(\w_{t-1}, b_{t-1}, \balpha_t; \X,\y)]$,
%where $\Q_b$ is an appropriate domain for $b$. The step size $\gamma$ is changed to $\displaystyle \gamma=\sqrt{(1/c)}/{2}$ when the bias term is considered, and the convergence rate is amplified by a constant $\sqrt{2}$.

\paragraph{Domain constraint on primal variable}
%In previous sections, we assume the domain of the primal variable $\w$ is the entire space(e.g. $\mathbb R^d$).  Actually, Algorithm~\ref{alg:1}, \ref{alg:2} can handle the domain on $\w$, i.e. $\w\in \Q_\w$. The only change is adding the domain constraint $\w\in\Q_\w$ in the composite gradient mapping, i.e. 
%\begin{align}\label{eqn:concom1}
%\w_t =  \arg\min_{\w\in\Q_\w} \frac{1}{2}\|\w - (\w_{t-1} - \gamma G_\w(\w_{t-1}, \alpha_{t}))\|_2^2 + \gamma\lambda R(\w)
%\end{align}
%for Algorithm~\ref{alg:1}, 
%and 
%\begin{align}\label{eqn:concom2}
%\w_t =  \arg\min_{\w\in\Q_\w} \frac{1}{2}\|\w - (\u_{t-1} - \gamma G_\w(\u_{t-1}, \alpha_{t-1}))\|_2^2 + \gamma\lambda R(\w)
%\end{align}
%for Algorithm~\ref{alg:2},
%where we assume the domain constrained composite gradient mapping can be solved efficiently.  
Now we discuss how to generalize the convergence analysis to the case when a convex domain $\Q_\w$ is imposed on $\w$.  %that the convergence rate in Theorem~\ref{them:main},~\ref{them:main2} still hold except that $\min_{\w} F(\w, \widehat\alpha_T)$ is replaced with $\min_{\w\in\Q_\w}F(\w,\widehat\alpha_T)$. Although, we can not have the equivalent gradient mapping on $\w_t$ as in Lemma~\ref{lem:6},~\ref{lem:11}.
%\begin{align*}
%\w_t = \u_{t-1} - \gamma G_\w(\u_{t-1}, \alpha_{t-1}) - \gamma\lambda \partial R(\w_t)
%\end{align*}
We introduce  $\widehat R(\w) = \lambda R(\w) + Q(\w)$, where $Q(\w)$ is an indicator function for $\w\in\Q_\w$, i.e. 
\begin{align*}
Q(\w) =\left\{ \begin{array}{cc}0&\w\in\Q_\w\\ +\infty&\text{ otherwise }\end{array}\right..
\end{align*}
Then we can write the domain constrained composite gradient mapping in step 5 of Algorithm~\ref{alg:1} or step 4 of Algorithm~\ref{alg:2} into a domain free composite gradient mapping as the following: 
\begin{align*}
\w_t &=  \arg\min_{\w\in\mathbb R^d} \frac{1}{2}\|\w - (\w_{t-1} - \gamma G_\w(\w_{t-1}, \balpha_{t}))\|_2^2 + \gamma\widehat R(\w),\\
\w_t &=  \arg\min_{\w\in\mathbb R^d} \frac{1}{2}\|\w - (\u_{t-1} - \gamma G_\w(\u_{t-1}, \balpha_{t-1}))\|_2^2 + \gamma\widehat R(\w).
\end{align*}
Then we have an equivalent gradient mapping, 
\begin{align*}
\w_t &= \w_{t-1} - \gamma G_\w(\w_{t-1}, \balpha_{t}) - \gamma \partial\widehat R(\w_t),\\
\w_t &= \u_{t-1} - \gamma G_\w(\u_{t-1}, \balpha_{t-1}) - \gamma \partial\widehat R(\w_t).
\end{align*}
Then Lemmas~\ref{lem:6} and~\ref{lem:7}, and Lemmas~\ref{lem:11} and~\ref{lem:12} all hold as long as we replace $\lambda\mathbf v_t$ with $\widehat{\mathbf v}_t \in \partial\widehat R(\w_t)$. Finally in proving Theorems~\ref{them:main} we can absorb $Q(\w)$ in $L(\w, \balpha) + \widehat R(\w)$ into the domain constraint.

\paragraph{Additional constraints on dual variables}One advantage of the proposed primal dual prox method is that it provides a convenient way to handle additional constraints on the dual variables $\alpha$. Several studies introduce additional constraints on the dual variables. In~\citep{DBLP:conf/nips/DekelS06}, the authors address a budget SVM problem by introducing a $1-\infty$ interpolation norm on the empirical hinge loss, leading to a sparsity constraint $\|\balpha\|_1 \leq m$ on the dual variables, where $m$ is the target number of support vectors. The corresponding optimization problem is given by
\begin{eqnarray}
\min_{\w\in\mathbb R^d} \max_{\balpha\in[0, 1]^n, \|\balpha\|_1\leq m} \frac{1}{n}\sum_{i=1}^n \alpha_i (1-y_i\w^{\top}\x_i) + \lambda R(\w). \label{eqn:constrained-obj}
\end{eqnarray}
In~\citep{robustmetric10uai}, a similar idea is applied to learn a distance metric from noisy training examples. We can directly apply Algorithms~\ref{alg:1} or~\ref{alg:2} to (\ref{eqn:constrained-obj}) with $\Q_{\balpha}$ given by $\Q_{\balpha}=\{\balpha: \balpha\in[0, 1]^n, \|\balpha\|_1\leq m\}$. The prox mapping to this domain can be efficiently computed by Lemma~\ref{lem:2}. It is straightforward to show that the convergence rate is $[D_1 + m]/[\sqrt{2n}T]$ in this case.

\section{Experiments}\label{sec:exp}
In this section we present empirical studies to verify the efficiency of the proposed algorithm.  We organize our experiments as follows.  
\begin{itemize}
\item In subsections~\ref{sec:exp1}, \ref{sec:exp2}, and \ref{sec:exp3} we compare the proposed algorithm to the state-of-the-art first order methods that directly update the primal variable at each iteration. We apply all the algorithms  to three different tasks with different non-smooth loss functions and regularizers. The baseline first order methods used in this study include  the gradient descent algorithm (\textbf{gd}), the forward and backward splitting algorithm (\textbf{fobos})~\citep{Duchi:2009:EOB:1577069.1755882}, the regularized dual averaging algorithm (\textbf{rda})~\citep{NIPS2009_1048}, the accelerated gradient descent algorithm (\textbf{agd})~\citep{10.1109/ICDM.2009.128}. %., and Nesterov's smoothing technique (\textbf{nest})~\citep{Nesterov2005}.  
Since the proposed method is a non-stochastic method, we compare it to the non-stochastic variant of \textbf{gd, fobos}, and \textbf{rda}.  Note that \textbf{gd, fobos, rda}, and  \textbf{agd} share the same convergence rate of $O(1/\sqrt{T})$ for non-smooth problems.%, while \textbf{nest} has a convergence rate of $O(1/T)$ for the non-smooth problems considered in our experiments.

\item In subsection~\ref{sec:expg}, our algorithm is  compared to the state-of-the-art primal dual gradient method~\citep{Nesterov:2005:EGT:1081200.1085585}, which employs an excessive gap technique for non-smooth optimization, updates both the primal and dual variables at each iteration,  and has a convergence rate of $O(1/T)$.

\item In subsection~\ref{sec:exp4}, we test the proposed algorithm for optimizing problem in~(\ref{eqn:constrained-obj}) with a sparsity constraint on the dual variables. 
\item In subsection~\ref{sec:exp5}, we compare the two variants of the proposed method on a data set when $n\gg d$, and compare Pdprox to the Pegasos algorithm.
\end{itemize}

%To apply Nesterov's smoothing technique and excessive gap technique, we write both the loss function and  the regularizer into an explicit max form.  
All the algorithms are implemented in Matlab (except otherwise mentioned) and run on a 2.4GHZ machine. Since the performance of the baseline algorithms \textbf{gd}, \textbf{fobos} and \textbf{rda}  depends heavily on the initial value of the stepsize, we generate $21$ values for the initial stepsize by scaling their theoretically optimal values with factors $2^{[-10:1:10]}$, and report the best convergence among the $21$ possible values. The stepsize of \textbf{agd} is changed adaptively in the optimization process, and we just give it an appropriate initial step size. 
%A similar approach is applied to generate 21 values for the smoothing parameter of \textbf{nest}, and only the best performance is reported. 
Since in the first four subsections we focus on comparison with baselines, we use the Pdprox-dual algorithm (Algorithm~\ref{alg:1})  of the proposed Pdprox method.  {We also use the tuning technique to select the best scale-up factor for the step size $\gamma$ of Pdprox.}  Finally, all algorithms are initialized with a solution of all zeros.

\subsection{Group lasso regularizer for Grouped Feature Selection}\label{sec:exp1}
\begin{figure*}[t]
\centering
%\subfigure[$\lambda=10^{-3}$]{\includegraphics[scale=0.25]{gl_hinge_lamda_1e-3_T_1000_nonest.eps}}\hspace*{0.2in}
%\subfigure[$\lambda=10^{-5}$]{\includegraphics[scale=0.25]{gl_hinge_lamda_1e-5_T_1000_nonest.eps}}\\
%\subfigure[$\lambda=10^{-3}$]{\includegraphics[scale=0.25]{gl_abs_lamda_1e-3_T_1000_nonest.eps}}\hspace*{0.2in}
%\subfigure[$\lambda=10^{-5}$]{\includegraphics[scale=0.25]{gl_abs_lamda_1e-5_T_1000_nonest.eps}}
\subfigure[$\lambda=10^{-3}$]{\includegraphics[scale=0.25]{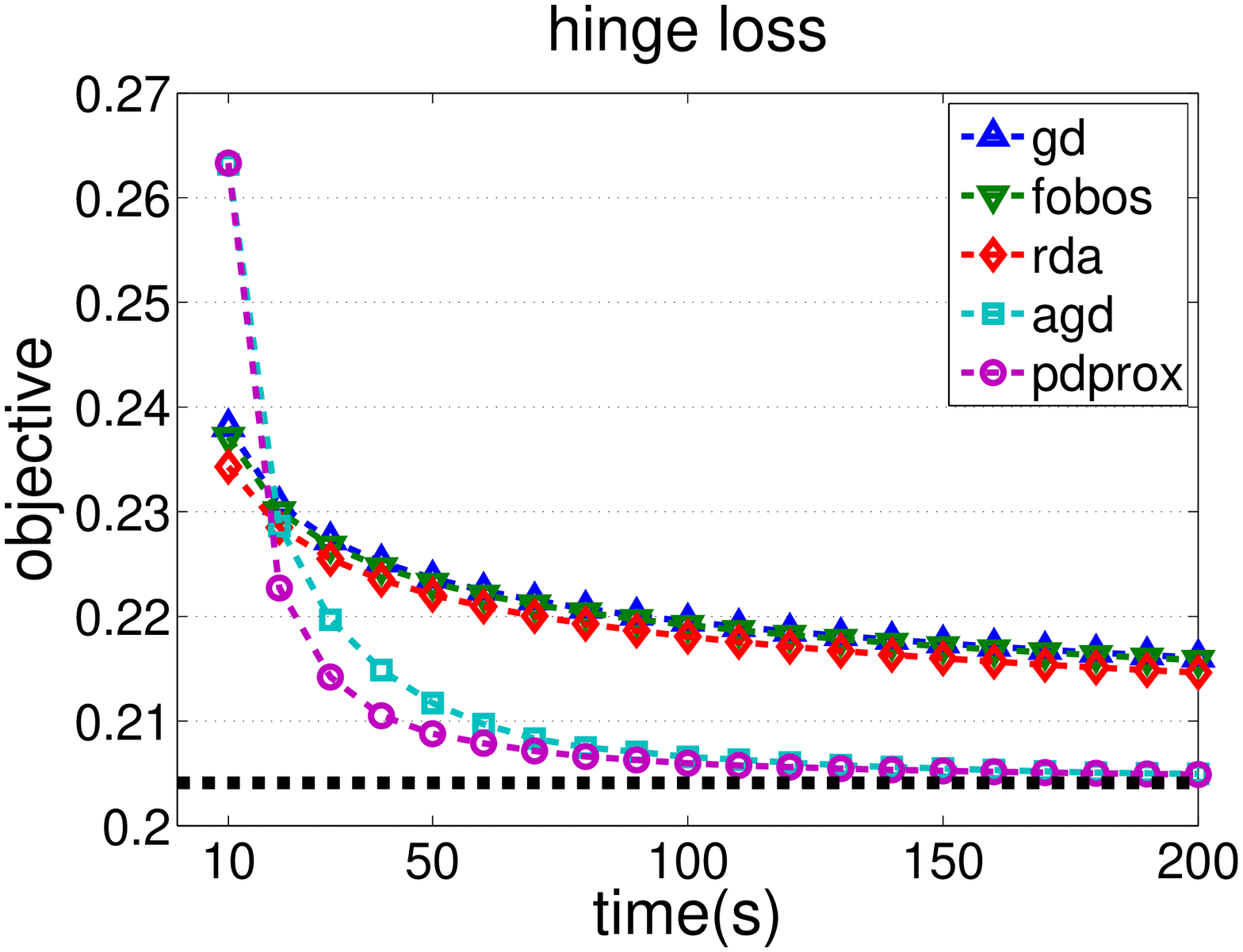}}\hspace*{0.2in}
\subfigure[$\lambda=10^{-5}$]{\includegraphics[scale=0.25]{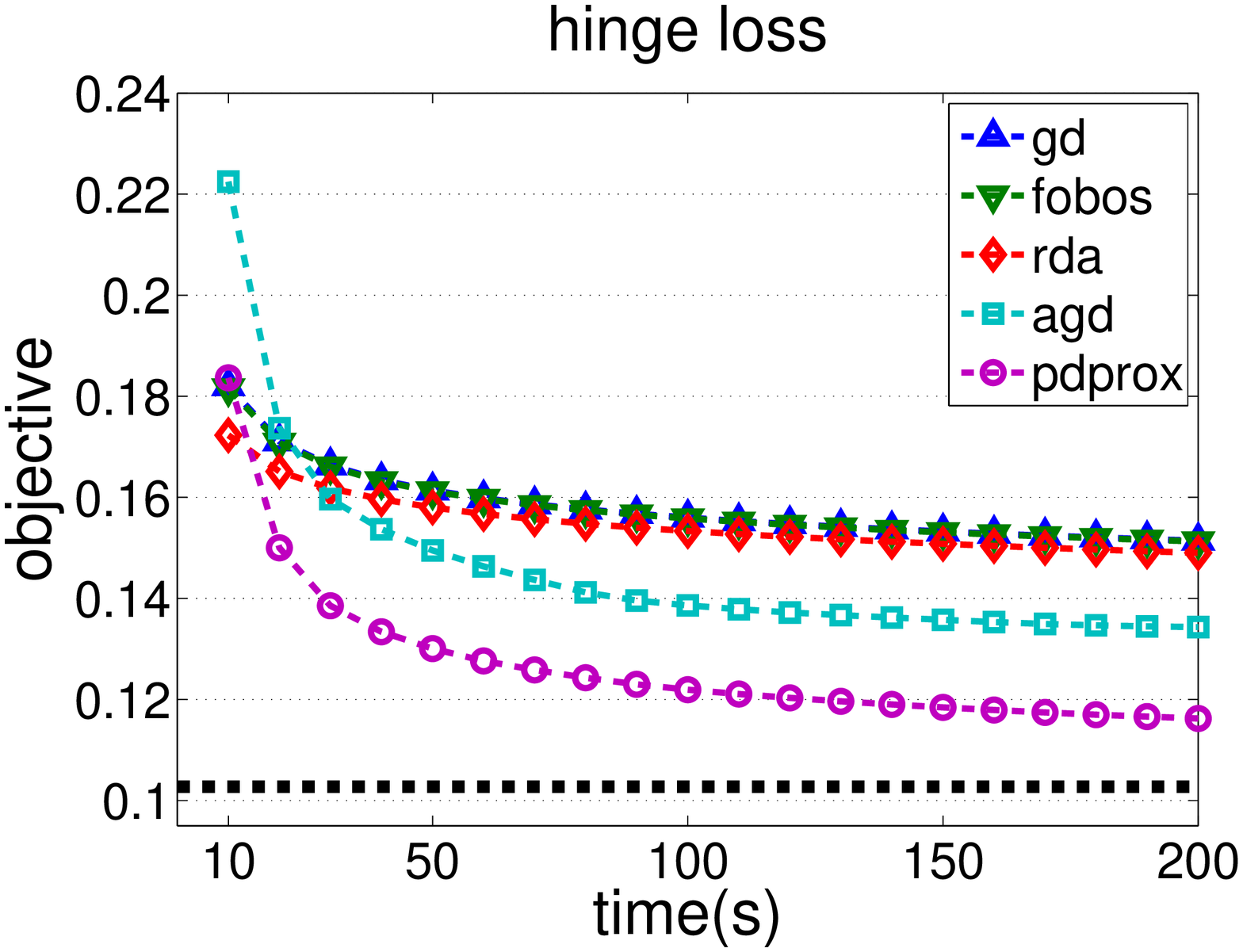}}\\
\subfigure[$\lambda=10^{-3}$]{\includegraphics[scale=0.25]{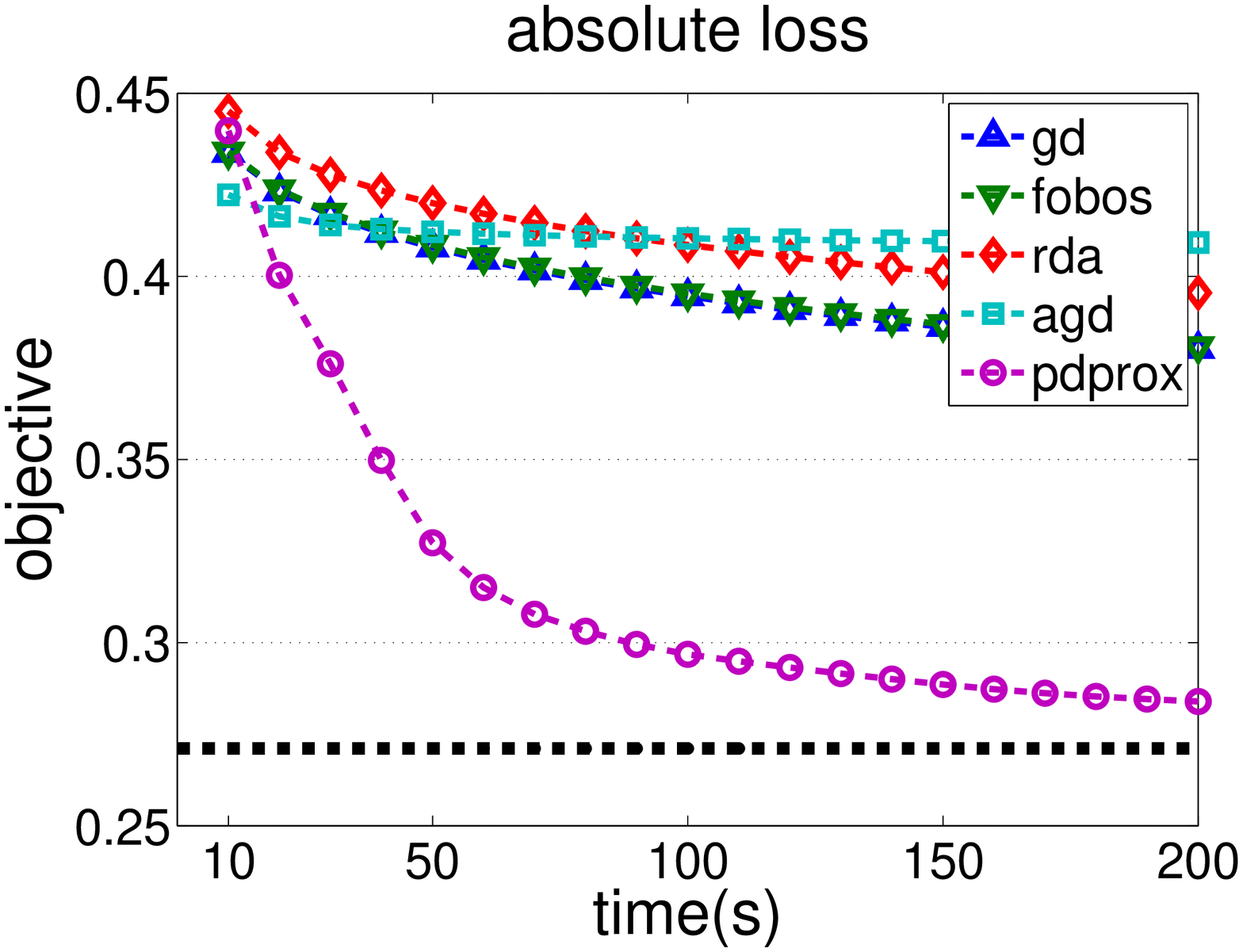}}\hspace*{0.2in}
\subfigure[$\lambda=10^{-5}$]{\includegraphics[scale=0.25]{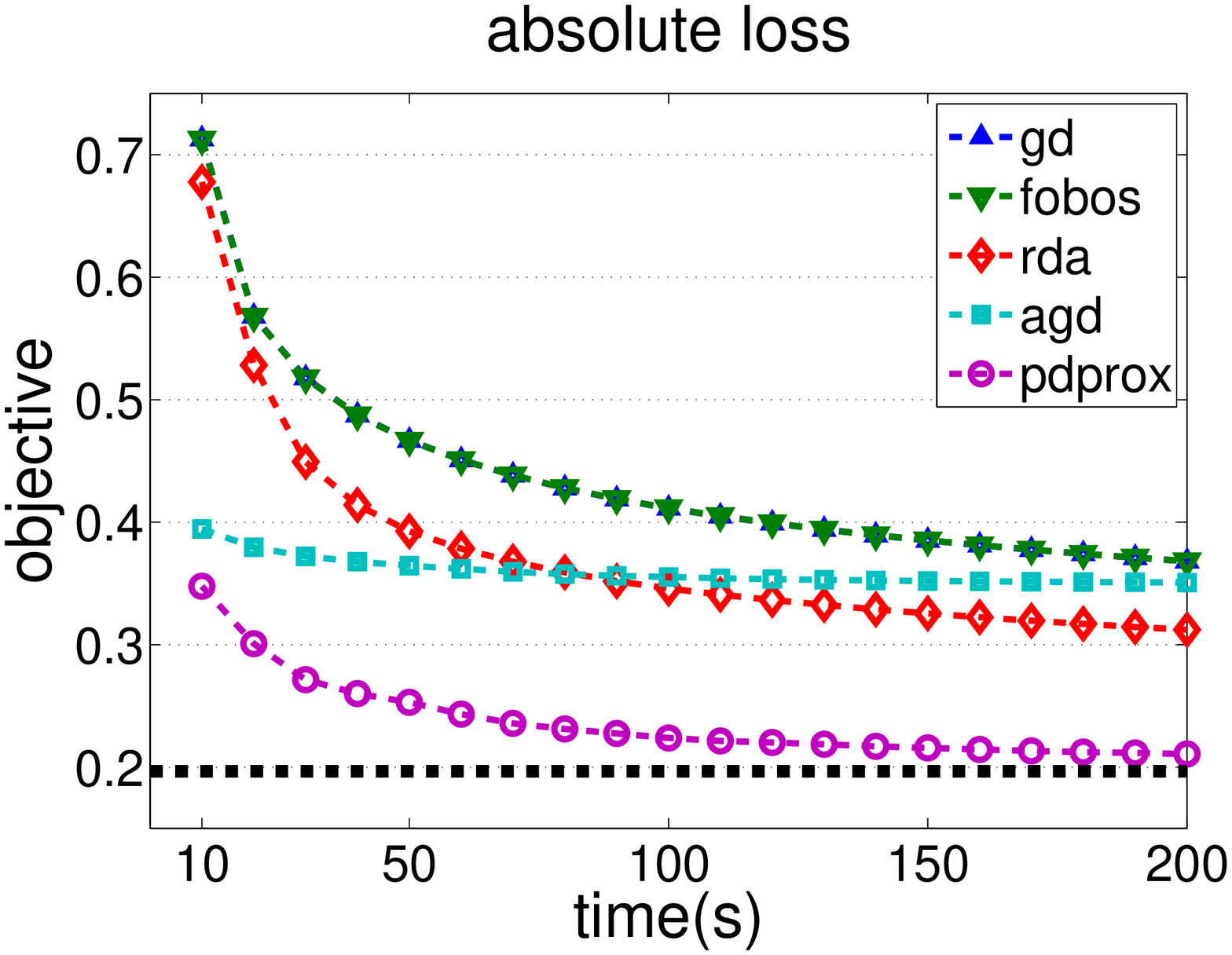}}

\caption{Comparison of convergence speed for hinge loss ((a),(b)) and absolute loss ((c),(d)) with group lasso regularizer. Note that for better visualization we plot the objective starting from 10 seconds  in all figures. The objective of all algorithms at 0 second is 1. {The black bold dashed lines in all Figures show the optimal objective value by running Pdprox with a large number of iterations so that the difference between the last two objective values is less than $10^{-4}$. } }\label{fig:1}
\end{figure*}
%We focus on the non-smooth optimization problem that uses 
In this experiment  we use the group lasso for regularization, i.e., $R(\w)=\sum_g\sqrt{d_g} \|\w_g\|_2$, where $\w_g$ corresponds to the $g$th group variables and $d_g$ is the number of variables in group $g$. To apply Nesterov's method, we can write  $R(\w)=\max_{\|\u_g\|_2\leq 1} \sum_g\sqrt{d_g}\w_g^{\top}\u_g$. We use the MEMset Donar dataset~\citep{Yeo:2003:MEM:640075.640118} as the testbed. This dataset was originally used for splice site detection. It is divided into a training set and a test set: the training set consists  of $8,415$ true  and $179,438$ false donor sites, and the testing set has $4,208$ true  and $89,717$ false donor sites. Each example in this dataset was originally described by a sequence of \{A, C, G, T\} of length $7$. We follow~\citep{DBLP:conf/icml/YangXKL10} and generate group features with up to three-way interactions between the $7$ positions, leading to $2,604$ attributes in $63$ groups. We normalize the length of each example to $1$. Following the experimental setup in~\citep{DBLP:conf/icml/YangXKL10}, we construct a balanced training dataset consisting of all $8,415$ true and $8,415$ false donor sites that are randomly sampled from all $179,438$ false sites.

Two non-smooth loss functions are examined in this experiment: hinge loss and absolute loss.   Figure~\ref{fig:1} plots the values of the objective function vs. running time (second), using two different values of regularization parameter, i.e., $\lambda=10^{-3}, 10^{-5}$  to produce different levels of sparsity. We observe that { (i) the proposed algorithm \textbf{Pdprox} clearly outperforms all the baseline algorithms in all the cases; (ii) for the absolute loss, which has a sharp curvature change at zero compared to hinge loss, the baseline algorithms of \textbf{gd, fobos, rda, agd}, especially of \textbf{agd} that is originally designed for smooth loss functions,  deteriorate significantly compared to the proposed algorithm \textbf{Pdprox}.%; (iii) the convergent performance of \textbf{nest} depends heavily on the value of $\lambda$. The reason is that \textbf{nest} handles the regularizer differently from other algorithms and the value of $\lambda$ affects the smoothness of the smoothed objective and therefore affects the convergent performance of \textbf{nest}. 
}   Finally, we observe that for the hinge loss and $\lambda = 10^{-3}$, the classification performance of the proposed algorithm on the testing dataset is $0.6565$, measured by maximum correlation coefficient~\citep{Yeo:2003:MEM:640075.640118}. This is almost identical to the best result reported in~\citep{DBLP:conf/icml/YangXKL10} (i.e., $0.6520$).

\begin{figure}[t]
\centering
%\subfigure[$\lambda=10^{-3}$]{\includegraphics[scale=0.25]{l1inf_abs_lamda_1e-3_school_nonest.eps}}\hspace*{0.2in}
%\subfigure[$\lambda=10^{-5}$]{\includegraphics[scale=0.25]{l1inf_abs_lamda_1e-5_school_nonest.eps}}\\
%\subfigure[$\lambda=10^{-3}$]{\includegraphics[scale=0.25]{l1inf_sens_1e_2_lamda_1e-3_school_nonest.eps}}%\hspace*{0.2in}
%\subfigure[$\lambda=10^{-5}$]{\includegraphics[scale=0.25]{l1inf_sens_1e_2_lamda_1e-5_school_nonest.eps}}
\subfigure[$\lambda=10^{-3}$]{\includegraphics[scale=0.25]{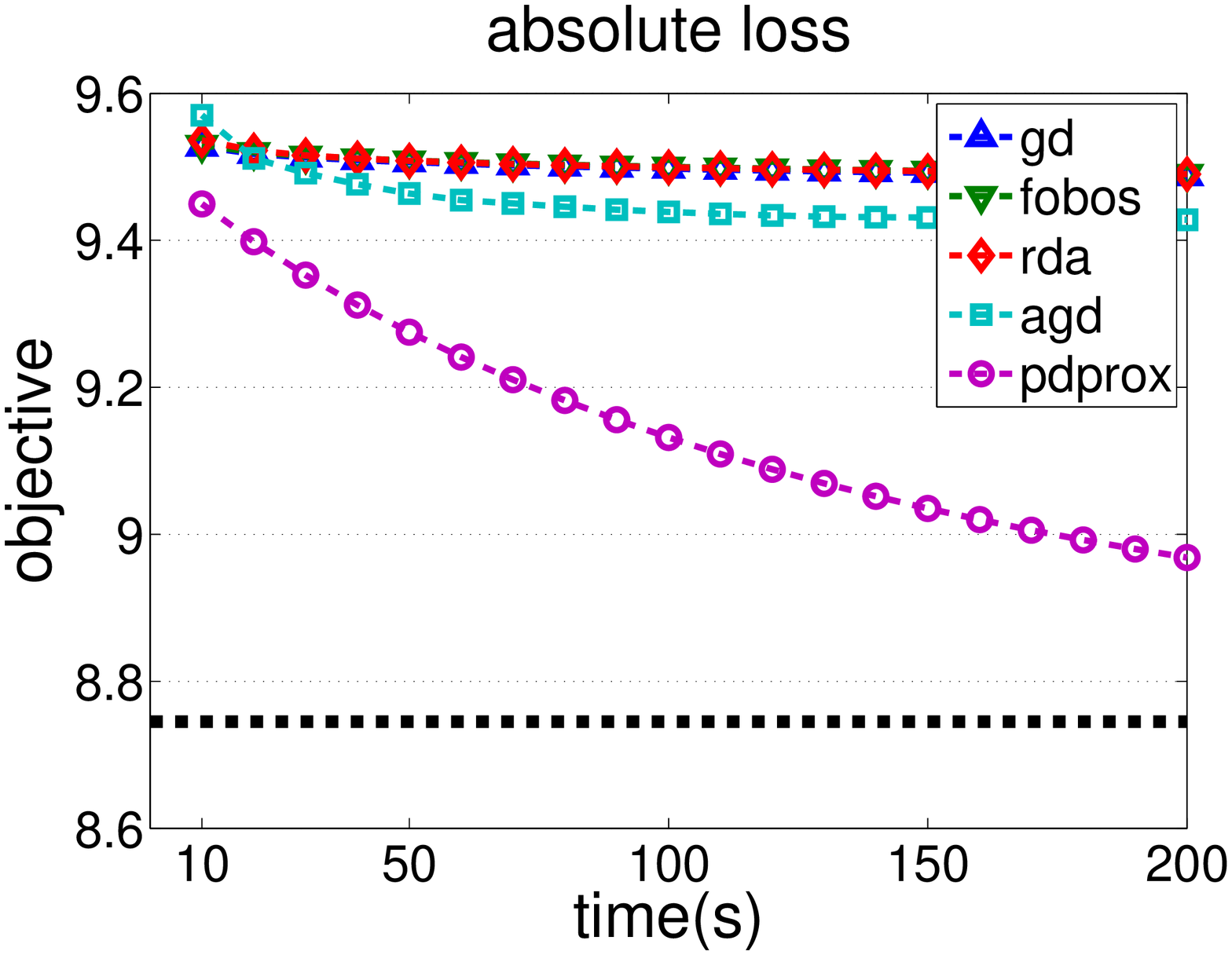}}\hspace*{0.2in}
\subfigure[$\lambda=10^{-5}$]{\includegraphics[scale=0.25]{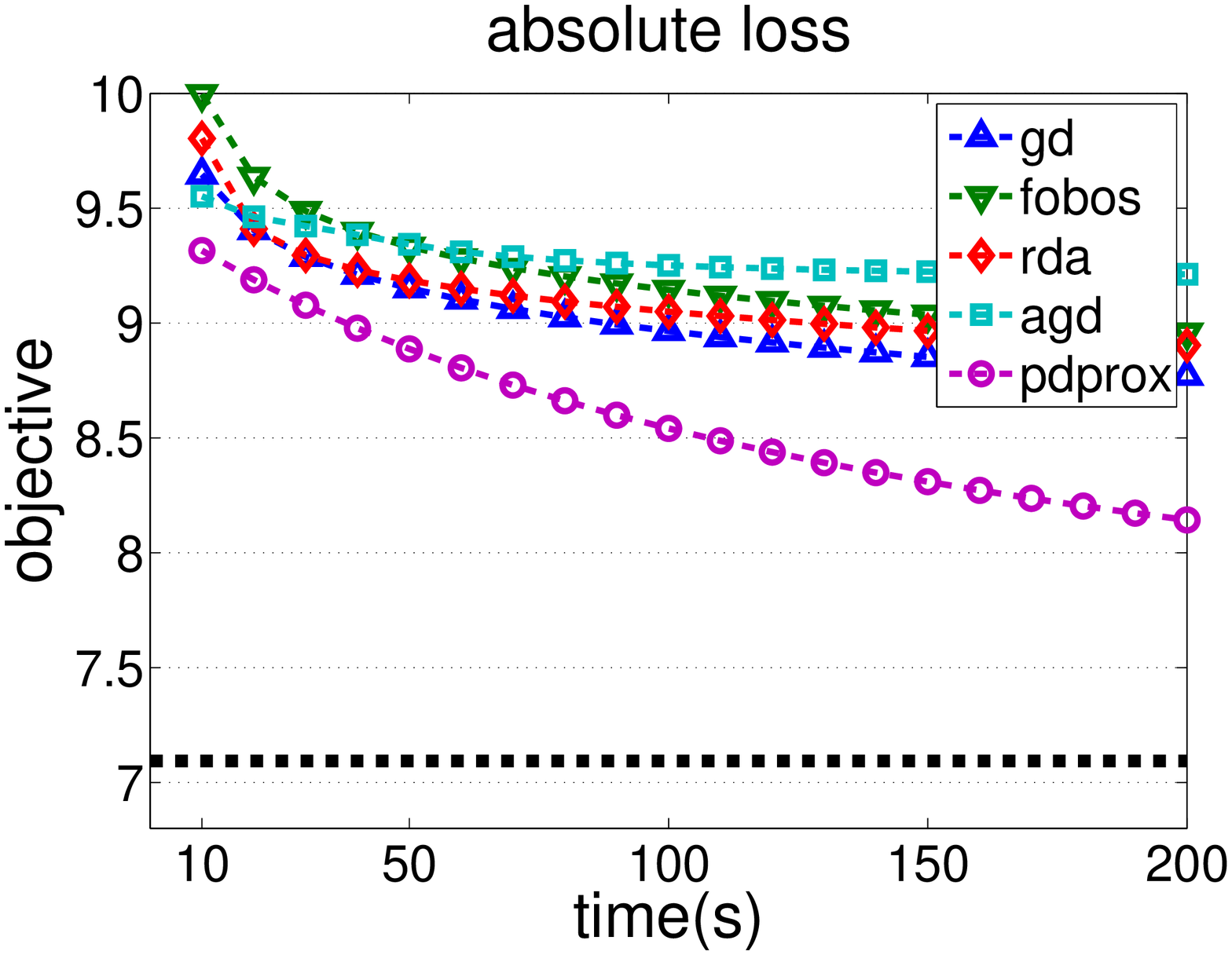}}\\
\subfigure[$\lambda=10^{-3}$]{\includegraphics[scale=0.25]{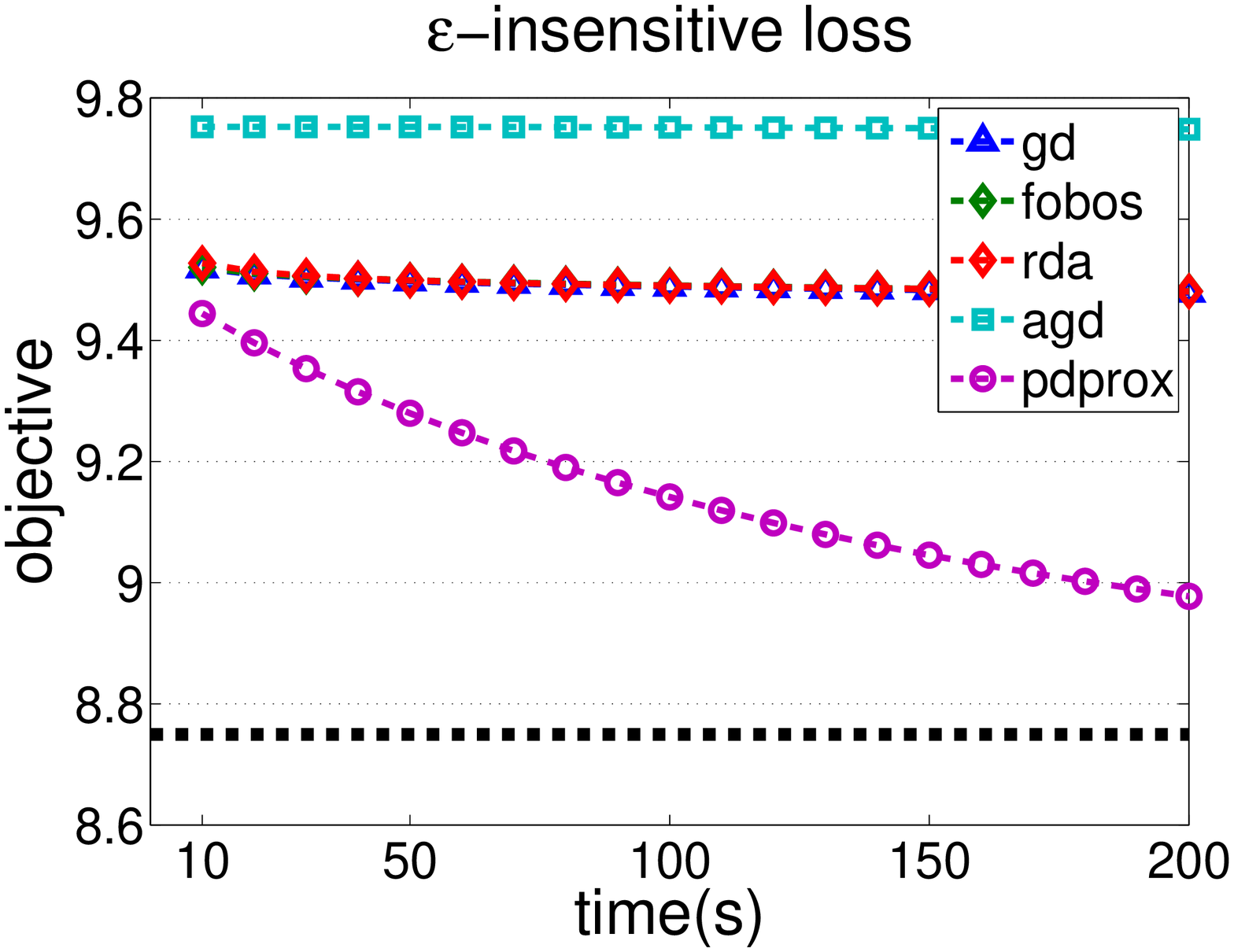}}\hspace*{0.2in}
\subfigure[$\lambda=10^{-5}$]{\includegraphics[scale=0.25]{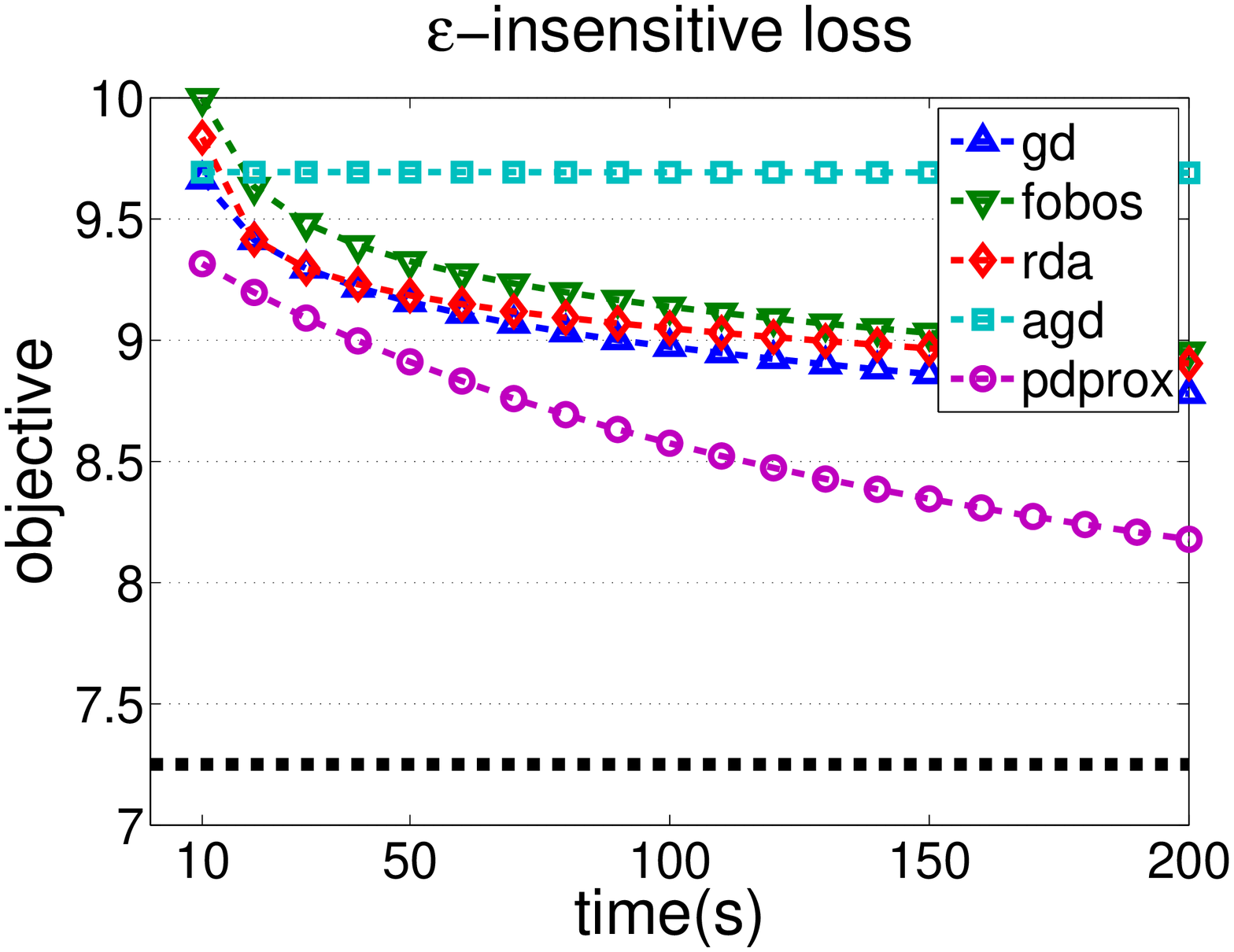}}

\caption{Comparison of convergence speed for absolute loss ((a),(b)) and $\epsilon$-insensitive loss ((c),(d)) with $\ell_{1,\infty}$ regularizer. Note that for better visualization we plot the objective starting from 10 seconds  in all figures. The objective  of all algorithms at 0 second is 20.52. {The black bold dashed lines in all Figures show the optimal objective value by running Pdprox with a large number of iterations so that the difference between the last two objective values is less than $10^{-4}$.}}\label{fig:2}
\end{figure}

\subsection{$\ell_{1,\infty}$ regularization for Multi-Task Learning}\label{sec:exp2}
In this experiment, we perform  multi-task regression with $\ell_{1,\infty}$ regularizer~\citep{10.1109/ICDM.2009.128}. Let $\mathbf W=(\w_1,\cdots, \w_k)\in\mathbb R^{d\times k}$ denote the $k$ linear hypotheses for regression.  The $\ell_{1,\infty}$ regularizer is given by $R(\mathbf W)=\sum_{j=1}^d \|\w^{j}\|_\infty$, where $\w^j$ is the $j$th row of $\mathbf W$. To apply Nesterov's method, we rewrite the $\ell_{1,\infty}$ regularizer as  $R(\mathbf W)=\max_{\|\u_j\|_1\leq 1}\sum_{j=1}^d{\u_j}^{\top}\w^{j}$. We use  the School data set~\citep{journals/ml/ArgyriouEP08}, a common dataset for multi-task learning. This data set contains the examination scores of $15,362$ students from $139$ secondary schools corresponding to $139$ tasks, one for each school. Each student in this dataset is described by 27 attributes.  We follow the setup in~\citep{journals/ml/ArgyriouEP08}, and generate a training data set with $75\%$ of the examples from each school and a testing data set with the remaining examples. We test the algorithms using both the absolute loss and the $\epsilon$-insensitive loss with $\epsilon=0.01$.  The initial stepsize for \textbf{gd}, \textbf{fobos}, and \textbf{rda} are tuned similarly as that for the experiment of group lasso. We plot the objective versus the running time  in Figure~\ref{fig:2}, from which we observe the similar results in the group feature selection task, i.e. (i) the proposed \textbf{Pdprox} algorithm outperforms the baseline algorithms, (ii) the baseline algorithm of \textbf{agd} becomes even worse for $\epsilon$-insensitive loss than for absolute loss.  Finally, we observe that the regression performance measured by root mean square error (RMSE) on the testing data set  for absolute loss and $\epsilon$-insensitive loss is $10.34$ (optimized by \textbf{Pdprox}), comparable to the performance reported in~\citep{10.1109/ICDM.2009.128}.

\begin{figure}[t]
\centering
%\subfigure[$\lambda=10^{-3}$]{\includegraphics[scale=0.25]{trace_hinge_lamda_1e-3_alldata_nonest.eps}}\hspace*{0.2in}
%\subfigure[$\lambda=10^{-5}$]{\includegraphics[scale=0.25]{trace_hinge_lamda_1e-5_alldata_nonest.eps}}\\
%\subfigure[$\lambda=10^{-3}$]{\includegraphics[scale=0.25]{trace_abs_lamda_1e-3_alldata_nonest.eps}}\hspace*{0.2in}
%\subfigure[$\lambda=10^{-5}$]{\includegraphics[scale=0.25]{trace_abs_lamda_1e-5_alldata_nonest.eps}}\\
\subfigure[$\lambda=10^{-3}$]{\includegraphics[scale=0.25]{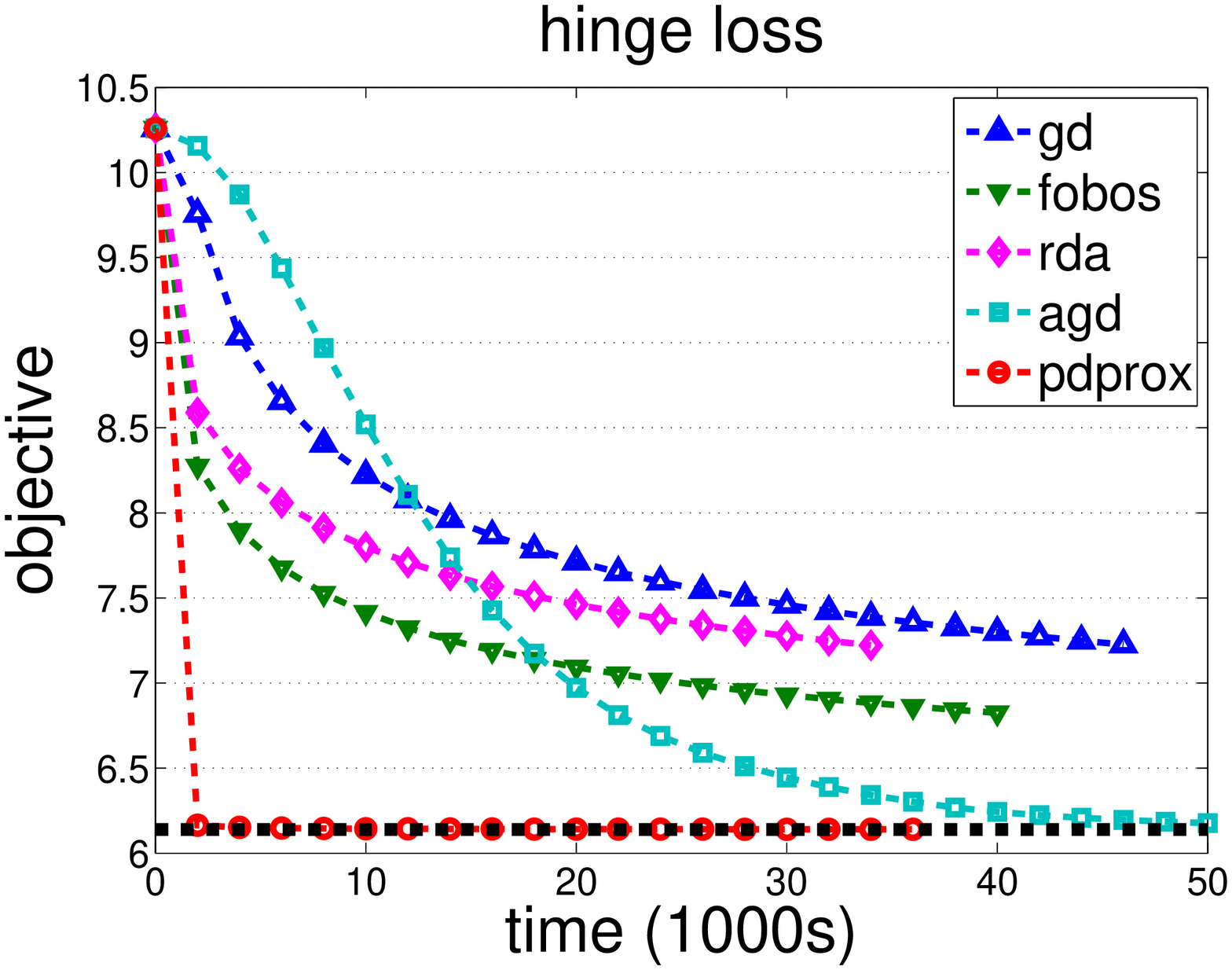}}\hspace*{0.2in}
\subfigure[$\lambda=10^{-5}$]{\includegraphics[scale=0.25]{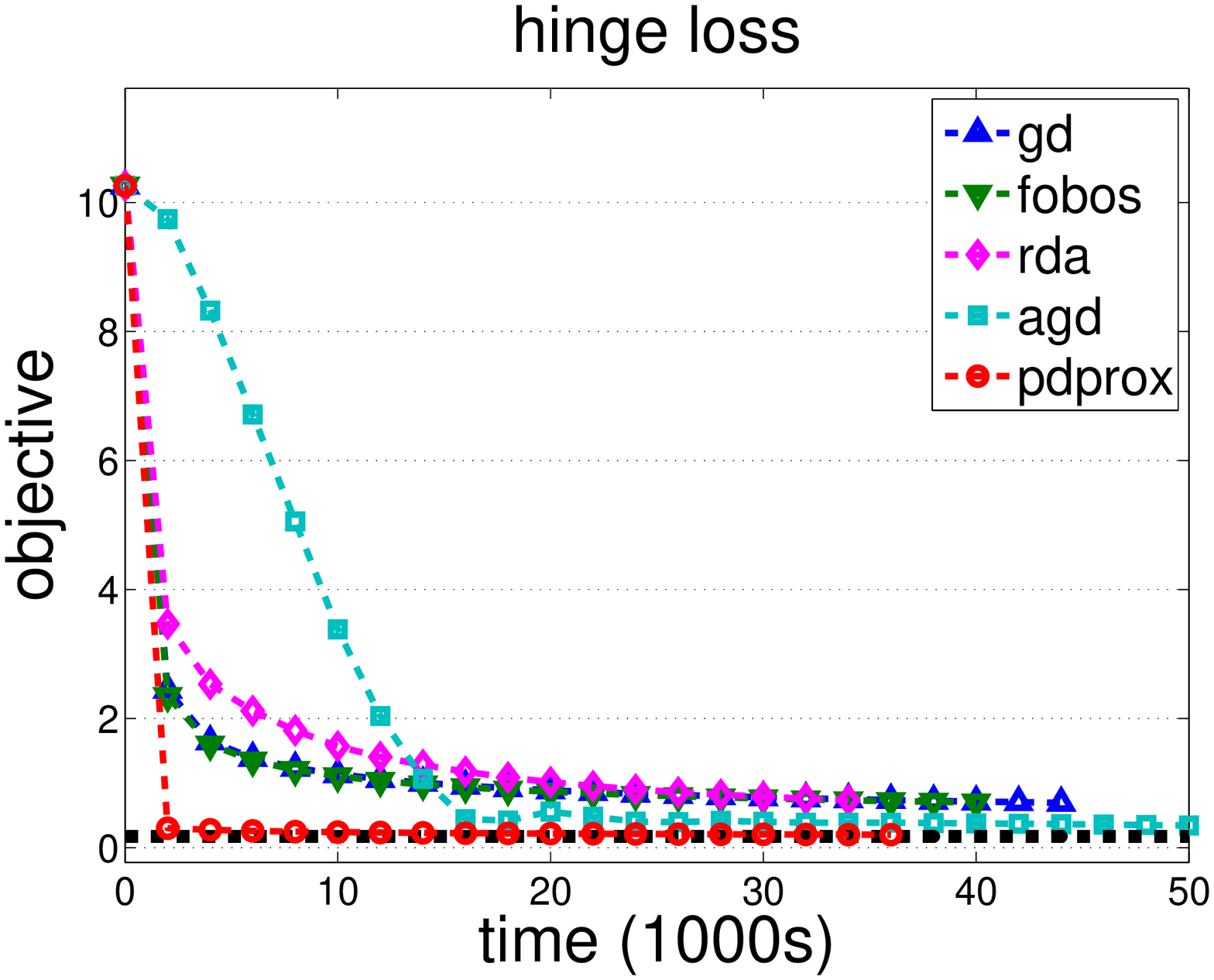}}\\
\subfigure[$\lambda=10^{-3}$]{\includegraphics[scale=0.25]{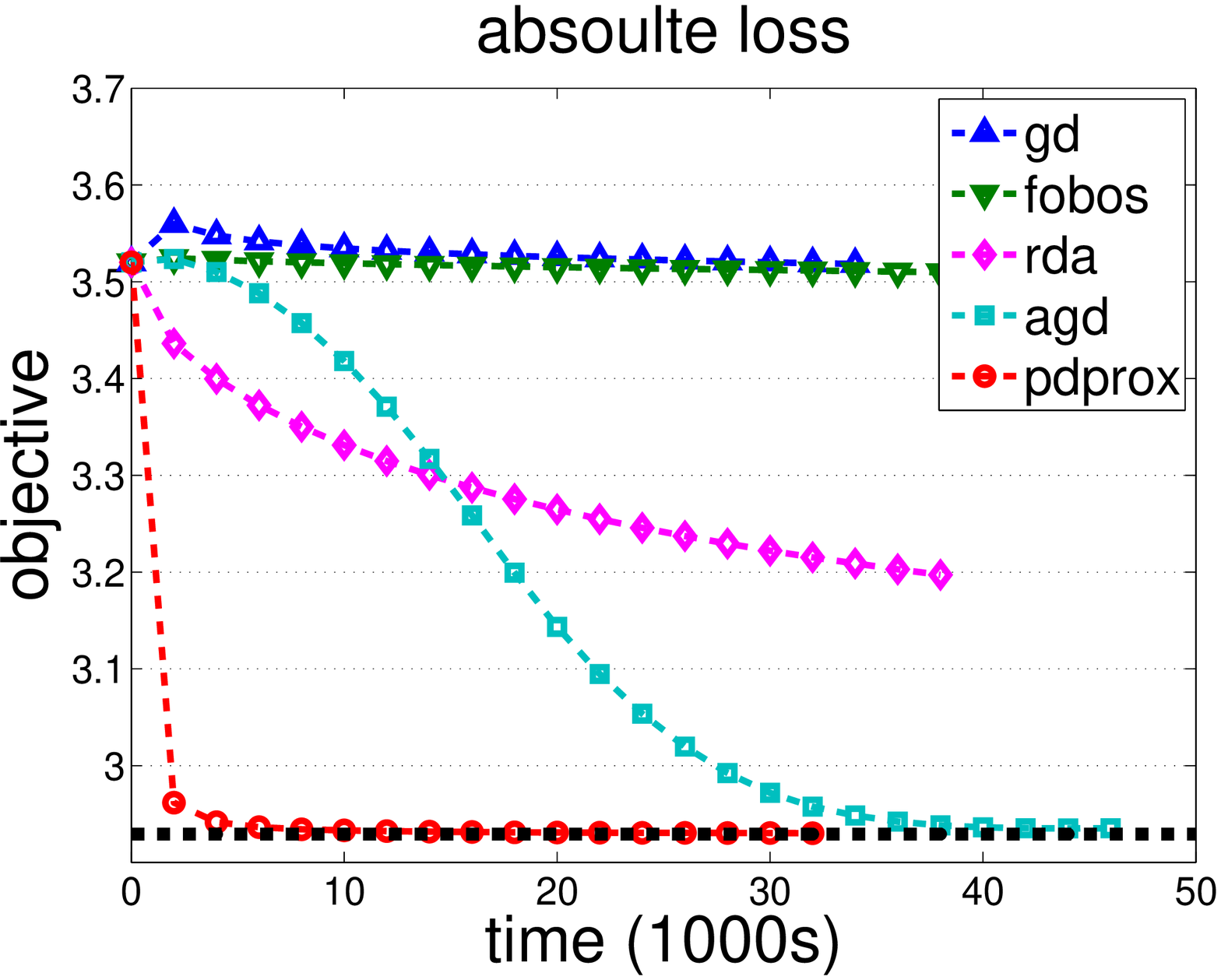}}\hspace*{0.2in}
\subfigure[$\lambda=10^{-5}$]{\includegraphics[scale=0.25]{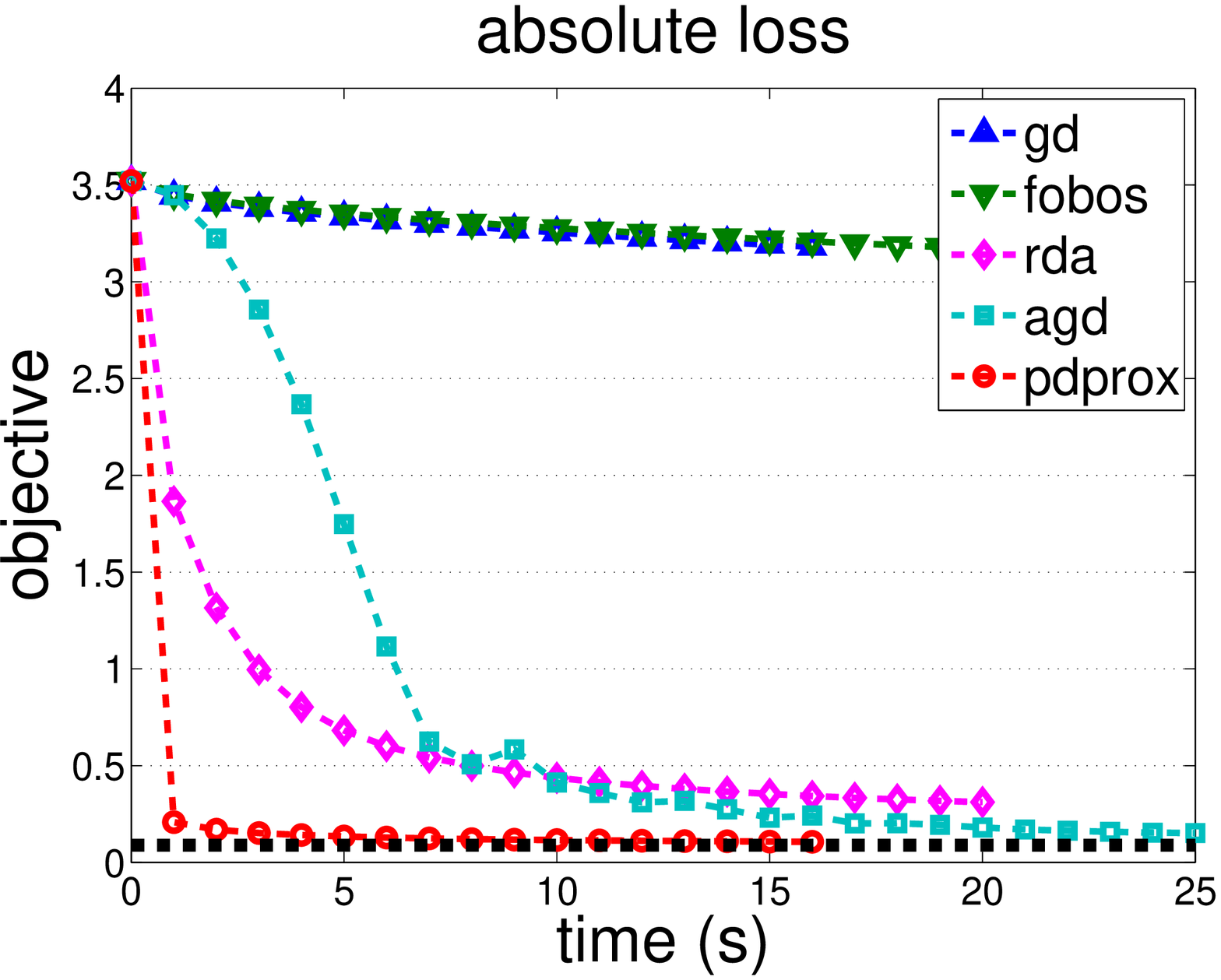}}\\
\caption{Comparison of convergence speed for (a,b): max-margin matrix factorization with hinge loss and trace norm regularizer,  and (c,d): matrix completion with absolute loss and trace norm regularizer. {The black bold dashed lines in all Figures show the optimal objective value by running Pdprox with a large number of iterations so that the difference between the last two objective values is less than $10^{-4}$.}}\label{fig:2-1}
\end{figure}

\subsection{Trace norm regularization for Max-Margin Matrix Factorization/ Matrix Completion}\label{sec:exp3}
In this experiment, we evaluate the proposed method using trace norm regularization, a regularizer often  used in max-margin matrix factorization and matrix completion, where the goal is to recover a full matrix $\mathbf X$ from partially observed matrix $\mathbf Y$.  The objective  is composed of a loss function measuring the difference between $\mathbf X$ and $\mathbf Y$ on the observed entries and a trace norm regularizer on $\mathbf X$, assuming that $\mathbf X$ is low rank.  Hinge loss function is used in max-margin matrix factorization~\citep{Rennie:2005:FMM:1102351.1102441,citeulike:3224462}, and absolute loss is used instead of square loss in matrix completion. We test on 100K MovieLens data set~\footnote{ \url{http://www.cs.umn.edu/Research/GroupLens/}} that contains 1 million ratings  from $943$ users on 1682 movies. Since there are five distinct ratings that can be assigned to each movie,  we follow~\citep{Rennie:2005:FMM:1102351.1102441,citeulike:3224462} by introducing  four thresholds $\theta_{1,2,3,4}$ to measure the hinge loss between the predicted value $X_{ij}$ and the ground truth $Y_{ij}$. Because our goal is to demonstrate the efficiency of the proposed algorithm for non-smooth optimization, therefore we simply set $\theta_{1,2,3,4}=(0, 3, 6, 9)$.  Note that we did not compare to the optimization algorithm in~\citep{Rennie:2005:FMM:1102351.1102441} since it cast the problem into a non-convex problem by using explicit factorization of $\mathbf X=\mathbf U\mathbf V^{\top}$, which suffers a local minimum, and the optimization algorithm in~\citep{citeulike:3224462} since it formulated the problem into a SDP problem, which suffers from a high computational cost. To apply Nesterov's method, we write $\|\mathbf X\|_1=\max_{\|\mathbf A\|\leq 1} tr(\mathbf A^{\top}\mathbf X)$, and at each iteration we need to solve a maximization problem $\max_{\|\mathbf A\|\leq 1}\lambda tr(\mathbf A^{\top}\mathbf X ) - \frac{\mu}{2}\|\mathbf A\|_F^2$, where $\|\mathbf A\|$ is the spectral norm on $\mathbf A$. The solution of this optimization is obtained  by performing SVD decomposition of $\mathbf X$ and thresholding  the singular values appropriately. Since MovieLens data set is much larger than the data sets used in last two subsections, in this experiment, we (i) run all the algorithms for 1000 iterations and plot the objective versus time; (ii) enlarge the range of tuning parameters to $2^{[-15:1:15]}$. The results are shown in Figure~\ref{fig:2-1}, from which we observe that (i) Pdprox can quickly reduce the objective in a small amount of time, e.g., for absolute loss when setting $\lambda=10^{-3}$  in order to obtain a solution with an accuracy of $10^{-3}$, \textbf{Pdprox} needs $10^{3}$ second, while \textbf{agd} needs $3.2\times 10^{4}$ seconds;  (ii) for absolute loss no matter how we tune the stepsizes for each baseline algorithm, Pdprox performs the best; and (iii) for hinge loss when $\lambda=10^{-5}$, by tuning the stepsizes of baseline algorithms,  \textbf{gd, fobos}, and  \textbf{rda} can achieve comparable performance to Pdprox. We note that although  \textbf{agd} can achieve smaller objective value than Pdprox at the end of $1000$ iterations, however, the objective value is reduced slowly. 
%$\mathbf A= \mathbf X(\mu \mathbf I + \mathbf Z)^{-1}$, $\mathbf X= \mathbf U\Sigma\mathbf V^{\top}$, $\mathbf Z= \mathbf V\Sigma_z\mathbf V^{\top}, \sigma(\mathbf A)_i= \sigma(\mathbf X)_i/\mu$, if $\sigma(\mathbf X)_i\leq \mu$, otherwise equal 1. 

\subsection{Comparison: Pdprox vs Primal-Dual method with excessive gap technique}\label{sec:expg}
In this section, we compare the proposed primal dual prox method to Nesterov's primal dual method~\citep{Nesterov:2005:EGT:1081200.1085585}, which is an improvement of his algorithm in~\citep{Nesterov2005}. The algorithm in~\citep{Nesterov2005} for non-smooth optimization suffers a problem of setting the value of smoothing parameter that requires the number of iterations to be fixed in advance.  \citep{Nesterov:2005:EGT:1081200.1085585} addresses the problem by exploring an excessive gap technique and updating both the primal and dual variables, which is similar to the proposed Pdprox method.  We refer to this baseline as \textbf{Pdexg}.  We run both algorithms on the three tasks as in subsections~\ref{sec:exp1}, \ref{sec:exp2}, and \ref{sec:exp3}, i.e.,  group feature selection with  hinge loss and group lasso regularizer on MEMset Donar data set,  multi-task learning with $\epsilon$-insensitive loss and $\ell_{1,\infty}$ regularizer on School data set,  and matrix completion with absolute loss  and trace norm regularizer on 100K MovieLens data set.   To implement the primal dual method with excessive gap technique, we need to intentionally add a domain on the optimal primal variable, which can be derived from the formulation.  For example,  in group feature selection problem whose objective is $1/n\sum_{i=1}^n\ell(\w^{\top}\x_i, y_i) + \lambda \sum_g \sqrt{d_g}\|\w_g\|_2$, we can derive that the optimal primal variable $\w^*$ lies in $\|\w\|_2\leq \sum_{g}\|\w_g\|_2\leq \frac{1}{\lambda \sqrt{d_{\min}}}$, where $d_{\min}=\min_g d_g$. Similar techniques are applied to multi-task learning and matrix completion. 

\begin{figure}[t]
\centering
%\subfigure[Group feature selection: hinge loss and group lasso regularizer with $\lambda=10^{-3}$.]{\includegraphics[scale=0.25]{pd_prox_exg_hinge_gl_time_iterations_lamda_1e-3.eps}}\hspace*{0.2in}
%\subfigure[Group feature selection: hinge loss and group lasso regularizer with $\lambda=10^{-5}$. ]{\includegraphics[scale=0.25]{pd_prox_exg_hinge_gl_time_iterations_lamda_1e-5.eps}}\hspace*{0.2in}

\subfigure[Group feature selection: hinge loss and group lasso regularizer with $\lambda=10^{-3}$.]{\includegraphics[scale=0.25]{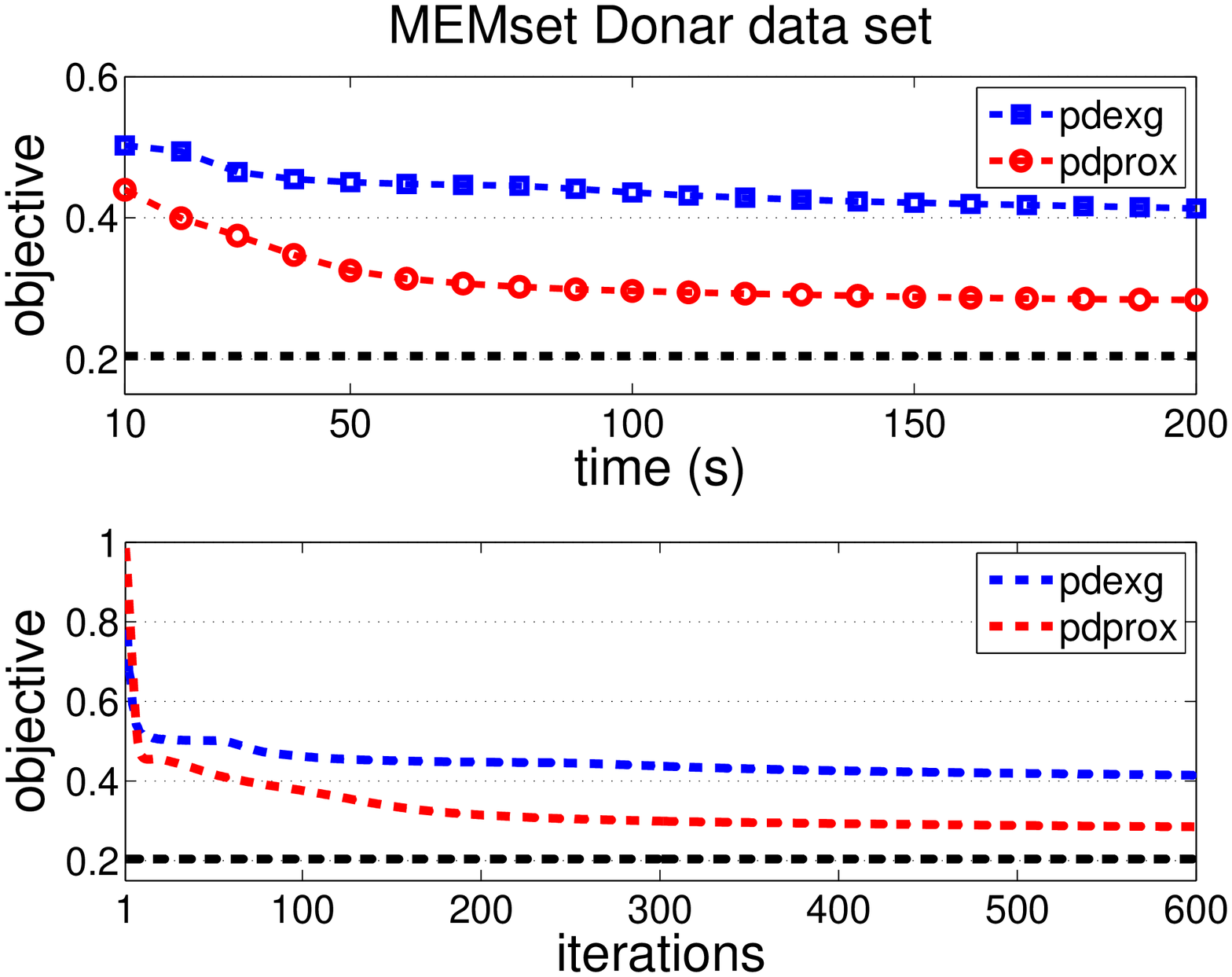}}\hspace*{0.2in}
\subfigure[Group feature selection: hinge loss and group lasso regularizer with $\lambda=10^{-5}$. ]{\includegraphics[scale=0.25]{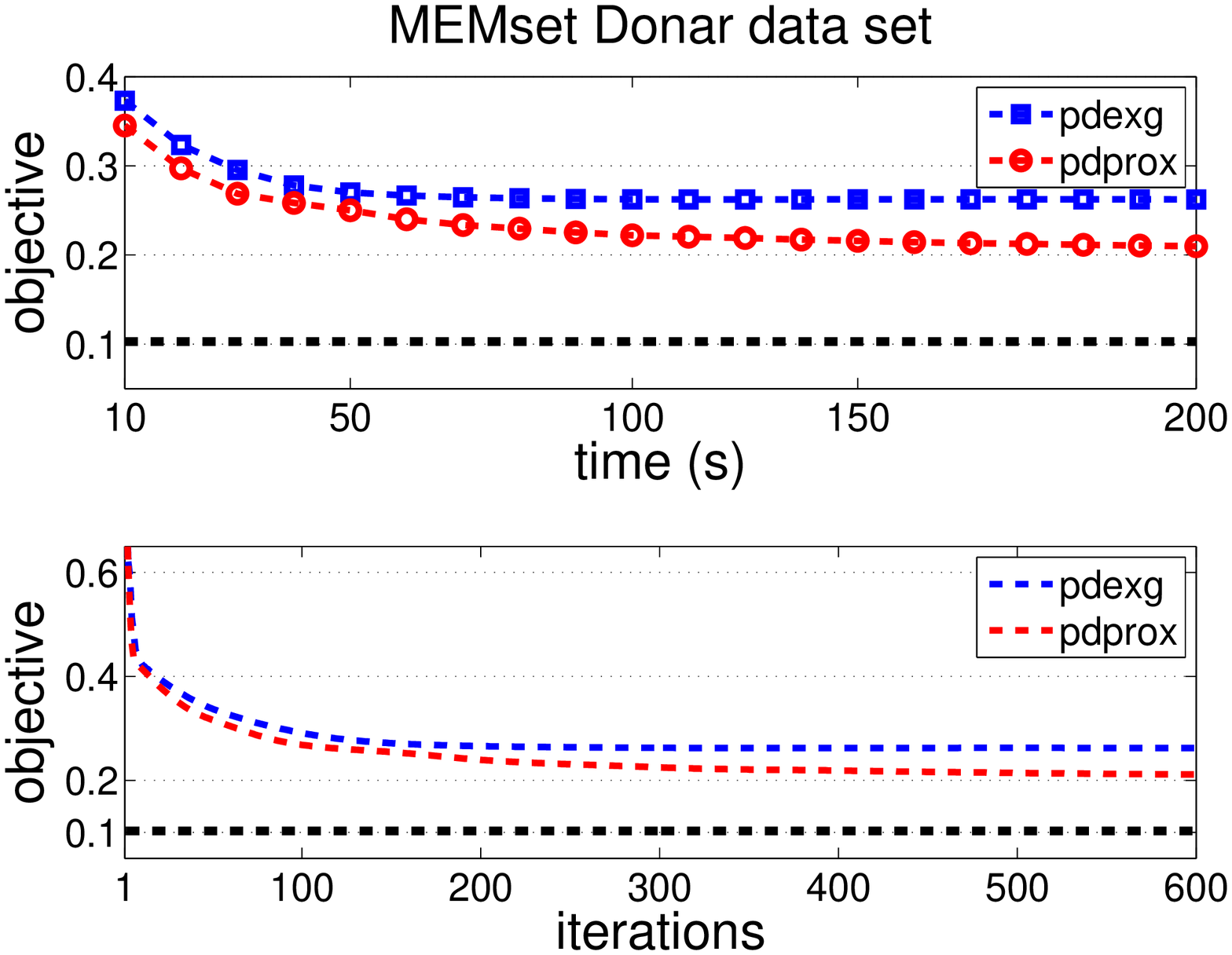}}\hspace*{0.2in}

%\subfigure[Multi-task learning: $\epsilon$-insensitive loss and $\ell_{1,\infty}$ regularizer with $\lambda=10^{-3}$. ]{\includegraphics[scale=0.25]{pd_prox_exg_epsilon_l1inf_time_iterations_lamda_1e-3.eps}}\hspace*{0.2in}
%\subfigure[Multi-task learning: $\epsilon$-insensitive loss and $\ell_{1,\infty}$ regularizer with $\lambda=10^{-5}$. ]{\includegraphics[scale=0.25]{pd_prox_exg_epsilon_l1inf_time_iterations_lamda_1e-5.eps}}\hspace*{0.2in}

\subfigure[Multi-task learning: $\epsilon$-insensitive loss and $\ell_{1,\infty}$ regularizer with $\lambda=10^{-3}$. ]{\includegraphics[scale=0.25]{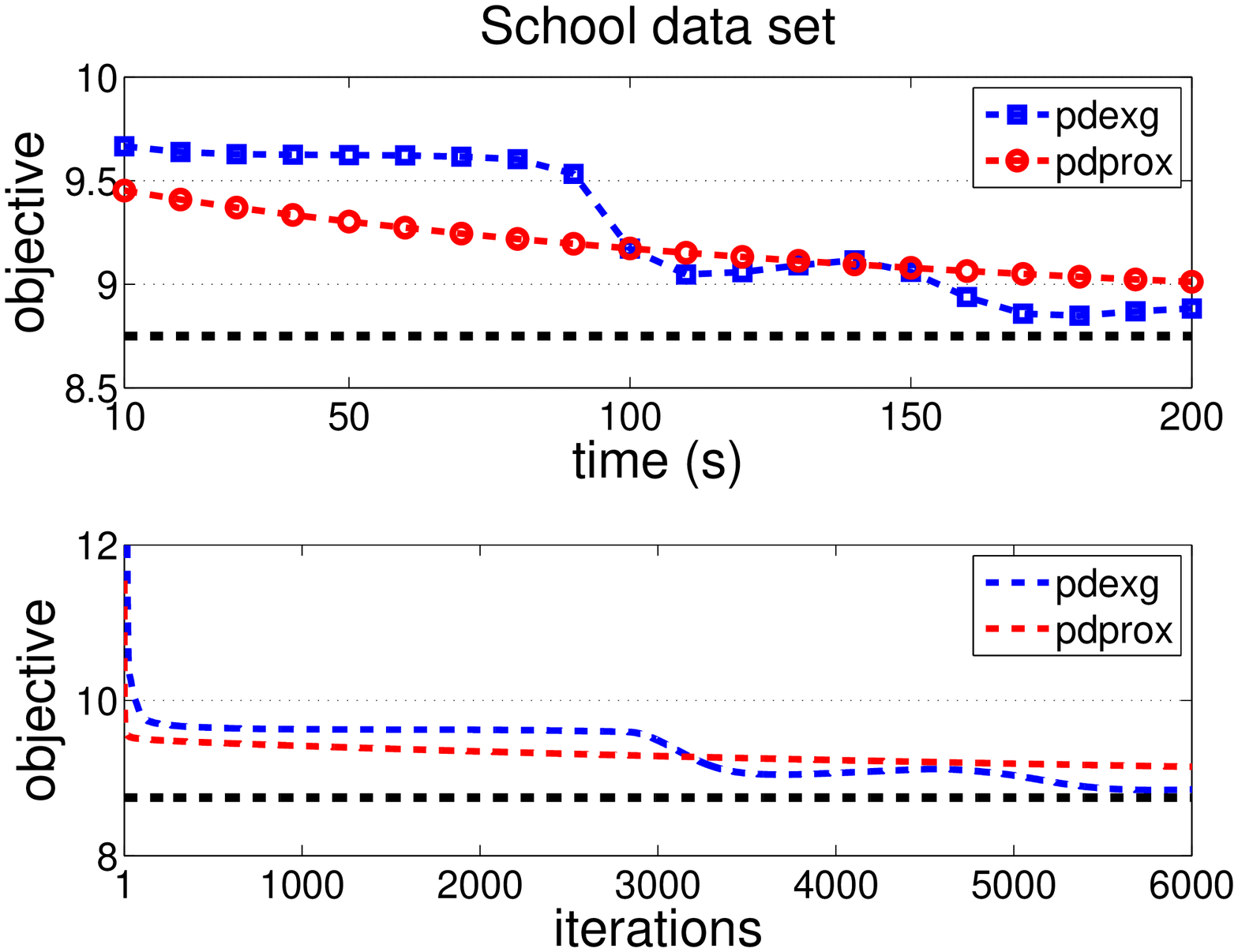}}\hspace*{0.2in}
\subfigure[Multi-task learning: $\epsilon$-insensitive loss and $\ell_{1,\infty}$ regularizer with $\lambda=10^{-5}$. ]{\includegraphics[scale=0.25]{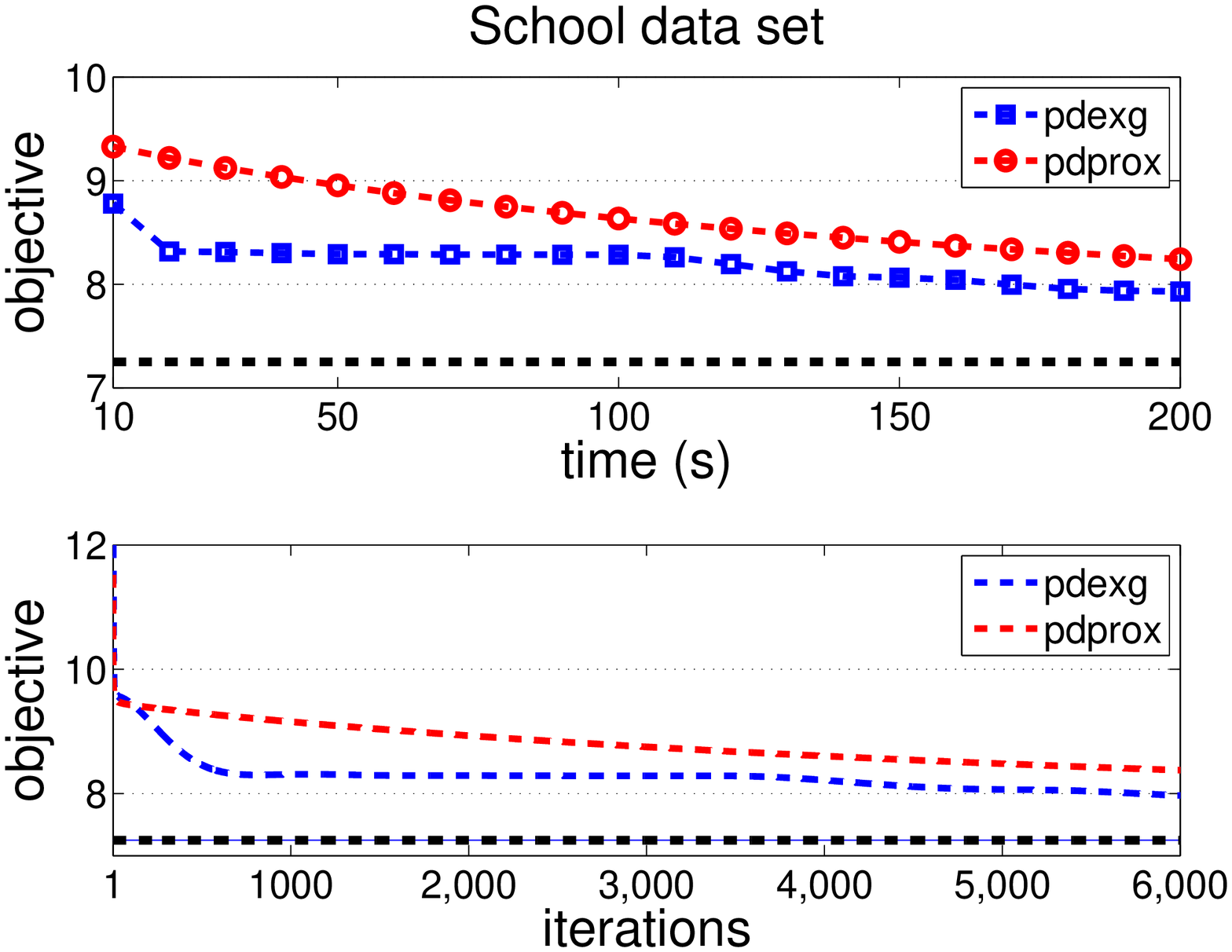}}\hspace*{0.2in}

%\subfigure[Matrix completion: absolute loss and trace norm regularizer with $\lambda=10^{-3}$.]{\includegraphics[scale=0.25]{pd_prox_exg_abs_trace_time_iterations_lamda_1e-3.eps}}\hspace*{0.2in}
%\subfigure[Matrix completion: absolute loss and trace norm regularizer with $\lambda=10^{-5}$.]{\includegraphics[scale=0.25]{pd_prox_exg_abs_trace_time_iterations_lamda_1e-5.eps}}\hspace*{0.2in}

\subfigure[Matrix completion: absolute loss and trace norm regularizer with $\lambda=10^{-3}$.]{\includegraphics[scale=0.25]{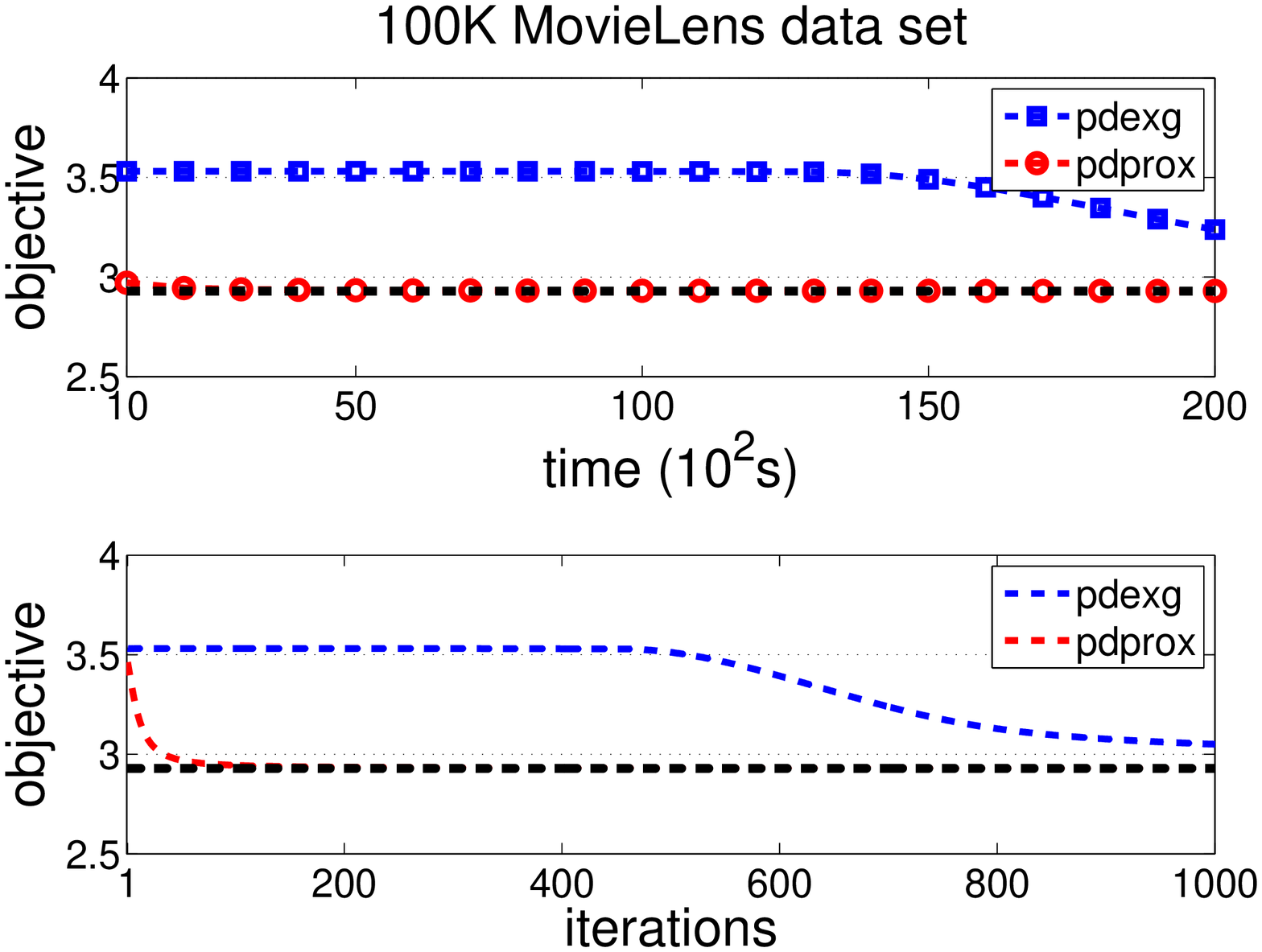}}\hspace*{0.2in}
\subfigure[Matrix completion: absolute loss and trace norm regularizer with $\lambda=10^{-5}$.]{\includegraphics[scale=0.25]{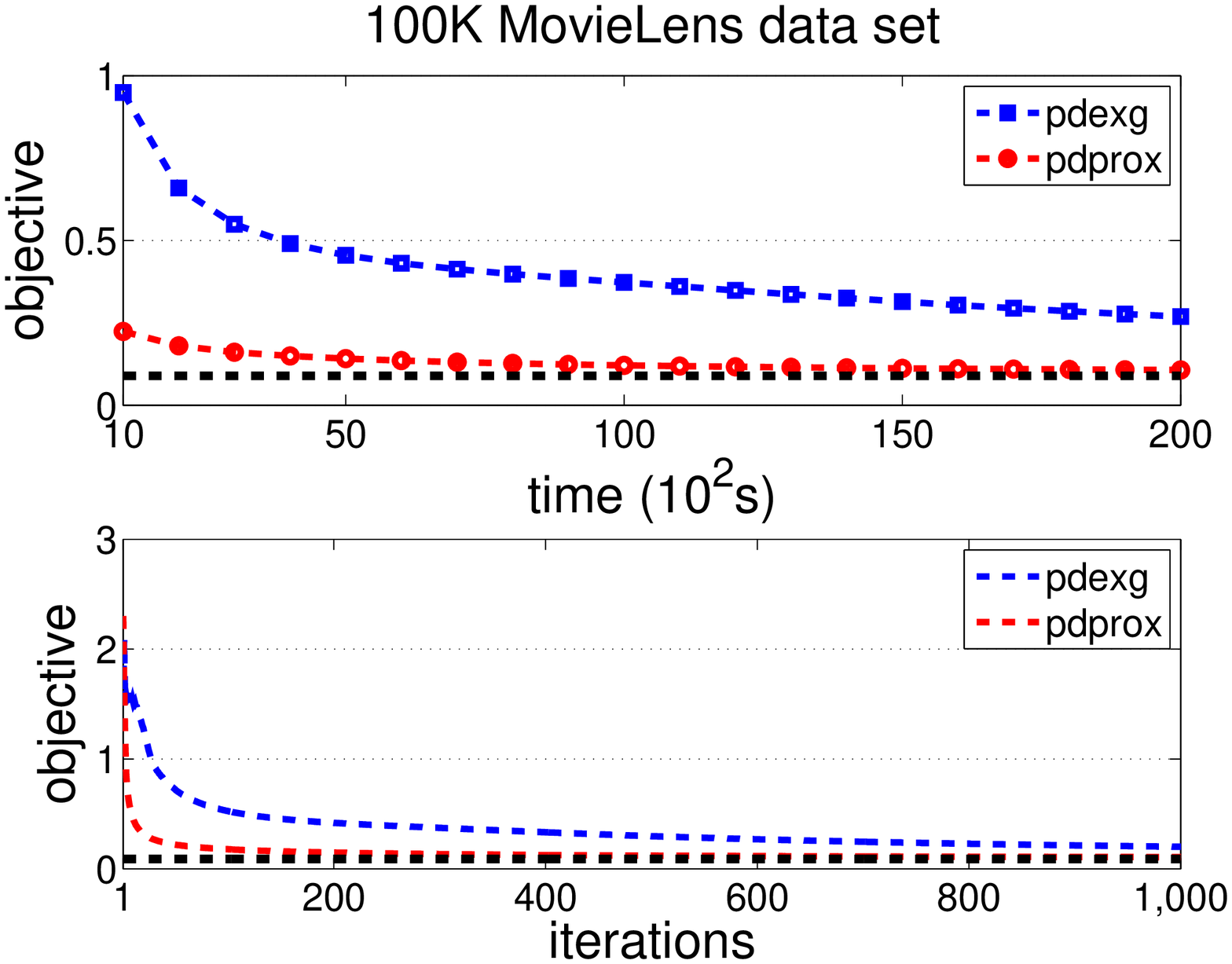}}\hspace*{0.2in}

\caption{Pdprox vs Primal-Dual method with excessive gap technique. {The black bold dashed lines in all Figures show the optimal objective value by running Pdprox with a large number of iterations so that the difference between the last two objective values is less than $10^{-4}$.}}\label{fig:exg}
\end{figure} 

The performance of the two algorithms on the three tasks is shown  in Figure~\ref{fig:exg}. Since both algorithms are in the same category, i.e. updating both primal and dual variables at each iteration and having a convergence rate in the order of $O(1/T)$,  we also plot the objective versus the number of  iterations in the bottom panels of each subfigure in Figure~\ref{fig:exg}.
%Before discussing the obtained results, we need to make a point clear about Figure~\ref{fig:exg}. Although both algorithms are primal dual methods and share the same convergence rate $O(1/T)$, but in terms of running time, there is a slight difference between two algorithms. The proposed Pdprox algorithm  is more efficient than Pdexg at each single iteration since Pdexg needs to solve several auxiliary optimization problems at each iteration. Hence, due to the mentioned difference, it is more meaningful to compare the objective value versus the number of iterations instead of comparing the objective versus running time. 
The results show that the proposed Pdprox method converges faster than Pdexg on MEMset Donar data set for group feature selection  with hinge loss and group lasso regularizer, and on 100K MovieLens data set for matrix completion with absolute loss and trace norm regularizer. However, Pdexg performs better on School data set  for multi-task learning with $\epsilon$-insensitive loss and $\ell_{1,\infty}$ regularizer. One interesting phenomenon we can observe from Figure~\ref{fig:exg}  is that for larger values of  $\lambda$ (e.g., $10^{-3}$), the improvement of Pdprox over Pdexg is also larger. The reason is that the proposed Pdprox captures the sparsity of primal variable at each iteration. This does not hold for  Pdexg because it casts the non-smooth regularizer into a dual form and consequently does not explore the sparsity of the primal variable at each iteration. Therefore the larger of $\lambda$, the sparser of the primal variable at each iteration in Pdprox that yields  to larger improvement  over Pdexg. For the example of group feature selection task with hinge loss and group lasso regularizer,  when setting $\lambda=10^{-3}$, the sparsity of the primal variable (i.e., the proportion of the number of group features with zero norm) in Pdprox averaged over all iterations is $0.7886$. However, by reducing $\lambda$ to $10^{-5}$ the average sparsity of the primal variable in Pdprox is reduced to $0$. In both settings the average sparsity of the primal variable in Pdexg is $0$.  The same argument also  explains why Pdprox does not perform as well as Pdexg on School data set when setting $\lambda=10^{-5}$, since in this case the primal variables in both algorithms are not sparse. When setting $\lambda=10^{-3}$, the average sparsity (i.e., the proportion of the number of features with zero norm across all tasks) of the primal variable in Pdprox and Pdexg is $0.3766$ and $0$, respectively. 
Finally, we also observe similar performance of the two algorithms on the three tasks with other loss functions including absolute loss for group feature selection, absolute loss for multi-task learning,  and hinge loss for max-margin matrix factorization.

%Note that we plot the objective versus the number of iterations instead of running
%time is because both algorithms are primal dual method and achieve $O(1/T)$ convergence
%rate, and therefore by visualizing the objective versus the number of iterations we can compare the convergence rate of the two algorithms empirically. However, the proposed
%algorithm Pdprox is more efficient than Pdexg at each iteration, since Pdexg needs to
%solve several auxiliary optimization problems at each iteration, and hence comparing the
%objective versus the number of iterations is less fair for Pdprox than comparing the objective
%versus running time. The results show that the proposed Pdprox method converges faster than Pdexg on MEMset Donar data set for group feature selection  with hinge loss and group lasso regularizer and 100K MovieLens data set for matrix completion with absolute loss and trace norm regularizer. However, Pdexg performs better on School data set  for multi-task learning with $\epsilon$-insensitive loss and $\ell_{1,\infty}$ regularizer. We also observe similar performance of the two algorithms on the three tasks with other loss functions, e.g. absolute loss for group feature selection, absolute loss for multi-task learning,  and hinge loss for max-margin matrix factorization. 

\begin{table}[t]
%\centering\caption{Comparison of running time (second) and classification accuracy  for~(\ref{eqn:constrained-obj}) with varied $m$}\label{tab:m}\small{
%\begin{tabular}{llcc}
%\toprule
% Data Set&m&Running time & Accuracy\\
% \midrule
% %&m=10&9s&0.8009\\
% %&m=10&0.1s&0.8009\\
% %\raisebox{0.1in}{ionosphere}&m=250(n)&49s&0.7381\\
%% \raisebox{0.1in}{ionosphere}&m=n=250&0.05s&0.7381\\
%% \midrule
%%&m=10&56s&0.8328\\
%% &m=10&0.7s&0.8304\\
% %\raisebox{0.1in}{arcene}&m=100(n)&1792s&0.7232\\
%% \raisebox{0.1in}{arcene}&m=n=100&2s&0.7232\\
% %\midrule
%% &m=100&81s&0.7700\\
%&m=100&13s&0.7675\\
% %\raisebox{0.1in}{a9a}&m=32432(n)&$>$5h$^*$&0.7593\\
% \raisebox{0.1in}{a9a}&m=n=32432&125s&0.7589\\
% \midrule
% %&m=100&795s&0.9580\\
% &m=100&19s&0.9558\\
% %\raisebox{0.1in}{rcv1}&m=22421(n)&$>4$h$^*$&0.9654\\
% \raisebox{0.1in}{rcv1}&m=n=22421&147s&0.9654\\
% \midrule
%\bottomrule
%\end{tabular}}
\centering\caption{Running time (forth column) and classification accuracy (fifth column) of Pdprox for~(\ref{eqn:constrained-obj}) and of Liblinear on noisily labeled training data, where noise is added to labels by random flipping with a probability $0.2$. We  fix $\lambda=1/n$ or $C=1$ in Liblinear. In the second column, we report the number of training examples ($n$), the number of attributes ($d$), and also the accuracy by training Liblinear on the original data and evaluating it on the testing data. }\label{tab:m}\small{
{
\begin{tabular}{lllcc}
\toprule
 Data Set&($n,d$)/ACC&Alg.&Running Time & ACC\\
 \midrule
 %&m=10&9s&0.8009\\
 %&m=10&0.1s&0.8009\\
 %\raisebox{0.1in}{ionosphere}&m=250(n)&49s&0.7381\\
% \raisebox{0.1in}{ionosphere}&m=n=250&0.05s&0.7381\\
% \midrule
%&m=10&56s&0.8328\\
% &m=10&0.7s&0.8304\\
 %\raisebox{0.1in}{arcene}&m=100(n)&1792s&0.7232\\
% \raisebox{0.1in}{arcene}&m=n=100&2s&0.7232\\
 %\midrule
% &m=100&81s&0.7700\\
{a9a}&(32561, 123)&Pdprox(m=200)&0.82s(0.01)&0.8344(0.00)\\
 %&m=32432(n)&$>$5h$^*$&0.7593\\
%(n=32561, d=123)&m=32561&684s&0.8860\\
&0.8501& Liblinear&1.15s(0.57)&0.7890(0.00)\\
 \midrule
 %&m=100&795s&0.9580\\
 rcv1&(20242, 47236)&Pdprox(m=200)&1.57s(0.23)&0.9405(0.00)\\
 %\raisebox{0.1in}{rcv1}&m=22421(n)&$>4$h$^*$&0.9654\\
%(n=20242, d=47236)&m=20242&952s&0.9772\\
&0.9654&Liblinear&3.30s(0.74)&0.9366(0.00)\\
 \midrule
covtype&(571012, 54)  &Pdprox(m=4000)&48s(3.34)&0.7358($0.00$)\\
 %\raisebox{0.1in}{rcv1}&m=22421(n)&$>4$h$^*$&0.9654\\
 %(n=571012, d=54)&m=571012&22524s&0.7768\\
&0.7580&Liblinear&37s(0.64)&0.6866($0.00$)\\
\bottomrule
\end{tabular}}}

\end{table}

\subsection{Sparsity constraint on the dual variables}\label{sec:exp4}
In this subsection, we examine empirically the proposed algorithm for optimizing the problem in equation~(\ref{eqn:constrained-obj}), in which a sparsity constraint is introduced for the dual variables.  We test the algorithm on three large data sets from the UCI repository, namely, a9a, rcv1(binary) and covtye\footnote{\url{http://www.csie.ntu.edu.tw/~cjlin/libsvmtools/datasets}}.  In the experiments we use $\ell^2_2$ regularizer and  fix $\lambda = 1/n$.   First, we  run the proposed algorithm 100 seconds on the three data sets with different values of $m=100, 200, 400$ and plot the objective versus the number of iterations. The results are  shown in Figure~\ref{fig:m}, which verify that the  convergence is faster with smaller $m$, which is consistent with  the convergence bound $O([D + m]/[\sqrt{2n}\lambda]$) of the proposed algorithm for~(\ref{eqn:constrained-obj}).%, i.e., the smaller the $m$, the faster the convergence.

Second, we demonstrate that the formulation in equation~(\ref{eqn:constrained-obj}) with a sparsity constraint on the dual variables is useful in the case when labels are contaminated with noise. To generate the noise in labels, we randomly flip the labels with a probability $0.2$. We run both the proposed algorithm for~(\ref{eqn:constrained-obj}) and Liblinear\footnote{\url{http://www.csie.ntu.edu.tw/~cjlin/liblinear}} on the training data with noise added to the labels. The stopping criterion for the proposed algorithm is when duality gap is less than $10^{-3}$, and for Liblinear is when the maximal dual violation is less than $10^{-3}$. 
The running time and accuracy on testing data  averaged over $5$ random trials are reported in Table~\ref{tab:m}, which demonstrate that in the presence of noise in labels, by adding a sparsity constraint on the dual variables, we are able to obtain better performance than Liblinear trained on the noisily labeled data. Furthermore the running time of Pdprox is comparable to, if not less than, that of Liblinear.  

Finally, we note that choosing a small $m$ in equation~(\ref{eqn:constrained-obj}) is different from simply training a classifier with a small number of examples. For instance, for rcv1, we have run the experiment with $200$ training examples, randomly selected from the entire data set. With the same stopping criterion, the testing performance is $0.8131(\pm 0.05)$, significantly lower than that of optimizing~(\ref{eqn:constrained-obj}) with $m = 200$.

\begin{figure}[t]
\centering
\subfigure[a9a]{\includegraphics[scale=0.18]{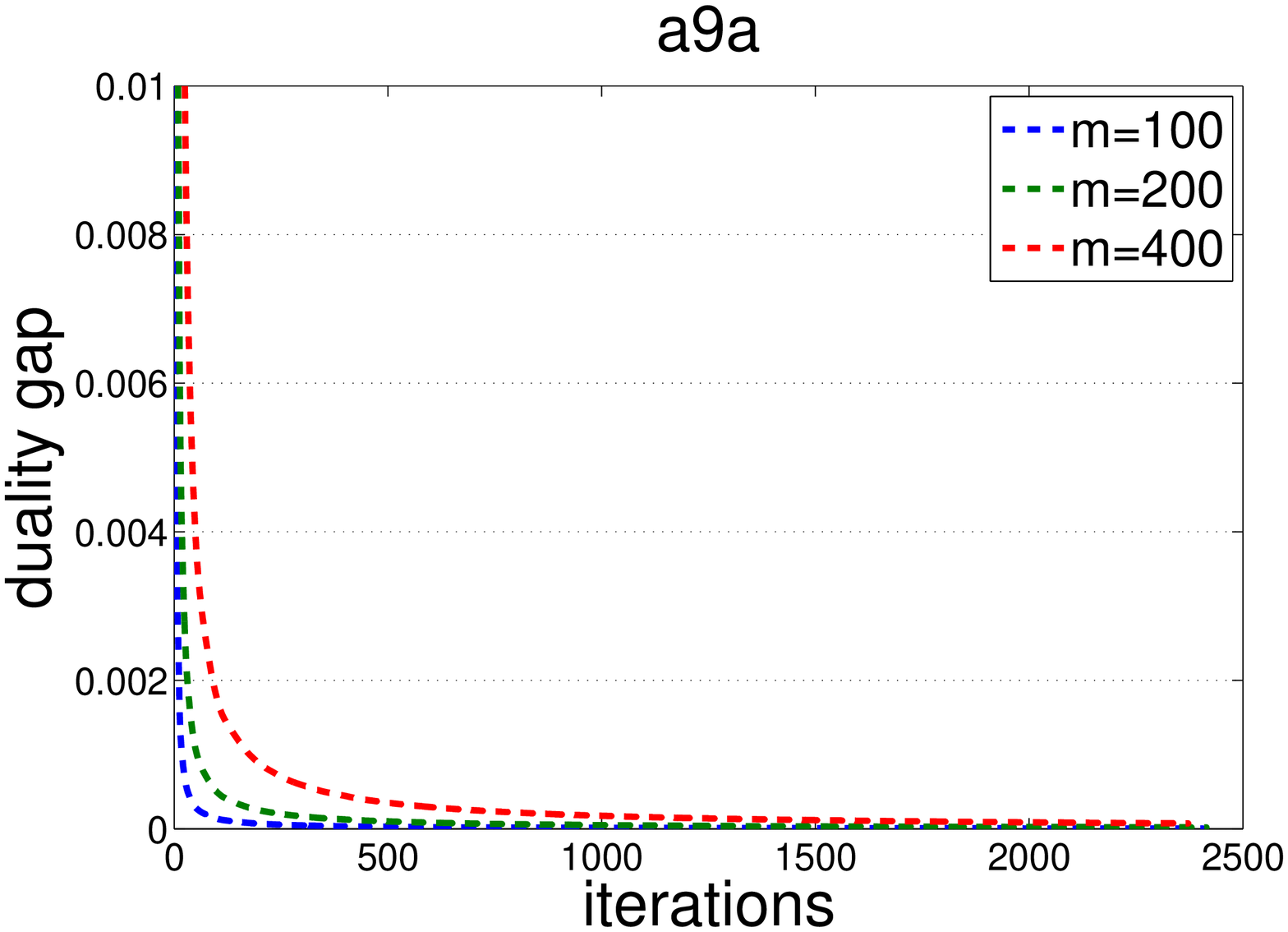}}\hspace*{-0.1in}
\subfigure[rcv1]{\includegraphics[scale=0.18]{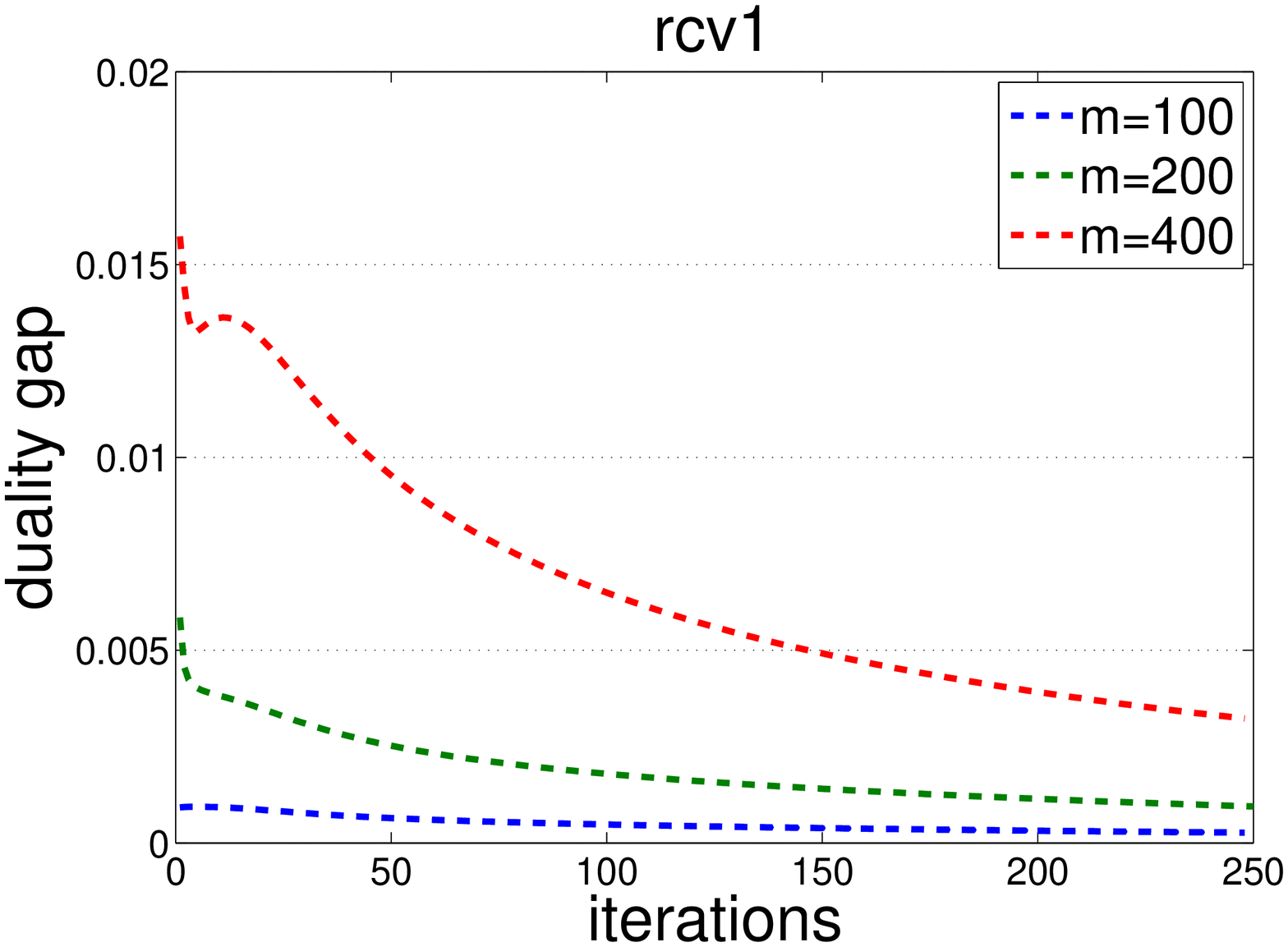}}\hspace*{-0.1in}
\subfigure[covtype]{\includegraphics[scale=0.18]{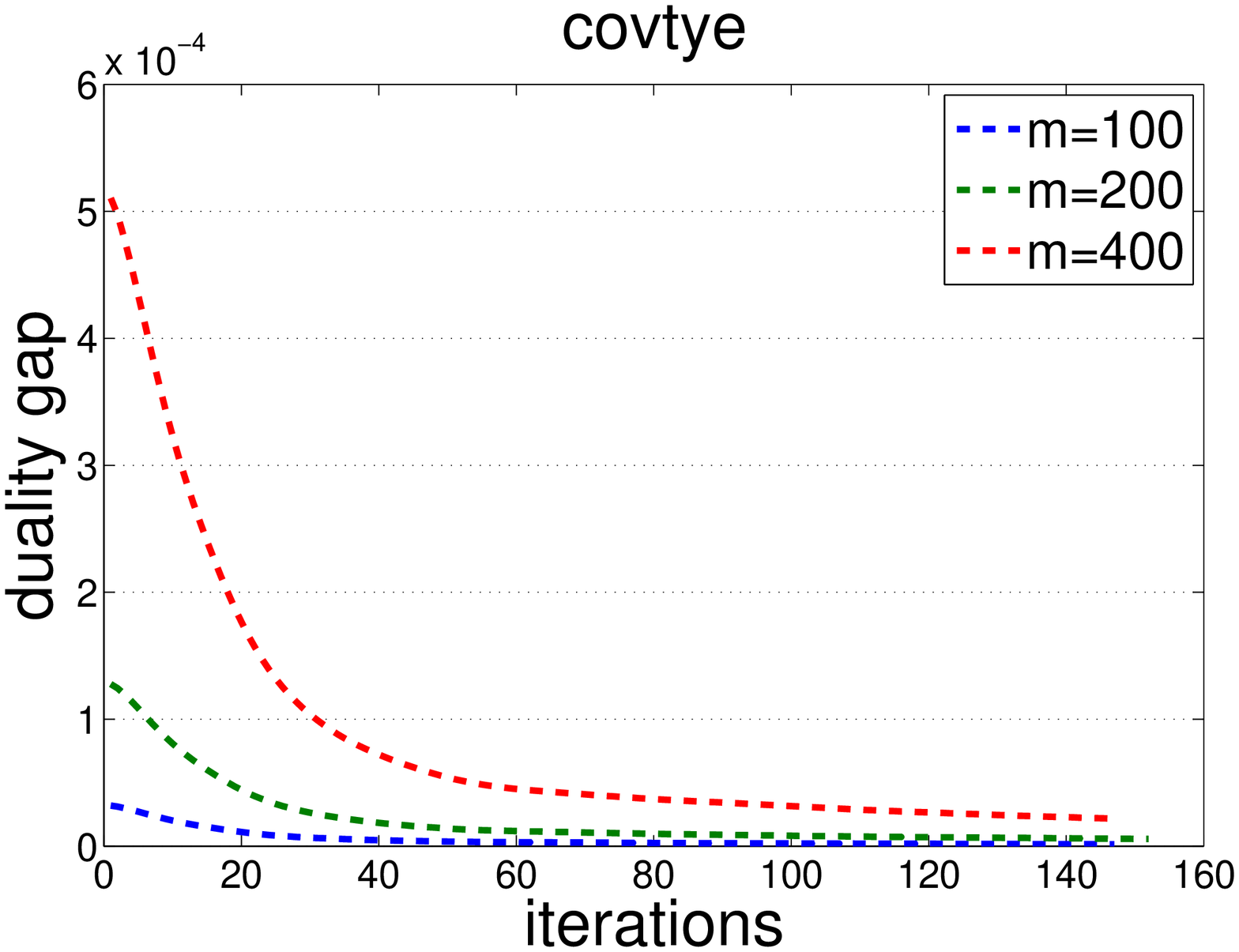}}\hspace*{-0.1in}
\caption{Comparison of Convergence with varied $m$.}\label{fig:m}
\end{figure}

%Finally, we  present the experimental results showing the converge time vs. $n$ with small value of $m$. We use rcv1 data %set, and constructing three training sets with $n=1000, 10000, 100000$ from the whole data. In all cases, we fixed $\lambda=10^{-4}$ and $m=100$. We run \textbf{pdprox}  with $\ell_2^2$ regularizer, and stop the algorithm when obtaining a solution in precision of $\epsilon=10^{-5}. $ The result is shown shown in Table~\ref{tab:n}. The BAR is evaluated on the same  separate $10000$ examples in the three cases. We can see that  the converge time on large $n=10000, 100000$ is smaller than small $n=1000$, and we still obtain very good classifier.
%%The number of iterations for obtaining a solution in precision of $10^{-4}$ on the three training sets are $T=40000, 1000, %100$, and the corresponding cpu time is $289s, 61s, 142s$.
%
%\begin{table}[t]
%\centering\caption{Comparison of Convergence for increasing $n$ with fixed $m$ and $\lambda$.}\label{tab:n}
%\begin{tabular}{llcc}
%\toprule
%n&Iterations&CoT&BAR\\
% \midrule
% 1000&459000&3112s&0.9246\\
% %10000&600&54s&0.9589(epsilon=1e-4)\\
% 10000&2200&179s&0.9573\\
% 100000&200&285s&0.9673\\
% %100000&100&116s&0.9579(epsilon=1e-4)\\
%  %100000&10&35s&0.9252(epsilon=1e-4)\\
%\bottomrule
%\end{tabular}
%\end{table}

\subsection{Comparison:  double-primal vs double-dual implementation}\label{sec:exp5}
%We have discussed the difference between   Pdprox-primal algorithm and Pdprox-dual and  the Primal Dual algorithm in~\citep{Chambolle:2011:FPA:1968993.1969036}, to which we refer as \textbf{Pdsplit}. %, in the order of updating the primal variable and the dual variable and the gradients used in the updating as well. Pdprox-dual algorithm (Algorithm~\ref{alg:3}) is the same to Pdsplit in terms of the updates on $(\balpha_t, \w_t)$ but different  in the number of maintained copies of the primal and the dual variables. 
%From  the algorithms presented in Section~\ref{sec:algo} and the comparison shown in Appendix~\ref{sec:apc},
From the discussion in subsection~\ref{sec:imp}, we have seen  that  both Pdprox-primal and Pdprox-dual algorithm can be implemented either  by maintaining two dual variables, to which we refer as double-dual implementation,  or by maintaining two primal variables, to which we refer as double-primal implementation. One implementation could be more efficient than the other implementation, depending on the nature of applications.  %There exist many applications that can be considered for demonstration. 
For example, in multi-task regression with $\ell_2$ loss~\citep{citeulike:9315911}, if the number of examples is much larger than the number of attributes, i.e., $n\gg d$, and the number of tasks $K$ is large, then the size of  dual variable $\alpha\in\R^{n\times K}$ is much larger than the size of primal variable $W\in\mathbb R^{d\times K}$. It would be expected that the double-primal implementation is more efficient than the double-dual implementation. In contrast, in matrix completion with absolute loss, if the number of observed entries $|\Omega|$ which corresponds to the size of dual variable is much less than the total number of entries $n^2$ which corresponds to the size of primal variable, then the double-dual implementation would be more efficient than the double-primal implementation. 

 In the following experiment, we restrict our demonstration to a binary classification problem that given a set of training examples $(\x_i,y_i),i=1,\ldots, n$, where $\x_i\in\R^d$, one aims to learn a prediction model $\w\in\R^d.$ We choose web spam data set~\footnote{\url{http://www.csie.ntu.edu.tw/~cjlin/libsvmtools/datasets/binary.html}} as the testing bed, which contains $350000$ examples, and  $16609143$ trigrams extracted for each example. We use hinge loss and $\ell_2^2$ regularizer with $\lambda=1/n$, where $n$ is the number of experimented data. 
  %Two different set of features are extracted, 254 unigrams and 16609143 trigrams. In the experiments below, we use hinge loss and $\ell^2_2$ regularizer with $\lambda=1/n$, where $n$ is the number of experimented data. 

%In the first experiment, we show that when $n\gg d$,  which is very common in real applications where the number of attributes is limited due to the difficulty or high cost in obtaining the attributes, the double-primal implementation is more efficient than double-dual implementation. We use the unigrams as the attributes. 

%In the second experiment, 
We demonstrate that when $d\gg n$, the double-dual implementation is more efficient than double-primal implementation. For  the  purpose of demonstration, we  randomly sample from the whole data a subset of $n=1000$ examples, which have a total of $8287348$ features,  and we solve the sub-optimization problem over the subset.   It is worth noting that such kind of  problem appears commonly in distributed computing  on individual nodes when the number of attributes is huge. %For example in ADMM~\citep{Boyd:2011:DOS:2185815.2185816}, at each iteration we need to solve a sub-optimization problem on each individual node over a subset of the whole data. %when the size of the primal variable is much larger than the size of the dual variable, i.e., $d\gg n$,  Pdprox-dual implemented by Algorithm~\ref{alg:3} (or Pdprox-primal implemented by maintaing two dual variables and one primal variable) is more preferable than Pdsplit.  
%We use the  trigrams as attributes. 
The objective value versus running time of the two implementations of Pdprox-dual are plotted in Figure~\ref{fig:pdvs}, which shows that double-dual implementation  is more efficient than double-primal implementation is this case. %Finally, we note that the performance of Pdprox-primal is  close to Pdsplit, where the difference comes from the difference in updating the extra primal variable. 
As a complement, we also plot the objective of Pdprox-dual and Pdprox-primal both with double-dual implementation, which shows that Pdprox-primal and Pdprox-dual performs similarly. 

\begin{figure}[t]
\centering
\includegraphics[scale=0.18]{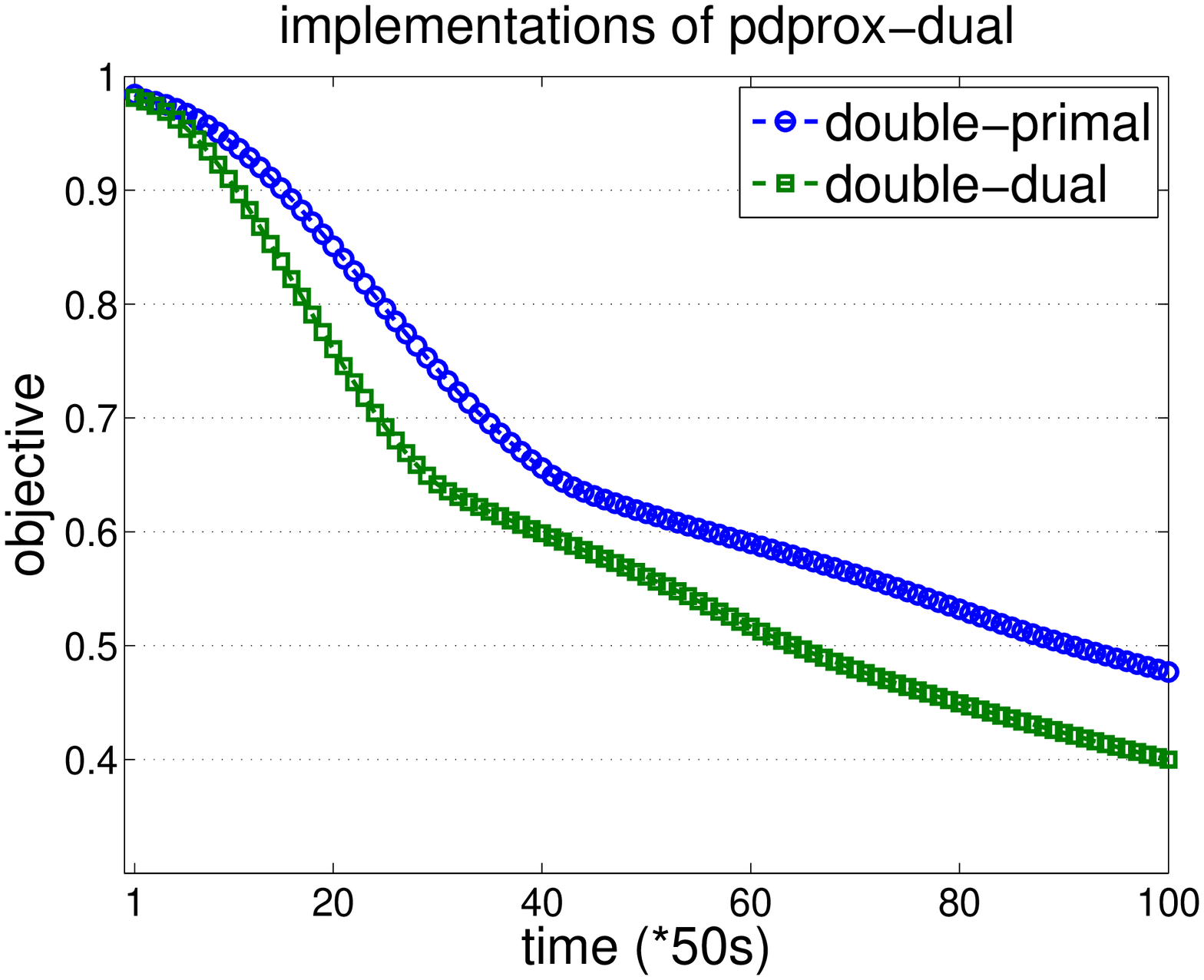}\hspace*{-0.1in}
\includegraphics[scale=0.18]{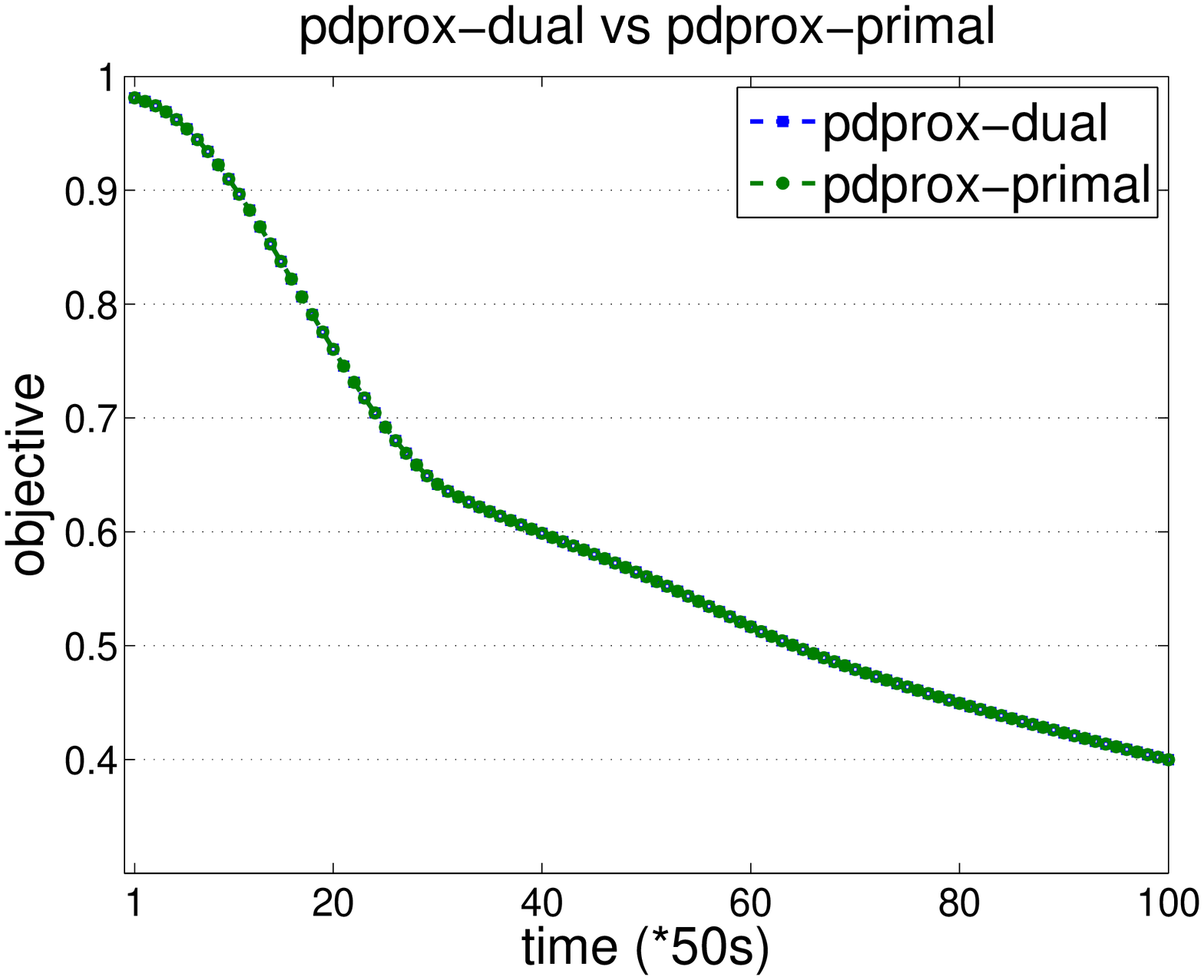}

%\subfigure[covtype data]{\includegraphics[scale=0.18]{pdprox_primal_vs_dual_hl.eps}}\hspace*{-0.1in}
%\subfigure[covtype data]{\includegraphics[scale=0.18]{pdprox_primal_vs_dual_ghl.eps}}\hspace*{-0.1in}
%\subfigure[a subset of news20 data ]{\includegraphics[scale=0.18]{pdprox1_vs_pdprox2_news20_l2svm.eps}}\hspace*{-0.1in}
%%\subfigure[a subset of news20 data]{\includegraphics[scale=0.27]{pdprox1_vs_pdprox2_news20_l2svm.eps}}\hspace*{-0.1in}
\caption{Comparison of double-primal implementation vs. double-dual implementation of Pdprox-dual, and Comparison of Pdprox-dual vs. Pdprox-primal both with double-dual implementation, on a subset of webspam data using trigram features.}\label{fig:pdvs}
\end{figure}

{
\subsection{Comparison for solving $\ell^2_2$ regularized SVM}\label{sec:exp5}
In this subsection, we compare the proposed Pdprox method with Pegasos for solving $\ell_2^2$ regularized SVM when $\lambda = O(n^{-1/(1+\epsilon}), \epsilon\in(0,1]$. We also compare Pdprox using one step size and two step sizes, and compare them to the accelerated version proposed in~\citep{Chambolle:2011:FPA:1968993.1969036} for strongly convex functions. %For a fair comparison,  we use $\ell_2^2$ norm for the regularizer since Pegasos is designed for $\ell_2^2$ regularizer.  
We implement Pdprox-dual algorithm (by double-dual implementation) in C++  using the same data structures as coded by Shai Shalev-Shwartz~\footnote{\url{http://www.cs.huji.ac.il/~shais/code/index.html}}. 
% $10^4$ iterations for the two different domains of $\balpha$: $\Q_{\balpha}=\{\balpha: \balpha\in[0,1]^n\}$ and $\Q_{\balpha}=\{\balpha: \balpha\in[0,1]^n, \|\balpha\|_1\leq 100\}$. The parameter $\lambda$ is fixed to $1/n$.  At the end of $10^4$ iterations, two algorithms have the same objective value. We can see that Pdprox-primal speeds up Pdprox-dual, especially when the domain of $\balpha$ is complex. 
\begin{figure}[t]
\centering\hspace*{-0.1in}
%\subfigure[$\lambda=n^{-0.5}$]{\includegraphics[scale=0.18]{pdprox_vs_pegasos_n_0p5_time.eps}}\hspace*{-0.1in}
%\subfigure[$\lambda=n^{-0.8}$]{\includegraphics[scale=0.18]{pdprox_vs_pegasos_n_0p8_time.eps}}\hspace*{-0.1in}
%\subfigure[$\lambda=n^{-1}$]{\includegraphics[scale=0.18]{pdprox_vs_pegasos_n_1_time.eps}}
%
%
%
%\hspace*{-0.1in}
%\subfigure[$\lambda=n^{-0.5}$]{\includegraphics[scale=0.18]{pdprox_vs_pegasos_n_0p5_iterations.eps}}\hspace*{-0.1in}
%\subfigure[$\lambda=n^{-0.8}$]{\includegraphics[scale=0.18]{pdprox_vs_pegasos_n_0p8_iterations.eps}}\hspace*{-0.1in}
%\subfigure[$\lambda=n^{-1}$]{\includegraphics[scale=0.18]{pdprox_vs_pegasos_n_1_iterations.eps}}
%
%\hspace*{-0.1in}
%\subfigure[$\lambda=n^{-0.5}$]{\includegraphics[scale=0.193]{pdprox_vs_pegasos_n_0p5_500iterations.eps}}\hspace*{-0.1in}
%\subfigure[$\lambda=n^{-0.8}$]{\includegraphics[scale=0.193]{pdprox_vs_pegasos_n_0p8_500iterations.eps}}\hspace*{-0.1in}
%\subfigure[$\lambda=n^{-1}$]{\includegraphics[scale=0.193]{pdprox_vs_pegasos_n_1_500iterations.eps}}

\subfigure[$\lambda=n^{-0.5}$]{\includegraphics[scale=0.2]{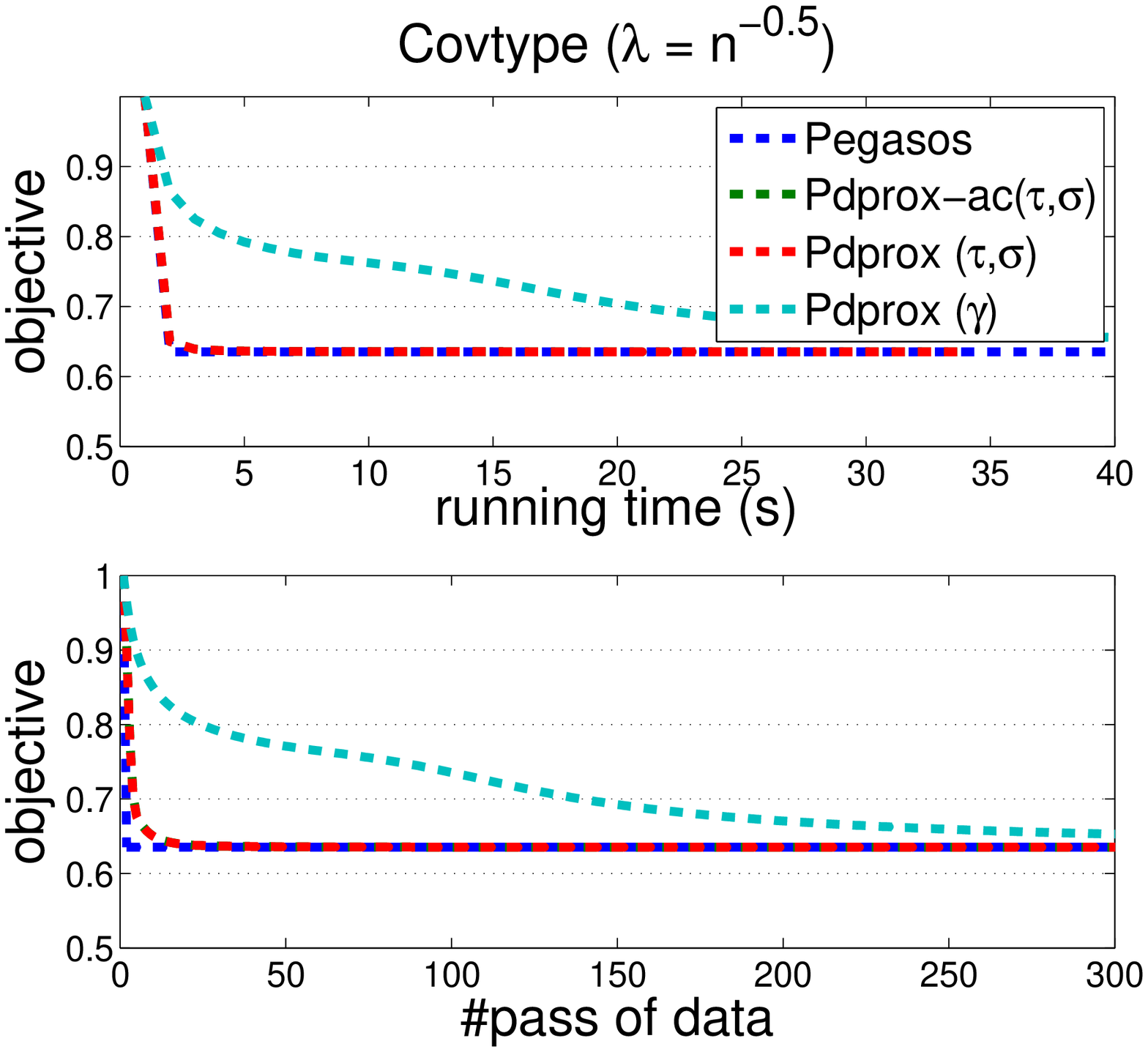}}\hspace*{-0.1in}
\subfigure[$\lambda=n^{-0.8}$]{\includegraphics[scale=0.2]{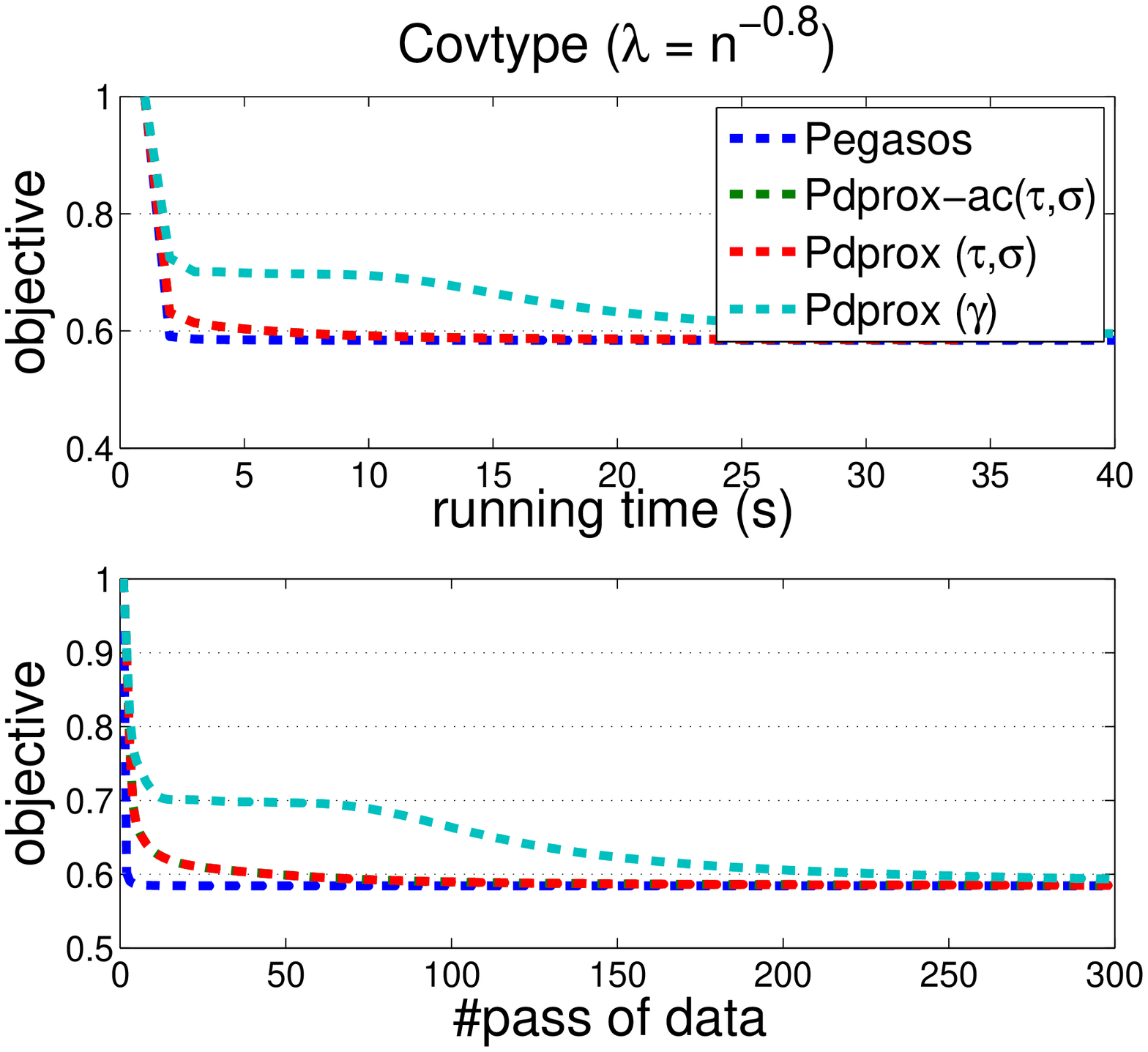}}\hspace*{-0.1in}
\subfigure[$\lambda=n^{-1}$]{\includegraphics[scale=0.2]{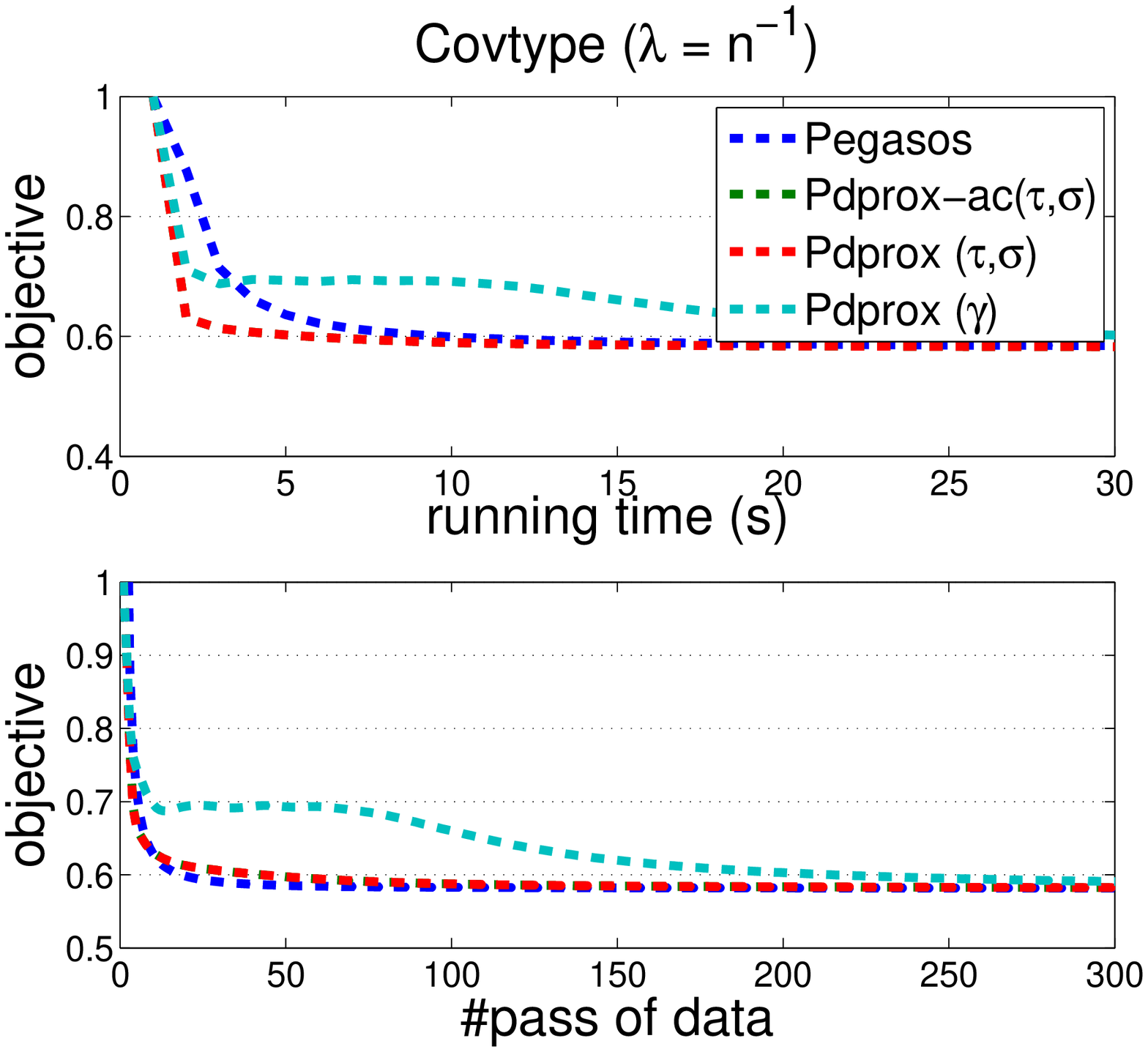}}

\caption{Comparison of convergence speed of Pdprox vs. Pegasos on covtype data set. The best ratio between the step size $\tau$ for updating $\w$ and the step size $\sigma$ for updating $\balpha$ is 0.01. The curves of Pdprox-ac($\tau,\sigma$) are almost identical to that of Pdprox ($\tau,\sigma$).  }\label{fig:3}
\end{figure}

Figure~\ref{fig:3} shows the comparison of \textbf{Pdprox} vs. \textbf{Pegasos} on covtype data set with three different levels of $\lambda=n^{-0.5}, n^{-0.8}, n^{-1}$.  We compute the objective value of Pdprox after each iteration and compute the objective value of Pegasos after one effective pass of all data (i.e., $n$ number of iterations where $n$ is the total number of training examples). We also compare the one step size scheme (Pdprox ($\gamma$)) with  the two step sizes scheme (Pdprox ($\tau,\sigma$)) and  the accelerated version (Pdprox-ac($\tau,\sigma$)) proposed in~\citep{Chambolle:2011:FPA:1968993.1969036} for strongly convex functions. The relative ratio between the step size $\tau$ for updating the primal variable and the step size $\sigma$ for updating the dual variable  is selected among a set of values $\{1000, 100, 10, 1, 0.1, 0.01, 0.001\}$.

The results demonstrate that (1) the two step sizes scheme with careful tuning of the relative ratio yields better convergences than the one step size scheme; (2) Pegasos still remains a state-of-the-art algorithm for solving the $\ell_2^2$ regularized SVM; but when the problem is relatively difficult, i.e., $\lambda$ is relatively small (e.g., less than $1/n$), the Pdprox algorithm with the two step sizes may converge faster in terms of running time; (3) the accelerated version for solving SVM  is almost identical the basic version. 
}
%For Pdprox, we use the Pdprox-primal algorithm. Since Pegasos can perform stochastic updating,  we also report the performance of Pegasos algorithm by sampling $1$ example and $1\%$ examples for computing the stochastic gradient. We run both algorithms 1000 seconds, report the objective versus time in the top panel of Figure~\ref{fig:3}, and report the objective versus the number of iterations (achieved in  1000 seconds) in the middle panel. In the bottom panel, we also report the objective versus the first 500 iterations to better understand the behavior of algorithms at the beginning of iterations.  The results show that { (i) Pdprox converges  faster than Pegasos when $\lambda = O(n^{-1/(1+\tau)}), \tau = 0, 0.25,1$, which verifies our discussion in subsection~\ref{sec:ext}; (ii) the objective of Pegasos always has a sharp jump at the beginning of iterations, which becomes severe when $\lambda$ is reduced. This is because at the beginning, the gradient of the primal variable is not informative compared to that in Pdprox which uses the updated dual variables to compute the gradient.   Finally,  we observe that the performance of stochastic Pegasos with sampling $1\%$ examples becomes better than the performance of batch Pegasos when $\lambda$ is reduced to be less than $n^{-0.8}$,  however, stochastic Pegasos  performs consistently worse than Pdprox.}

\section{Conclusions}\label{sec:conc}
In this paper, we study non-smooth optimization in machine learning where both the loss function and the regularizer are non-smooth. We develop an efficient gradient based method for a family of non-smooth optimization problems in which the dual form of the loss function can be expressed as a bilinear function in primal and dual variables. {We show that,  assuming the proximal step can be efficiently solved,  the proposed algorithm achieves a convergence rate of $O(1/T)$, faster than $O(1/\sqrt{T})$ suffered by many other first order methods for non-smooth optimization. In contrast to  existing studies on non-smooth optimization, our work enjoys more simplicity in implementation and analysis, and provides a unified methodology  for a diverse set of non-smooth optimization problems. Our empirical studies demonstrate the efficiency of the proposed algorithm in comparison with the state-of-the-art first order methods for solving many non-smooth machine learning problems, and the effectiveness of the proposed algorithm for optimizing the problem with a sparse constraint on the dual variables for tackling the noise in labels.} In future, we plan to adapt the proposed algorithm for stochastic updating and for distributed computing environments.

%In this appendix we prove the following theorem from
%Section~\ref{sec:textree-generalization}:
%\bibliographystyle{spmpsci} 
\bibliography{icml11}

\appendix

\section{Derivation of   constant $c$ for (generalized) hinge loss}
\label{app:constant}
%as used in Lemma~\ref{lem:1}
As mentioned before, it is easy to derive the constant $c$ in equations~(\ref{eqn:const}) and~(\ref{eqn:const2})  for the non-smooth loss functions listed before under the assumption  that $\|\x\|_2\leq R$.  As an example, here we derive the constant  for  hinge loss and generalized hinge loss. For other non-smooth loss functions, we can derive the value of $c$  in a similar way.  For hinge loss,  $L(\w, \balpha)$ in~(\ref{eqn:dual2}) is given by
\begin{align*}
L(\w, \balpha; \X, \y) = \frac{1}{n}\sum_{i=1}^n \alpha_i(1- y_i \w^{\top}\x_i), 
\end{align*}
and its partial gradients are
\begin{align*}
G_{\balpha}(\w, \balpha) &= \frac{1}{n}\mathbf 1 - \frac{1}{n}(\x_1y_1,\cdots, \x_ny_n)^{\top}\w, \\
G_\w(\w,\balpha)&= - \frac{1}{n}\X(\balpha\circ \mathbf y),
\end{align*}
where $\mathbf 1$ denotes a vector of all ones, and $\circ$ denotes the element-wise product. Then, 
\begin{align*}
\|G_{\balpha}(\w_1, \balpha_1) - G_{\balpha}(\w_2,\balpha_2)\|^2_2&= \frac{1}{n^2}\sum_{i=1}^n (\w_1^{\top}\x_i y_i - \w_2^{\top}\x_iy_i)^2\leq \frac{R^2}{n}\|\w_1-\w_2\|_2^2,\\
\|G_\w(\w_1, \balpha_1) - G_\w(\w_2,\balpha_2)\|^2_2&= \frac{1}{n^2}\left\|\sum_{i=1}^n(\alpha^1_i - \alpha^2_i)y_i\x_i\right\|_2^2\leq \frac{R^2}{n}\sum_{i=1}^n(\alpha^1_i -\alpha^2_i)^2 =\frac{R^2}{n}\|\balpha_1-\balpha_2\|_2^2,
\end{align*}
which implies $c=R^2/n$. 
 For the example of generalized hinge loss, 
$L(\w, \balpha)$ in~(\ref{eqn:dual2}) is 
\begin{align*}
L(\w, \balpha; \X, \y) = \frac{1}{n}\sum_{i=1}^n \alpha^1_i(1- ay_i \w^{\top}\x_i)+ \alpha^2_i(1-y_i\w^{\top}\x_i), 
\end{align*}
where $\balpha=[\balpha^1,\balpha^2]\in\Q_{\balpha}=\{\balpha: \balpha\in [0, 1]^{n\times 2}, \balpha^1+\balpha^2\leq \mathbf 1\}$, and its partial gradients are
\begin{align*}
G_{\balpha}(\w, \balpha) &= \frac{1}{n}[\mathbf 1, \mathbf 1] - \frac{1}{n}[a(\x_1y_1,\cdots, \x_ny_n)^{\top} \w, (\x_1y_1,\cdots, \x_ny_n)^{\top} \w], \\
G_\w(\w,\balpha)&= - \frac{1}{n}\X(a(\balpha^1\circ \mathbf y) + \balpha^2\circ\mathbf y),
\end{align*}
where $\mathbf 1$ denotes a vector of all ones, and $\circ$ denotes the element-wise product. Then for any $\w_1,\w_2$ and $\balpha_1=(\balpha^{1,1}, \balpha^{2,1}), \balpha_2=(\balpha^{1,2}, \balpha^{2,2})\in \Q_{\balpha}$, given $\|\x\|_2\leq R$, we have
\begin{align*}
\|G_{\balpha}(\w_1, \balpha_1) - G_{\balpha}(\w_2,\balpha_2)\|^2_F&= \frac{a^2+1}{n^2}\sum_{i=1}^n (\w_1^{\top}\x_i y_i - \w_2^{\top}\x_iy_i)^2\leq \frac{(a^2+1)R^2}{n}\|\w_1-\w_2\|_2^2,\\
\|G_\w(\w_1, \balpha_1) - G_\w(\w_2,\balpha_2)\|^2_2&= \frac{1}{n^2}\left\|\sum_{i=1}^na(\alpha^{1,1}_i - \alpha^{1,2}_i)y_i\x_i + \sum_{i=1}^n(\alpha^{2,1}_i - \alpha^{2,2}_i)y_i\x_i\right\|_2^2\\
&\leq \frac{2a^2R^2}{n}\sum_{i=1}^n(\alpha^{1,1}_i -\alpha^{1,2}_i)^2+ \frac{2R^2}{n}\sum_{i=1}^n(\alpha_i^{2,1}-\alpha_i^{2,2})^2\\
&\leq \frac{2a^2R^2}{n}\|\balpha_1-\balpha_2\|_F^2,
\end{align*}
which implies $c=(a^2+1)R^2/n$ in equation~(\ref{eqn:const}) and $c=(2a^2R^2)/n$ in equation~(\ref{eqn:const2}). We can derive the value of $c$ in~(\ref{eqn:const}) and~(\ref{eqn:const2}) similarly for other non-smooth loss functions. 
%\end{itemize}

\section{Proof of Lemma~\ref{lem:1}}
\label{app:lem:1}
Since 
\begin{align*}
&G_{\balpha}(\w,\balpha; \X, \y) = \mathbf a(\X, \y) + H(\X, \y)^{\top}\w,\\
&G_\w(\w,\balpha; \X, \y) = \mathbf b(\X, \y) + H(\X, \y)\balpha.
\end{align*}
Then 
\begin{align*}
&\|G_{\balpha}(\w_1,\balpha_1; \X, \y)-G_{\balpha}(\w_2,\balpha_2; \X, \y)\|_2^2\leq \|H(\X, \y)^{\top}(\w_1-\w_2)\|_2^2\leq c \|\w_1-\w_2\|_2^2,
\\
&\|G_\w(\w_1,\balpha_1; \X, \y)-G_\w(\w_2,\balpha_2; \X, \y)\|_2^2\leq \|H(\X, \y)(\balpha_1-\balpha_2)\|_2^2\leq c \|\balpha_1-\balpha_2\|_2^2,
\end{align*}
where we use the assumption $\|H(\X, \y)\|_2^2 = \|H(\X, \y)^{\top}\|_2^2\leq c$.

\section{The differences between Algorithm 1 in~\citep{Chambolle:2011:FPA:1968993.1969036}  and Pdprox-primal algorithm (Algorithm~\ref{alg:2}) and Pdprox-dual algorithm (Algorithm~\ref{alg:3})}\label{sec:apc}
We make the following correspondences between our notations (appearing the R.H.S of the following equalities) and the notations in~\citep{Chambolle:2011:FPA:1968993.1969036} (appearing the L.H.S of the following equalities),
\begin{align*}
&\x=\w, \quad \y=\balpha, \quad \bar\x=\u\\
&G(\w)= \lambda R(\w) +\w^{\top}\b + I_{\Q_\w}(\w) \\
&F^*(\balpha) = - \balpha^{\top}\a + I_{\Q_{\balpha}}(\balpha)\\
&K = H^{\top} \\
& \alpha^{\top}H^{\top}\w  + \w^{\top}\b + \balpha^{\top}\a + c_0 = L(\w, \balpha)\\
&\delta=\tau=\gamma\\
&\theta=1
\end{align*}
where we suppress  the dependence of $\a, \b, H, c_0$ on $(\X, \y)$, and $I_{\Q}(\x)$ is an indicator function 
\begin{align*}
I_{\Q}(\x)=\left\{\begin{array}{cc}0, \text{ if $\x\in\Q$}\\ +\infty, \text{ otherwise}\end{array} \right.
\end{align*}
The problem in ~\citep{Chambolle:2011:FPA:1968993.1969036} is to solve
\begin{align*}
\min_\w\max_{\balpha} \mathcal O (\w,\balpha) = \alpha^{\top}H^{\top}\w  + G(\w)
- F^*(\alpha)
\end{align*}
and the updates in ~\citep{Chambolle:2011:FPA:1968993.1969036} are calculated by
\begin{align*}
&\balpha_t = \min_{\balpha}\frac{\|\balpha -(\balpha_{t-1} + \gamma
H^{\top}\u_{t-1} ) \|_2^2}{2\gamma}  + F^*(\alpha)\\
&\w_t = \min_{\w} \frac{\|\w - (\w_{t-1} - \gamma H\balpha_t
))\|_2^2}{2\gamma} + G(\w)\\
&\u_t = \w_t + \theta(\w_t - \w_{t-1})
\end{align*}
or equivalently
\begin{align*}
&\balpha_t = \min_{\balpha\in\Q_{\balpha}}\frac{\|\balpha
-(\balpha_{t-1} + \gamma (H^{\top}\u_{t-1}  + \a) ) \|_2^2}{2\gamma}\\
&\w_t = \min_{\w\in\Q_\w} \frac{\|\w - (\w_{t-1} - \gamma(H\balpha_t +
\b))\|_2^2}{2\gamma} + \lambda R(\w)\\
&\u_t = \w_t + \theta(\w_t - \w_{t-1})
\end{align*}
Note that the partial gradients of $L(\w,\balpha)$ are $G_\w(\w,\balpha) = G_\w(\balpha) = H\balpha  + b$ and
$G_{\balpha}(\w,\balpha)= G_{\balpha}(\w) = H^{\top}\w + \a$~\footnote{We use $G_\w$ and $G_{\balpha}$ to denote partial gradients.}, then we can
write the above updates as
\begin{align*}
&\balpha_t = \min_{\balpha\in\Q_{\balpha}}\frac{\|\balpha
-(\balpha_{t-1} + \gamma G_{\balpha}(\u_{t-1}) ) \|_2^2}{2}\\
&\w_t = \min_{\w\in\Q_\w} \frac{\|\w - (\w_{t-1} - \gamma
G_\w(\balpha_t))\|_2^2}{2} + \gamma\lambda R(\w)\\
&\u_t = \w_t + \theta(\w_t - \w_{t-1})
\end{align*}
However the updates of Pdprox-primal algorithm (Algorithm~\ref{alg:2})  in our paper are
\begin{align*}
&\w_t = \min_{\w\in\Q_\w} \frac{\|\w - (\u_{t-1} - \gamma
G_\w(\balpha_{t-1}))\|_2^2}{2} + \gamma \lambda R(\w)\\
&\balpha_t = \min_{\balpha\in\Q_{\balpha}}\frac{\|\balpha
-(\balpha_{t-1} + \gamma G_{\balpha}(\w_t) ) \|_2^2}{2}\\
&\u_t = \w_t + \gamma (G_\w(\balpha_{t-1}) - G_\w(\balpha_t))
\end{align*}
If we remove the extra primal variable $\u_t$, we have the following
updates of Algorithm 1 in~\citep{Chambolle:2011:FPA:1968993.1969036}:
\begin{equation}
\begin{aligned}
&\balpha_t = \min_{\balpha\in\Q_{\balpha}}\frac{\|\balpha
-(\balpha_{t-1} + \gamma G_{\balpha}(2\w_{t-1}- \w_{t-2}) )
\|_2^2}{2}\\
&\w_t = \min_{\w\in\Q_\w} \frac{\|\w - (\w_{t-1} - \gamma
G_\w(\balpha_t))\|_2^2}{2} + \gamma\lambda R(\w)
\end{aligned}
\end{equation}\label{eqn:pdprimal-2}
and the following updates of the Pdprox-primal algorithm:
\begin{equation}
\begin{aligned}
&\w_t = \min_{\w\in\Q_\w} \frac{\|\w - (\w_{t-1} -
\gamma G_\w(2\balpha_{t-1}-\balpha_{t-2}))\|_2^2}{2} + \gamma
\lambda R(\w)\\
&\balpha_t = \min_{\balpha\in\Q_{\balpha}}\frac{\|\balpha
-(\balpha_{t-1} + \gamma G_{\balpha}(\w_t) ) \|_2^2}{2}\\
\end{aligned}
\end{equation}
We can clearly see the difference between our updates and the updates of Algorithm 1 in~\citep{Chambolle:2011:FPA:1968993.1969036}, which lies in the order  of updating on the primal variable and the dual variable, and the gradients used in the updating as well.   On the other hand, if we remove the extra dual variable in Algorithm~\ref{alg:3},  the updates are the same to that of Algorithm in~\citep{Chambolle:2011:FPA:1968993.1969036}, i.e., 
\begin{equation}
\begin{aligned}
&\balpha_t = \min_{\balpha\in\Q_{\balpha}}\frac{\|\balpha
-(\balpha_{t-1} + \gamma (2G_{\balpha}(\w_{t-1})-G_{\balpha}( \w_{t-2}) )
\|_2^2}{2}\\
&\w_t = \min_{\w\in\Q_\w} \frac{\|\w - (\w_{t-1} - \gamma
G_\w(\balpha_t))\|_2^2}{2} + \gamma\lambda R(\w)
\end{aligned}
\end{equation}
by noting that $G_{\balpha}(\w)$ is linear in $\w$.  It is also worth noting that Pdprox-primal can be implemented by maintaing one primal variable and two dual variables as in~(20), and similarly Pdprox-dual can be implemented by maintaing two primal variables and one dual variable as in (21). Depending on the nature of applications, we can choose different  implementations for Pdprox-primal or Pdprox-dual to achieve better efficiency.

\section{Proof of Lemma~\ref{lem:7}}
\label{app:lem:7}
In order to prove Lemma~\ref{lem:7},  we first present the following lemma with its proof.  

  \begin{lemma}\label{lem:8}Let $Z$ be a convex compact set, and $U\subseteq Z$ be convex and closed,  $\z_0\in Z$, and $\gamma>0$. Considering the following points with fixed $\eta, \xi$, 
 \begin{align*}
 \z_h &= \arg\min_{\z\in U}\frac{1}{2}\|\z- (\z_0- \gamma\xi)\|_2^2,\\
 \z_1&=\arg\min_{\z\in U} \frac{1}{2}\|\z-(\z_0-\gamma\eta)\|_2^2,
 \end{align*}
 then for any $\z\in U$, we have
 \begin{align*}
 \gamma\eta^{\top}(\z_h-\z)&\leq  \frac{1}{2}\|\z-\z_0\|_2 -\frac{1}{2}\|\z-\z_1\|_2 + \gamma^2\|\xi-\eta\|_2^2 - \frac{1}{2}\Big{[}\|\z_h-\z_0\|_2^2 + \|\z_h-\z_1\|_2^2\Big{]}.
 \end{align*}
 \end{lemma}
Equipped  with above lemma, it is straightforward to prove Lemma~\ref{lem:7}.  
%\begin{proof}{of Lemma~\ref{lem:7}:}
%To prove Lemma~\ref{lem:7}, 
We note that the two updates in Lemma~\ref{lem:6} are the same as the two updates in Lemma~\ref{lem:8}  if we make the following correspondences: 
\begin{align*}
&U=Z=\mathbb R^d\times \Q_{\balpha},\quad \z=\left(\w\atop\alpha\right)\in U,\\
%&\omega(\z)=\frac{1}{2}\|\w\|^2_2+ \frac{1}{2}\|\alpha\|_2^2\\
&\z_0=\left(\u_{t-1}\atop \bbeta_{t-1}\right),\; \z_h=\left(\w_t\atop \balpha_t\right),\; \z_1=\left(\u_t\atop \bbeta_t\right),\\
&\xi=\left(G_\w(\u_{t-1}, \balpha_t)+\lambda\mathbf v_t\atop -G_{\balpha}(\u_{t-1}, \bbeta_{t-1})\right),\quad \eta=\left(G_\w(\w_t, \balpha_t)+ \lambda\mathbf v_t\atop -G_\alpha(\w_t, \balpha_t)\right).
\end{align*}
{\color{red}Then the inequality in Lemma~\ref{lem:7} follows immediately  the inequality in Lemma~\ref{lem:8}, which is stated explicitly again: 
\begin{align*}
&\gamma\begin{pmatrix} G_\w(\w_t, \balpha_t) +\lambda \mathbf v_t\\-G_{\balpha} (\w_t, \balpha_t)\end{pmatrix}^{\top}\begin{pmatrix}\w_t-\w\\ \balpha_t-\balpha\end{pmatrix}\leq \frac{1}{2}\left\|\begin{pmatrix}\w-\u_{t-1}\\ \balpha-\bbeta_{t-1}\end{pmatrix}\right\|_2^2 - \frac{1}{2}\left\|\begin{pmatrix}\w-\u_t\\ \balpha -\bbeta_t\end{pmatrix}\right\|_2^2\\
&+{\gamma^2}\left\|G_{\balpha}(\w_t, \balpha_t)- G_{\balpha}(\u_{t-1}, \bbeta_{t-1})\right\|_2^2- \frac{1}{2}\left[\|\w_t-\u_{t-1}\|_2^2+\underbrace{\|\balpha_t-\bbeta_{t-1}\|_2^2+\|\w_t-\u_{t}\|_2^2+\|\balpha_t-\bbeta_{t}\|_2^2}\limits_{\geq 0}\right].
%&+ \frac{\gamma^2}{2}\left \|\frac{1}{n}\overline{\X}\w_{t}-\frac{1}{n}\overline{\X}\z_{t-1}\right\|^2
%&
%&\leq \frac{1}{2}\left\|\begin{pmatrix}\w-\z_{t-1}\\ u-v_{t-1}\end{pmatrix}\right\|^2 - \frac{1}{2}\left\|\begin{pmatrix}\w-\z_t\\ u -v_t\end{pmatrix}\right\|^2 \\
%&+ \frac{\gamma^2}{2}\sum_i\left\|\frac{1}{n}y_i\x_i^{\top}\left(\w_t-\z_{t-1}\right)\right\|^2 - \frac{1}{2}\|\w_t-\z_{t-1}\|^2\\
%&\leq \frac{1}{2}\left\|\begin{pmatrix}\w-\z_{t-1}\\ u-v_{t-1}\end{pmatrix}\right\|^2 - \frac{1}{2}\left\|\begin{pmatrix}\w-\z_t\\ u -v_t\end{pmatrix}\right\|^2 \\
%&+ \frac{\gamma^2 R^2}{2n}\left\|\left(\w_t-\z_{t-1}\right)\right\|^2 - \frac{1}{2}\|\w_t-\z_{t-1}\|^2\\
\end{align*}}

%\end{proof}

Lemma~\ref{lem:8} is a special case of Lemma 3.1~\citep{Nemirovski2005} for Euclidean norm.  A proof is provided here for completeness.

\begin{proof}[of Lemma~\ref{lem:8}]
Since 
\begin{align*}
 \z_h &= \arg\min_{\z\in U} \frac{1}{2}\|\z - (\z_0 -\gamma\xi)\|^2_2,\\
 \z_1&=\arg\min_{\z\in U}\frac{1}{2}\|\z - (\z_0 -\gamma\eta)\|^2_2,
  \end{align*}
 by the first order optimality  condition, we have
 \begin{align}
& (\z-\z_h)^{\top} (\gamma\xi - \z_0 + \z_h)\geq 0, \forall \z\in U\label{eqn:b1},\\
&(\z-\z_1)^{\top} (\gamma\eta- \z_0 + \z_1)\geq 0, \forall \z\in U\label{eqn:b2}.
 \end{align}
 Applying~(\ref{eqn:b1}) with $\z= \z_1$ and~(\ref{eqn:b2}) with $\z= \z_h$, we get 
 \begin{align*}
 & \gamma(\z_h-\z_1)^{\top}\xi\leq  (\z_0 - \z_h)^{\top}(\z_h-\z_1),\\
&\gamma(\z_1-\z_h)^{\top} \eta\leq (\z_0 - \z_1)^{\top}(\z_1-\z_h).
 \end{align*}
 Summing up the two inequalities, we have
 \begin{align*}
 \gamma (\z_h-\z_1)^{\top}(\xi-\eta) \leq (\z_1 - \z_h)^{\top}(\z_h-\z_1) = - \|\z_1-\z_h\|_2^2.
 \end{align*}
 Then 
 \begin{align}
 \gamma \|\xi-\eta\|_2\|\z_h-\z_1\|_2&\geq   -\gamma (\z_h-\z_1)^{\top}(\xi-\eta)\geq \|\z_1-\z_h\|_2^2.\label{eqn:b3}
 \end{align}
% where in the last inequality, we use the strong convexity of $\omega(\z)$.
We continue the proof as follows: 
 \begin{align*}
\frac{1}{2}\|\z- \z_0\|^2_2&- \frac{1}{2}\|\z- \z_1\|^2_2\\%\frac{1}{2}\|\z_1\|^2_2-\frac{1}{2}\|\z_0\|^2_2 +(\z-\z_1)^{\top}\z_1 -(\z-\z_0)^{\top}\z_0 \\
%= &\frac{1}{2}\|\z_1\|^2_2-\frac{1}{2}\|\z_0\|_2^2 +(\z-\z_1)^{\top}\z_1 -( \z - \z_1)^{\top}\z_0 -  (\z_1 -\z_0)^{\top} \z_0 \\
= &\frac{1}{2}\|\z_1\|_2^2 -\frac{1}{2}\|\z_0\|_2^2-(\z_1 -\z_0 )^{\top} \z_0 +(\z-\z_1)^{\top} (\z_1- \z_0)\\
=& \frac{1}{2}\|\z_1\|_2^2 - \frac{1}{2}\|\z_0\|_2^2  -( \z_1 -\z_0)^{\top}\z_0 + (\z-\z_1)^{\top}(\gamma\eta+   \z_1- \z_0) - (\z-\z_1)^{\top} \gamma\eta\\
\geq &\frac{1}{2}\|\z_1\|_2^2 - \frac{1}{2}\|\z_0\|_2^2  -( \z_1 -\z_0)^{\top}\z_0  - (\z-\z_1)^{\top} \gamma\eta\\
=& \underbrace{\frac{1}{2}\|\z_1\|_2^2 - \frac{1}{2}\|\z_0\|_2^2  -( \z_1 -\z_0)^{\top}\z_0  - (\z_h-\z_1)^{\top} \gamma\eta}\limits_{\epsilon} + (\z_h -\z)^{\top}\gamma\eta,
\end{align*}
where the inequality follows~(\ref{eqn:b2}).
 \begin{align*}
 \epsilon =& \frac{1}{2}\|\z_1\|_2^2   -\frac{1}{2}\|\z_0\|_2^2- (\z_1 -\z_0 ) ^{\top}\z_0 - (\z_h-\z_1)^{\top}\gamma\eta\\
=& \frac{1}{2}\|\z_1\|_2^2  -\frac{1}{2}\|\z_0\|_2^2 - ( \z_1 -\z_0 )^{\top} \z_0 - (\z_h-\z_1)^{\top}\gamma ( \eta - \xi)-  ( \z_h-\z_1)^{\top}\gamma \xi \\
=&\frac{1}{2}\|\z_1\|_2^2  -\frac{1}{2}\|\z_0\|_2^2 - ( \z_1 -\z_0 )^{\top} \z_0 - (\z_h-\z_1)^{\top}\gamma ( \eta - \xi)\\
& + (\z_1 - \z_h)^{\top}(\gamma  \xi -\z_0 + \z_h )- (\z_1- \z_h)^{\top}(\z_h-\z_0 )\\
\geq &\frac{1}{2}\|\z_1\|_2^2  -\frac{1}{2}\|\z_0\|_2^2 - ( \z_1 -\z_0 )^{\top} \z_0 - (\z_h-\z_1)^{\top}\gamma ( \eta - \xi)-( \z_1- \z_h)^{\top} (\z_h-\z_0)\\
= & \frac{1}{2}\|\z_1\|_2^2 -\frac{1}{2}\|\z_0\|_2^2 - (\z_h -\z_0)^{\top} \z_0 -(\z_h-\z_1)^{\top}\gamma (\eta - \xi) - ( \z_1- \z_h)^{\top}\z_h \\
= &\left[ \frac{1}{2}\|\z_1\|_2^2- \frac{1}{2}\|\z_h\|_2^2 - (\z_1- \z_h)^{\top}\z_h \right]+\left[\frac{1}{2}\|\z_h\|_2^2 -\frac{1}{2}\|\z_0\|_2^2 -(\z_h -\z_0 )^{\top} \z_0 \right]\\
&- (\z_h-\z_1)^{\top}\gamma ( \eta- \xi)\\
\geq& \frac{1}{2} \| \z_h - \z_1\|_2^2 + \frac{1}{2}\|\z_h - \z_0 \|_2^2 - \gamma \| \z_h-\z_1 \|_2 \| \eta - \xi \|_2\\
\geq & \frac{1}{2}\{ \| \z_h-\z_1 \|^2 + \|\z_h - \z_0 \|^2\} -\gamma^2 \|\eta - \xi \|_2^2,
\end{align*}
where the first inequality follows~(\ref{eqn:b1}), and the last inequality follows~(\ref{eqn:b3}). Combining the above results, we have
\begin{align*}
\gamma(\z_h -\z)^{\top}\eta \leq \frac{1}{2}\|\z- \z_0\|^2_2&- \frac{1}{2}\|\z- \z_1\|^2_2+ \gamma^2 \|\eta - \xi \|_2^2 - \frac{1}{2} \{\| \z_h-\z_1  \|_2^2 + \|\z_h - \z_0 \|_2^2\}. 
\end{align*}

\end{proof}

\section{Proof of Lemma~\ref{lem:2}}
\label{app:lem:2}
By introducing Lagrangian multiplier for constraint $\sum_i \alpha_iv_i\leq\rho$, we have the following min-max problem
\begin{align*}
\max_{\eta}\min\limits_{\balpha\in[0, 1]^n}\frac{1}{2}\|\balpha - \widehat \balpha\|^2+ \eta \left(\sum_i \alpha_i v_i- \rho\right).
\end{align*}
The solution to $\balpha$ is $\alpha_i = [\widehat\alpha_i - \eta^*v_i]_{[0, 1]}$. By KKT condition, the optimal $\eta^*$ is equal to $0$ if $\sum_i[\widehat\balpha_i]_{[0, 1]}v_i< \rho$, otherwise we have
\begin{align*}
\sum_i [\widehat\alpha_i-\eta^* v_i]_{[0, 1]}v_i- \rho=0.
\end{align*}
Since the left side of above equation is a monotonically decreasing function in $\eta^*$, we can compute $\eta^*$ by efficient bi-section search.

\section{Proof of Lemma~\ref{lem:gr}}
\label{app:lem:gr}
Using the convex conjugate $V_*(\eta)$ of $V(z)$, the composite mapping can be written as
\begin{align*}
\min_{\w}\frac{1}{2}\|\w-\widehat\w\|_2^2 + \lambda \max_{\eta}\left( \eta \|\w\| - V_*(\eta)\right).
\end{align*}
The problem is equivalent to maximize the following function on $\eta$, 
\begin{align*}
\left(\min_{\w}\frac{1}{2}\|\w-\widehat\w\|_2^2 + \lambda \eta \|\w\|\right) -\lambda V_*(\eta).
\end{align*}
Let $\w(\eta)$ denote the solution to the minimization problem. Then the optimal solution of $\eta$ satisfies 
\begin{align*}
\lambda \|\w(\eta)\|- \lambda V'_*(\eta)=0,
\end{align*}
i.e. 
\begin{align*}
 \|\w(\eta)\|- V'_*(\eta)=0.
\end{align*}
It is easy to show that $\|\w(\eta)\|$ is a non-increasing function in $\eta$. Similarly, since $V_*(\eta)$ is a convex function, its negative gradient $-V'_*(\eta)$ is a non-increasing function. Therefore, we can compute the optimal $\eta$ by bi-section search.

\end{document}